\documentclass[]{fairmeta}
\usepackage[ruled,vlined,linesnumbered]{algorithm2e}
\DontPrintSemicolon
\usepackage{amssymb}
\usepackage{amsmath}
\usepackage{amsfonts}
\usepackage{mathtools}
\usepackage{amsthm}
\usepackage{enumitem}
\usepackage[usenames,dvipsnames]{xcolor}
\usepackage{hyperref}
\usepackage{thmtools}
\usepackage{fancyhdr}
\usepackage{color}
\usepackage{xspace}
\usepackage{graphicx}
\usepackage{mathrsfs}
\usepackage{titlesec}
\usepackage{titletoc}
\usepackage{booktabs}
\usepackage{subcaption}
\usepackage{listings}
\usepackage{longtable}
\usepackage{makecell}
\usepackage[most]{tcolorbox}

\definecolor{sectioncolor}{HTML}{19519a} 
\definecolor{light-gray}{HTML}{F5F5F5}

\usepackage{natbib}

  \DeclareMathOperator{\sign}{sign}

  \DeclareMathOperator{\trace}{tr}

  \newcommand{\transt}{\mkern-1.5mu\mathsf{T}}
  \newcommand{\trans}{^{\transt}}  

\graphicspath{{./pics/}}

\newcommand{\va}{\boldsymbol{a}}
\newcommand{\vb}{\boldsymbol{b}}

\newcommand{\vh}{\boldsymbol{h}}

\newcommand{\vk}{\boldsymbol{k}}

\newcommand{\vq}{\boldsymbol{q}}

\newcommand{\vs}{\boldsymbol{s}}

\newcommand{\vu}{\boldsymbol{u}}
\newcommand{\vv}{\boldsymbol{v}}

\newcommand{\vx}{\boldsymbol{x}}
\newcommand{\vy}{\boldsymbol{y}}
\newcommand{\vz}{\boldsymbol{z}}

\newcommand{\valpha}{\boldsymbol{\alpha}}

\newcommand{\vnu}{\boldsymbol{\nu}}

\newcommand{\mA}{\mathbf{A}}

\newcommand{\mE}{\mathbf{E}}

\newcommand{\mI}{\mathbf{I}}

\newcommand{\mP}{\mathbf{P}}

\newcommand{\mW}{\mathbf{W}}

\newcommand{\R}{\mathbb R}

\newcounter{ccnt}
\NewDocumentCommand{\nc}{m}{
  \refstepcounter{ccnt}\ensuremath{c_{\theccnt}}\IfValueT{#1}{\label{#1}}%
}

\newcounter{bigccnt}
\NewDocumentCommand{\nC}{m}{%
  \refstepcounter{bigccnt}\ensuremath{C_{\thebigccnt}}\IfValueT{#1}{\label{#1}}%
}

\newcommand\MTkillspecial[1]{
  \bgroup
  \catcode`\&=9
  \let\\\relax%
  \scantokens{#1}%
  \egroup
}
\newcommand{\DeclareCustomDelim}[3]{
  \DeclarePairedDelimiter{#1}{#2}{#3}
  \reDeclarePairedDelimiterInnerWrapper{#1}{star}{
    \mathopen{##1\vphantom{\MTkillspecial{##2}}\kern-\nulldelimiterspace\right.}
  ##2
  \mathclose{\left.\kern-\nulldelimiterspace\vphantom{\MTkillspecial{##2}}##3}
  }
}

\DeclareCustomDelim{\prn}{\lparen}{\rparen}
\DeclareCustomDelim{\crl}{\{}{\}}
\DeclareCustomDelim{\brk}{[}{]}
\DeclareCustomDelim{\norm}{\|}{\|}
\DeclareCustomDelim{\abs}{|}{|}
\DeclareCustomDelim{\angl}{\langle}{\rangle}

\newcommand{\Problet}{\ensuremath\mathbb{P}}
\newcommand{\Expectlet}{\ensuremath\mathbb{E}}

\newcommand{\Varlet}{\mathbb{V}}

\newcommand*\makeAlph[1]{\symbol{\numexpr96+#1}}
\NewDocumentCommand{\labrel}{m o }{%
  \IfNoValueTF{#2}{\makeAlph{#1}}{\stackrel{\mathrm{(\makeAlph{#1})}}{#2}}%
}
\NewDocumentCommand{\eqrefpt}{m m}{(\ref{#1}\makeAlph{#2})}

\let\Prob\undefined
\DeclarePairedDelimiterXPP\Prob[1]{\Problet}\{\}{}{
  
  #1}
\DeclarePairedDelimiterXPP\Expect[1]{\Expectlet}[]{}{
  
  #1}
\DeclarePairedDelimiterXPP\Var[1]{\Varlet}[]{}{
  
  #1}
\DeclarePairedDelimiterXPP\Esub[2]{\Expectlet_{#1}}[]{}{
  
  #2}
\DeclarePairedDelimiterXPP\Varsub[2]{\Varlet_{#1}}[]{}{
  
  #2}

\newcommand{\range}[2]{[#1:#2]}

\newboolean{isInternal}
\setboolean{isInternal}{false}
\ifthenelse{\boolean{isInternal}}{
  \newcommand{\internalComment}[1]{\textbf{\color{red}[}#1\textbf{\color{red}]}}
}{
  \newcommand{\internalComment}[1]{}
}

\NewDocumentCommand{\py}{+v}{}

\newsavebox{\pycodebox}      

\NewDocumentEnvironment{pycode}{}{%
  \VerbatimEnvironment%
  \begin{lrbox}{\pycodebox}%
    \begin{minipage}{\linewidth}%
      \begin{Verbatim}%
      }{
      \end{Verbatim}%
    \end{minipage}%
  \end{lrbox}%
}

\NewDocumentEnvironment{pycontext}{} {
  \VerbatimEnvironment
  \begin{lrbox}{\pycodebox}%
    \begin{minipage}{\linewidth}%
      \begin{Verbatim}
      }{
      \end{Verbatim}
    \end{minipage}
  \end{lrbox}
}

\newtcolorbox{roundnote}[1][]{
  enhanced,
  breakable,
  colback=metabg,
  colframe=sectioncolor,
  boxrule=0.5pt,
  arc=3mm,
  left=6pt,right=6pt,top=6pt,bottom=6pt,
  before upper=\setlength{\parskip}{\medskipamount},
  fonttitle=\bfseries,
  colbacktitle=sectioncolor!15,
  coltitle=black,
  #1
}

\declaretheoremstyle[
headfont=\bfseries,
spaceabove=\topsep,
spacebelow=\topsep,
bodyfont=\slshape,
]{plain}

\newcounter{theoremcnt}[section]

\declaretheorem[
style=plain,
sibling=theoremcnt,
name=Definition,
Refname={Definition,Definitions},
]{definition}

\declaretheorem[
style=plain,
sibling=theoremcnt,
name=Assumption,
Refname={Assumption,Assumptions},
]{assumption}

\declaretheorem[
style=plain,
name=Desideratum,
Refname={Desideratum,Desiderata},
]{desideratum}

\declaretheorem[
style=plain,
sibling=theoremcnt,
name=Lemma,
Refname={Lemma,Lemmas},
]{lemma}

\declaretheorem[
style=remark,
sibling=theoremcnt,
name=Remark,
Refname={Remark,Remarks},
]{remark}

\declaretheorem[
style=remark,
numbered=no,
name=Remark,
Refname={Remark,Remarks},
]{remark*}

\declaretheoremstyle[
headfont=\itshape,
sibling=subsection,
spaceabove=\medskipamount,
spacebelow=\medskipamount,
]{remark}

\definecolor{mplblue}{HTML}{1f77b4}
\definecolor{mplorange}{HTML}{ff7f0e}
\definecolor{mplgreen}{HTML}{2ca02c}
\definecolor{mplred}{HTML}{d62728}
\definecolor{mplpurple}{HTML}{9467bd}
\definecolor{mplbrown}{HTML}{8c564b}
\definecolor{mplpink}{HTML}{e377c2}
\definecolor{mplgray}{HTML}{7f7f7f}
\definecolor{mplolive}{HTML}{bcbd22}
\definecolor{mplcyan}{HTML}{17becf}

\newcommand{\dmodel}{d_{\text{model}}}
\newcommand{\din}{d_\text{in}}
\newcommand{\dout}{d_\text{out}}
\newcommand{\depth}{\ensuremath n_{\text{layers}}}
\newcommand{\depthcor}{\ensuremath m_{\text{depth}}}
\newcommand{\depthalph}{\ensuremath \alpha_{\text{depth}}}
\newcommand{\width}{\ensuremath d_{\text{model}}}
\newcommand{\widthcor}{\ensuremath m_{\text{width}}}
\newcommand{\datacor}{\ensuremath m_{\text{data}}}
\newcommand{\dataalph}{\ensuremath \alpha_{\text{data}}}
\newcommand{\nheads}{\ensuremath n_\text{heads}}

\newcommand{\mup}{$\mu$P\xspace}

\allowdisplaybreaks[4]

\title{Learning Rate Transfer in Normalized Transformers}

\author[1,2,*]{Boris Shigida}
\author[2]{Boris Hanin}
\author[1]{Andrey Gromov}

\affiliation[1]{Meta Superintelligence Labs}
\affiliation[2]{Princeton University}

\contribution[*]{work done during an internship at Meta}

\abstract{
The Normalized Transformer, or nGPT \citep{loshchilov2025ngpt}  achieves impressive training speedups and does not require weight decay or learning rate warmup. However, despite having hyperparameters that explicitly scale with model size, we observe that nGPT does not exhibit learning rate transfer across model dimension and token horizon. To rectify this, we combine numerical experiments with a principled use of alignment exponents \citep{everett2024scalingexponentsacross} to revisit and modify the $\mu$P approach to hyperparameter transfer \citep{pmlr-v139-yang21c}. The result is a novel nGPT parameterization we call $\nu$GPT. Through extensive empirical validation, we find $\nu$GPT exhibits learning rate transfer across width, depth, and token horizon.
}

\date{\today}
\correspondence{\email{bs1624@princeton.edu}}

\begin{document}

\maketitle

\section{Introduction}\label{sec:intro}

Neural network performance is sensitive to a range of optimization hyperparameters (HPs) including initialization scheme, learning rate, batch size, weight decay, and compute budget.
Since direct search for good HPs at large scale is impractical
it is useful in practice to extrapolate from HP sweeps in small models
trained on limited token horizons to performant HPs at much larger scale, using HP transfer techniques.
Our goal in this article is to extend and refine such techniques to the important setting of Normalized Transformers or nGPT
\citep{loshchilov2025ngpt}.
These models constrain weight and activation norms by extensive use of normalization (see \cref{sec:ngpt-def}).
Along with trainable scale parameters that ensure no loss in expressivity,
this removes the need for weight decay and learning rate warmup while achieving impressive performance and training speedups.
Using a mix of principled heuristics and empirical scaling law fits, we propose a novel parameterization we call $\nu$GPT (\cref{sec:mungpt} and \cref{tab:our-param}),
for transferring HPs across model depth, width, and token horizon.

Specifically:
\begin{itemize}
\item Inspired by the theoretical framework in \cite{everett2024scalingexponentsacross},
  we check empirically that weight-activation alignments in nGPT do not satisfy the  hypotheses that underlie typical $\mu$P-type parameterizations.
  Instead, starting with an empirically supported mid alignment assumption (in which weights and activations are partially but not completely aligned),
  we derive and validate (\cref{fig:width-fixed-steps}) a novel prescription for transferring learning rates when scaling model width.
  We find that our parametrization transfers over width somewhat better than $\mu$P (\cref{fig:ours-vs-mup-width}).

\item Adopting a heuristic from \cite{bordelon2024infinite,dey2025dont} for HP transfer across depth,
  we show that $\nu$GPT allows for transfer over depth and high stability of deep models (\cref{fig:depth-sweeps}).

\item Finally, we find that the optimal learning rate in nGPT scales across token horizon like $\#\text{tokens}^{-1/3}$, consistent with measurements in \cite{bjorck2025scaling} for un-normalized Transformers.

\item As a result, our experiments show that $\nu$GPT obtains HP transfer across depth, width, and token horizon, with no loss in performance compared to the original (well-tuned) nGPT baseline (\cref{fig:width-sweep-ngpt-ours-with-token-corr}).
\end{itemize}

\begin{figure}[t]
  \centering

  \begin{tabular}{c c}
    \begin{subfigure}[h]{0.48\textwidth}
      \includegraphics[width=\linewidth]{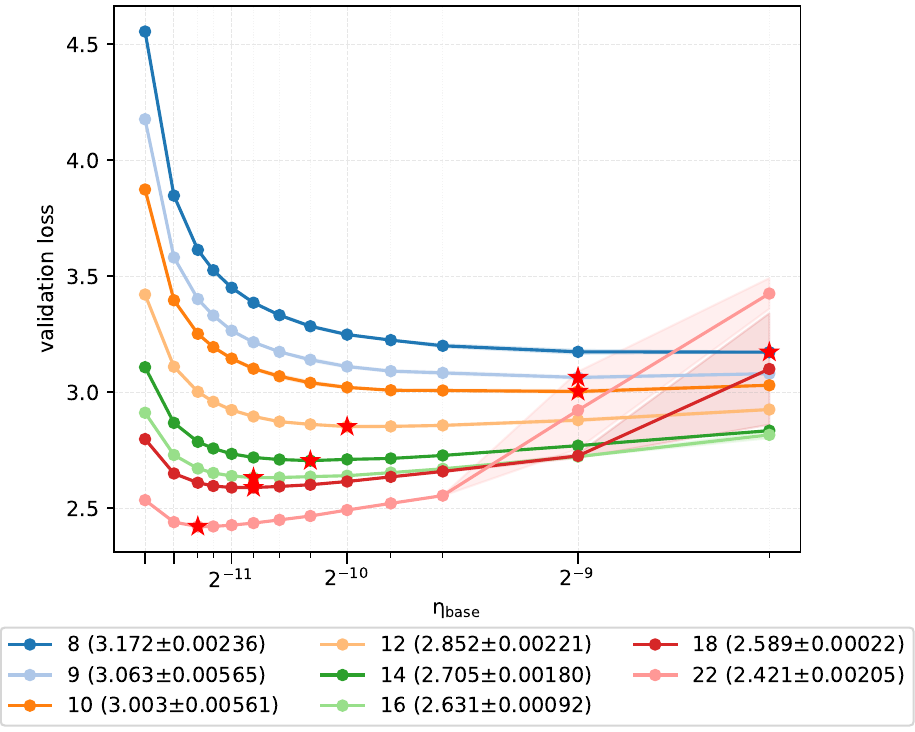}
      \caption{nGPT (baseline)\label{fig:width-sweep-ngpt-ours-tpp-baseline}}
    \end{subfigure}
    &
      \begin{subfigure}[h]{0.48\textwidth}
        \includegraphics[width=\linewidth]{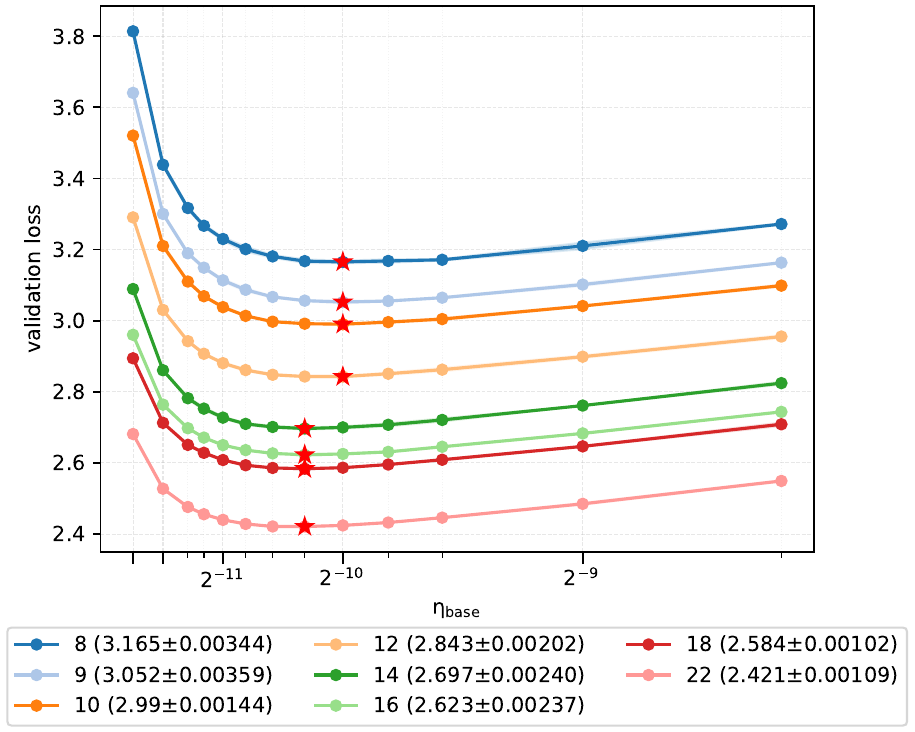}
        \caption{$\nu$GPT (ours)\internalComment{with $\eta_{\text{output}}$ multiplied by $2^{-1}$ and mid-alignment.}\label{fig:width-sweep-ngpt-ours-tpp-mup-tuned-midalign}}
      \end{subfigure}
  \end{tabular}

  \caption{Different model sizes at \textbf{fixed aspect ratio} $\depth = \nheads$ (in the legend, with the best loss in parentheses) and ``compute-optimal'' 20 tokens per parameter.
    The baseline nGPT does not show learning rate transfer.
    Our parametrization $\nu$GPT does, and performs no worse or slightly better.
    The points are averaged over three initialization seeds and no validation loss EMA is used.
  }
  \label{fig:width-sweep-ngpt-ours-with-token-corr}
\end{figure}

\begin{figure}[t]
  \centering
  \begin{tabular}{cc}
    \begin{subfigure}[h]{0.48\textwidth}
      \includegraphics[width=\linewidth]{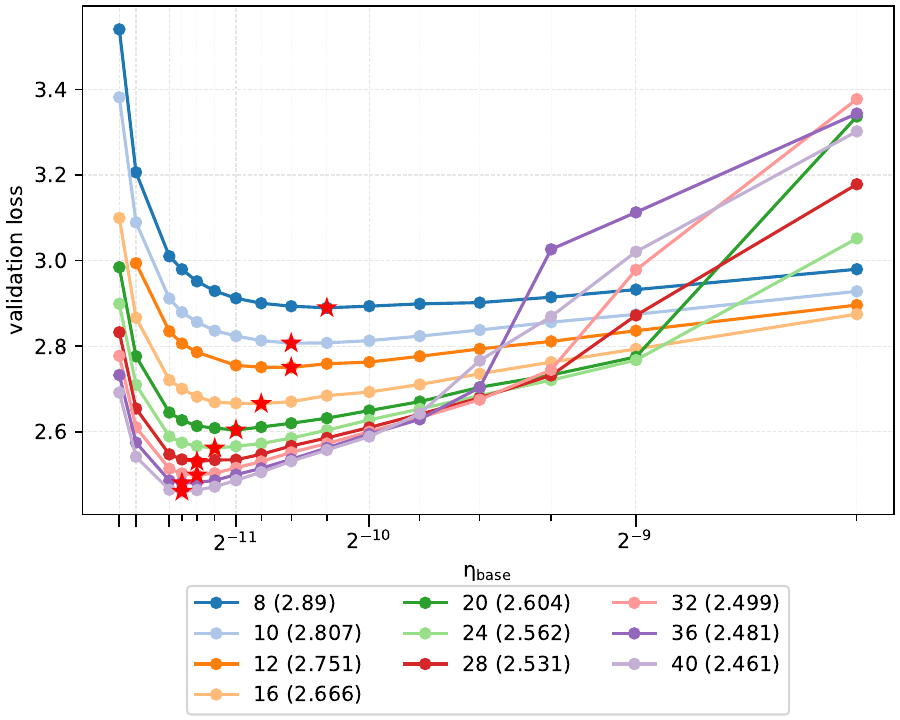}
      \caption{nGPT (baseline)}
      \label{fig:width-fixed-steps-baseline}
    \end{subfigure}
    &
    \begin{subfigure}[h]{0.48\textwidth} \includegraphics[width=\linewidth]{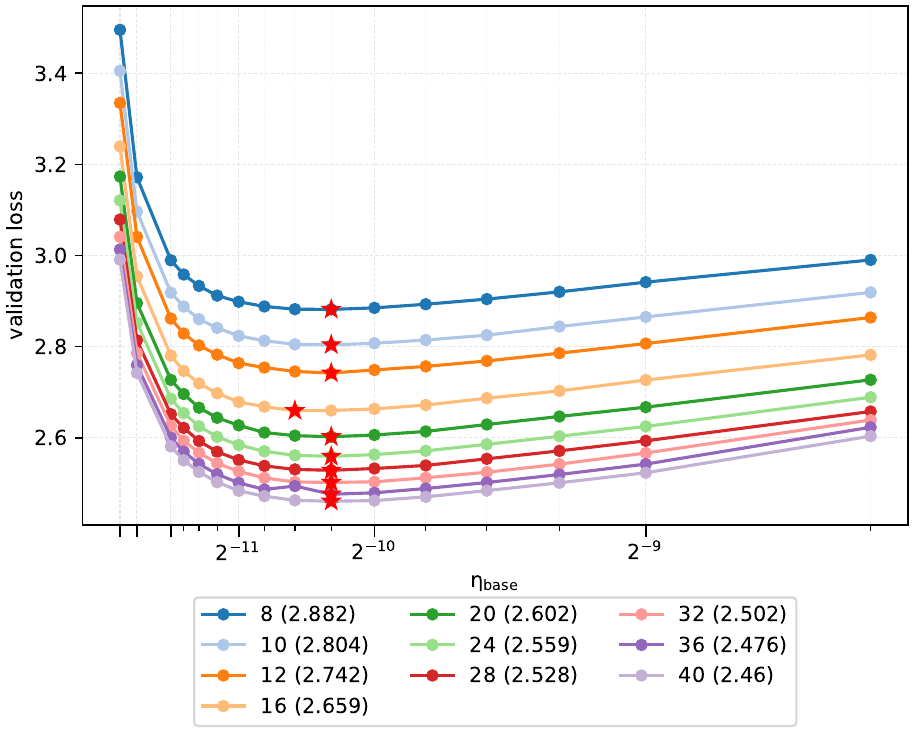}
      \caption{$\nu$GPT (ours)\internalComment{(mid align) with $\eta_{\text{output}}$ multiplied by $2^{-1}$}}
      \label{fig:width-fixed-steps-mupsamescalelogits-tuned-midalign}
    \end{subfigure}
  \end{tabular}

  \caption{
    \textbf{Width} sweeps with $\nheads$ increasing (in the legend, with the best loss in parentheses), head dimension and $\depth = 12$ fixed, 80\,000 iterations (about 21\,B tokens).
    Baseline nGPT does not show transfer, our parametrization $\nu$GPT shows essentially perfect transfer with no loss of performance.
  }
  \label{fig:width-fixed-steps}
\end{figure}

\subsection{Related work}\label{sec:related}

Early approaches to scaling the width (hidden dimension) of neural networks identified ``lazy'' infinite-width limits in which
the network becomes close to its linearization around initialization \citep{NEURIPS2018_5a4be1fa,NEURIPS2019_dbc4d84b,NEURIPS2019_ae614c55},
and in particular there is no feature learning.
It was then discovered that non-lazy limits exist
\citep{mei2018mean,NEURIPS2018_196f5641,sirignano2020mfplawlarge,nguyen2023rigorous}.
The Tensor Programs framework
\citep{yang2019widefeedforwardrecurrent,yang2020tensorprogramsiineural,yang2021tensorprogramsiib,yang2021tensorprogramsiiineural} allowed \citet{pmlr-v139-yang21c}
to classify large-width limits using a natural class they call ``abc-Parametrizations''
and describe a parametrization (stable, non-lazy) with weights updated as much as possible without blowup,
which they called the Maximal Update Parametrization (\mup).
This limit was studied with dynamical mean field theory in \citet{NEURIPS2022_d027a5c9}.
A simple theoretical perspective deriving \mup using spectral norms is provided in
\citet{yang2024spectralconditionfeature}.
It was empirically observed that hyperparameters such as the learning rate and initialization variance transfer
across width when using \mup
\citep{NEURIPS2021_8df7c2e3,lingle2025empiricalstudymu,vlassis2025a};
rigorous theoretical understanding of this phenomenon
is a subject of recent work \citep{hayou2025prooflearningrate}.
\Citet{everett2024scalingexponentsacross} conducted a large-scale empirical evaluation of learning rate transfer over width,
and questioned matrix-vector alignment assumptions in \mup leading to another non-equivalent parametrization that is not only non-trivial
but is not worse than
\mup in practice.
Relatedly,
\citet{kosson2026weight}
argue that such alignment assumptions of \mup are incorrect during training,
and correct alignment can be predicted theoretically with high precision if weight decay is used.
\Citet{blake2024up} combine \mup with Unit Scaling \citep{blake2023unitscalingout} ensuring that activations, weights and gradients have scale one at initialization.

Large width and depth limits for residual networks were identified using the Tensor Programs framework in \citet{yang2024tensor},
using dynamical mean field theory in \citet{bordelon2024depthwise}, with the latter later extended to Transformers in \citet{NEURIPS2024_3eff068e}.
In that literature, residual networks have the form $\vh^{\ell + 1} = \vh^{\ell} + n_{\text{layers}}^{-\depthalph} \mathcal{F}_{\ell}(\vh^{\ell})$, where $\mathcal{F}_{\ell}$ is the $\ell$th block,
and where $\depthalph$ is the exponent that controls the magnitude of the contribution from the new block $n_{\text{layers}}^{-\depthalph} \mathcal{F}_{\ell}(\vh^{\ell})$ relative to the previous
hidden state $\vh^{\ell}$; it can be shown that absent techniques like Post-LN, $\depthalph$ has to be between
$0.5$ and $1$ to avoid either blowup or triviality of the forward pass at large depths.
When using parametrizations that admit large width and depth limits, hyperparameters like the learning rate and initialization variance
may or may not transfer over depth (at small widths as well); there is disagreement about this issue in the literature.
\Citet{yang2024tensor} argue that $\depthalph = 1 / 2$ is optimal because it admits ``feature diversity'',
but is defective if residual blocks themselves have depth more than $1$ (where there is weight matrix multiplication inside of $\mathcal{F}_{\ell}$), the reason being that whenever $\depthalph \in [1 / 2, 1)$, the network becomes linearized in the weights of each layer.
In particular,
they do not identify transfer over depth for such realistic block structures.
\Citet{dey2025dont} observe transfer over depth for $\depthalph = 1$ but not $1 / 2$, and argue that it is because
$\depthalph = 1$ is the only value where the network is \textit{not} linearized in the weights of each layer.
\Citet{mlodozeniec2025completedhyperparametertransfer} observe transfer for both $\depthalph \in \crl{1 / 2, 1}$ and study which corrections should be made when changing token horizons and batch size. In short, transfer over depth for $\depthalph \neq 1$ is not a settled issue, whereas transfer for $\depthalph = 1$ appears uncontested (although it may lack feature diversity).
Concurrent work \citet{ren2026rethinking} applies $\mu$P-like derivations to a different variant of normalized models optimized with Muon (for most parameters), in particular noticing that Depth-$\mu$P is required in their setting.

There are alternative approaches to scaling the model size while finding near-optimal learning rates, such as viewing the network as consisting of modules and
normalizing the updates per-module \citep{large2024scalable}.

\subsection{Organization}

The rest of the paper is organized as follows. In order to keep our presentation self-contained, we devote a short \cref{sec:ngpt-def}
to defining nGPT and all its hyperparameters.
The next \cref{sec:mungpt} describes our reparametrization and the motivations behind it,
with more detailed proofs and derivations moved to \cref{sec:proofs-derivations}.
Ablations and illustrations of transfer across width, depth, and iteration count are provided in \cref{sec:experiments}.

\section{Normalized Transformers: definitions}\label{sec:ngpt-def}

The Normalized Transformer, or nGPT \citep{loshchilov2025ngpt}
maps an input sequence of one-hot vectors
$\vx_n \in \mathbb{R}^V$ (with $n \in \range{1}{\mathrm{SeqLen}}$)
of dimension $V$ (vocabulary size)
to a corresponding output sequence of vectors
$\vz_n \in \mathbb{R}^V$ (with $n \in \range{1}{\mathrm{SeqLen}}$)
through a series of normalized residual layers.
More precisely, like most Transformers, each input $\vx_n$ is passed through a trainable
linear embedding layer
\begin{equation*}
\vh_n^1 = \mE_{\text{input}} \vx_n \in \R^{d_{\text{model}}} ,\qquad \mE_{\text{input}}\in \mathbb{R}^{d_{\text{model}} \times V}.
\end{equation*}
The hidden state is then transformed through $n_{\text{layers}}$ residual blocks that alternate cross-token causal self-attention
\begin{align}
  \vh_{A, n}^{\ell + 1} &= \text{Norm}\prn[\big]{\text{Attention}_n(\crl{\vh_m^{\ell}}_{m = 1}^{\text{SeqLen}})}, \label{eq:SGBoZV}\\
  \vh_n^{\ell + 0.5} &= \text{Norm}\prn[\bigg]{\vh_n^{\ell} + \frac{\alpha_{\text{$A$,init}}}{\alpha_{\text{$A$,scale}}} \valpha_A^{\ell + 1} \odot (\vh_{A, n}^{\ell + 1} - \vh_n^{\ell})}, \label{eq:xyayiq}
\end{align}
and token-wise MLPs
\begin{align}
      \vh_{M, n}^{\ell + 1} &= \text{Norm}\prn[\big]{\text{MLP}(\vh_n^{\ell + 0.5})}, \label{eq:gDzJUa}\\
  \vh_n^{\ell + 1} &= \text{Norm}\prn[\bigg]{\vh_n^{\ell + 0.5} + \frac{\alpha_{\text{$M$,init}}}{\alpha_{\text{$M$,scale}}} \valpha_M^{\ell + 1} \odot (\vh_{M, n}^{\ell + 1} - \vh_n^{\ell + 0.5})} \label{eq:yxfkfR},
\end{align}
where we denote
\begin{equation*}
\text{Norm}(\vy) := \frac{\vy}{\norm{\vy}}
\end{equation*}
and $\odot$ is the component-wise product. The final residual state $\vh_n^{n_{\text{layers} + 1}}$ passes through a linear unembedding
\begin{equation}\label{eq:oDPOKs}
\vz_n = \frac{s_{\text{$z$,init}}}{s_{\text{$z$,scale}}} \vs_z\odot \hat{\vz}_n \in \R^{V},\qquad  \hat{\vz}_n = \mE_{\text{output}} \vh_n^{n_{\text{layers} + 1}},\qquad \mE_{\text{output}} \in \R^{V\times d_{\text{model}}}.
\end{equation}
We will suppress the dependence of $\vh_{A, n}$, $\vh_{M, n}$, $\vh_n$, $\valpha_A$, $\valpha_M$ on the block number to reduce notational clutter.
The columns of $\mE_{\text{input}}$ and rows $\mE_{\text{output}}$ are normalized before\footnote{In \citet{loshchilov2025ngpt}, they are normalized \textit{after} each step. This difference has no practical relevance.} starting each training step.

There are several key distinguishing features of nGPT when compared with a standard pre-LN Transformer:
\begin{itemize}
\item Standard $\text{RMSNorm}(\cdot)$ or $\mathrm{LayerNorm}(\cdot)$ are replaced with $\text{Norm}(\cdot)$, which $\ell_2$-normalizes the vector and has no trainable parameters.
\item The normalization is performed at the output of Attention and MLP blocks rather than input, which is sometimes called residual-post-norm \citep{Liu2021,olmo20252olmo2}.
\item The standard residual connection is replaced in \eqref{eq:xyayiq} and \eqref{eq:yxfkfR} by a ``linear interpolation'' function (LERP), in which trainable vectors  $\valpha_A$, $\valpha_M$ allow for a learned componentwise rescaling. The components of $\valpha_A$, $\valpha_M$  are constrained to be nonnegative. At initialization, each component of $\valpha_A$ is equal to a fixed global hyperparameter $\alpha_{\text{$A$,scale}}$, so that the value of $\frac{\alpha_{\text{$A$,init}}}{\alpha_{\text{$A$,scale}}} \valpha_A$ at initialization is equal (in each component) to a fixed global hyperparameter $\alpha_{\text{$A$,init}}$. The ratio $\frac{\alpha_{\text{$A$,init}}}{\alpha_{\text{$A$,scale}}}$ modifies the learning rate of $\valpha_A$. Rescalers $\valpha_M$ and $\vs_z$ are treated similarly.
\end{itemize}

In addition to the explicit normalization layers in the residual stream we normalize also the keys and queries inside the self-attention block. Specifically, the $n$th token output of an attention head is given by the standard expression
\begin{align*}
  \text{Head}_n(\crl{\vh_m}_{m = 1}^{\text{SeqLen}}) 
  &= \sum_{m = 1}^n \frac{\exp(\sqrt{d_{\text{key}}} \vq^{\prime\transt}_n \vk_m')}{\sum_{\tilde{m} = 1}^n \exp(\sqrt{d_{\text{key}}} \vq^{\prime \transt}_n \vk'_{\tilde{m}})} \vv_m,
\end{align*}
where $\vq'_n$ and $\vk'_m$ are normalized query and key vectors calculated using trainable matrices $\mW_q, \mW_k \in \mathbb{R}^{d_{\text{model}} \times d_{\text{key}}}$
\begin{equation}\label{eq:xgSNnZ}
\mathbb{R}^{d_{\text{key}}} \ni \vq_n = \text{QRot}_n(\mW_q\trans \vh_n), \quad \vq'_n = \frac{\vq_n}{\norm{\vq_n}} \odot \frac{s_{\text{$qk$,init}}}{s_{\text{$qk$,scale}}} \vs_{q k},
\end{equation}
and
\begin{equation}\label{eq:XBxnKv}
\mathbb{R}^{d_{\text{key}}} \ni \vk_m = \text{KRot}_m(\mW_k\trans \vh_m), \quad \vk'_m = \frac{\vk_m}{\norm{\vk_m}} \odot \frac{s_{\text{$qk$,init}}}{s_{\text{$qk$,scale}}} \vs_{q k}.
\end{equation}
Here $\text{QRot}_n(\cdot)$ and $\text{KRot}_m(\cdot)$ denote rotary positional embedding \citep{su2023roformerenhancedtransformer} maps,
whereas the value vectors are calculated using the trainable matrix $\mW_v \in \mathbb{R}^{d_{\text{model}} \times d_{\text{key}}}$
\begin{equation}\label{eq:MNAzYI}
\mathbb{R}^{d_{\text{key}}} \ni \vv_m = \mW_v\trans \vh_m.
\end{equation}
The meaning of $\vs_{qk}$, $s_{\text{$qk$,init}}$, and $s_{\text{$qk$,scale}}$ is analogous to $\valpha_{A}$, $\alpha_{\text{$A$,scale}}$, $\alpha_{\text{$A$,init}}$ discussed above.
In practice, there are $n_{\text{heads}}$ different attention heads with separate $(\mW_q, \mW_k, \mW_v, \vs_{qk})$. The outputs of all heads (each in $\mathbb{R}^{d_{\text{key}}}$) are concatenated (yielding a vector in $\mathbb{R}^{n_{\text{heads}} d_{\text{key}}}$),
and then a trainable linear transformation $\mW_O \in \mathbb{R}^{d_{\text{model}} \times n_{\text{heads}} d_{\text{key}}}$ transfers the resulting vector back to $\mathbb{R}^{d_{\text{model}}}$:
\begin{equation}\label{eq:JSpBcR}
  \text{Attention}_n(\crl{\vh_m}_{m = 1}^{\text{SeqLen}})
  = \mW_O
  \underbrace{\begin{bmatrix}
    \text{Head}_n^{(1)}(\crl{\vh_m}_{m = 1}^{\text{SeqLen}})\\
    \vdots\\
    \text{Head}_n^{(n_{\text{heads}})}(\crl{\vh_m}_{m = 1}^{\text{SeqLen}})
  \end{bmatrix}}_{\in \mathbb{R}^{n_{\text{heads}} d_{\text{key}}}}.
\end{equation}
Columns of $\mW_q^{(j)}, \mW_k^{(j)}, \mW_v^{(j)}$ for each head $j \in \range{1}{n_{\text{heads}}}$ and of $\mW_O$ are all normalized before starting each training step.

Finally, given a vector $\vh_n \in \R^{d_{\text{model}}}$ the MLP block computes
\begin{align*}
    \text{MLP}(\vh_n) = \mW_{o \text{MLP}} \prn[\big]{\text{SiLU}(\vnu) \odot \vu},
\end{align*}
where $\text{SiLU}(\vnu) = \vnu \odot \sigma(\vnu)$ with $\sigma(\cdot)$ the standard sigmoid (applied componentwise) and
\begin{align*}
    \vu =\mW_u \vh_n \odot \frac{s_{\text{$u$,init}}}{s_{\text{$u$,scale}}} \vs_{u}\in  \mathbb{R}^{d_{\text{MLP}}}, \quad \vnu = \mW_{\nu} \vh_n \odot \frac{s_{\text{$\nu$,init}}}{s_{\text{$\nu$,scale}}} d_{\text{model}}^{1 / 2} \vs_{\nu}\in  \mathbb{R}^{d_{\text{MLP}}}.
\end{align*}
The additional $d_{\text{model}}^{1 / 2}$ factor (explained in detail in Appendix A.1 of \citet{loshchilov2025ngpt}) is introduced to benefit better from the shape of SiLU (since the hidden state has a smaller scale than regular Transformers with $\text{RMSNorm}(\cdot)$).
The rows of $\mW_{\nu} \in \mathbb{R}^{d_{\text{MLP}} \times d_{\text{model}}}$, $\mW_{u} \in \mathbb{R}^{d_{\text{MLP}} \times d_{\text{model}}}$ and columns of $\mW_{o\text{MLP}} \in \mathbb{R}^{d_{\text{model}} \times d_{\text{MLP}}}$ are normalized before starting each training step.
The meaning of $(s_{\text{$u$,init}}, s_{\text{$u$,scale}}, \vs_u)$,
$(s_{\text{$\nu$,init}}, s_{\text{$\nu$,scale}}, \vs_\nu)$
is analogous to $(\alpha_{\text{$A$,scale}},\alpha_{\text{$A$,init}}, \valpha_{A})$ discussed above.

\section{Reparametrization for transfer over width, depth and token horizon}\label{sec:mungpt}

In this section we summarize our proposed $\nu$GPT parameterization for learning rate  transfer across width, depth, and token horizon.
We begin in \cref{sec:summary-of-changes} with a high-level summary of our parameterization.
We then provide in \cref{sec:width-depth-corr} a detailed discussion of the heuristics that lead to the depth and width transfer prescriptions in $\nu$GPT. Finally, we discuss in \cref{sec:token-horizon-corr} transfer across token horizon.

\subsection{Summary of the changes}\label{sec:summary-of-changes}

\begin{table}[t]
  \centering
  \begin{NiceTabular}{c c c c c c c}
    \toprule
    \multirow{2}{*}{\textbf{Hyperparameter}}
    & \multirow{2}{*}{\textbf{nGPT}}
    & \multicolumn{2}{c}{\textbf{$\boldsymbol{\mu}$P extensions}}
    & \multirow{2}{*}{\textbf{$\boldsymbol{\nu}$GPT (ours)}}
    \\
    \cmidrule(lr){3-4}
    & & \textbf{``Depth-$\boldsymbol{\mu}$P''} & \textbf{``CompleteP''} & \\
    \midrule
    $\eta_{\text{base}}$
    & $\eta_{\text{global}}$
    & $\eta_{\text{global}}$
    & $\eta_{\text{global}}$
    & $\eta_{\text{global}} {\color{mplblue} \datacor^{- 1 / 3}}$
    \\
    \midrule
    $\sigma^2_{\text{input}}$
    & arbitrary
    & arbitrary
    & arbitrary
    & arbitrary
    \\
    $\eta_{\text{input}}$
    & $\eta_{\text{base}}$
    & $\eta_{\text{base}} {\color{mplorange} \widthcor^{- 1 / 2}}$
    & $\eta_{\text{base}} {\color{mplorange} \widthcor^{- 1 / 2}}$
    & $\eta_{\text{base}} {\color{mplorange} \widthcor^{- 1 / 2}}$
    \\
    \midrule
    $\sigma^2_{\text{hidden}}$
    & arbitrary
    & arbitrary
    & arbitrary
    & arbitrary
    \\
    $\eta_{\text{hidden}}$
    & $\eta_{\text{base}}$
    & $\eta_{\text{base}} {\color{mplorange} \widthcor^{-1}} {\color{mplgreen} \depthcor^{-1 / 2}}$
    & $\eta_{\text{base}} {\color{mplorange} \widthcor^{-1}}$
    & $\eta_{\text{base}} {\color{mplorange} \widthcor^{-3 / 4}}$
    \\
    \midrule
    $\alpha_{\text{$A$,init}}$
    & $0.05$
    & $0.05 \, {\color{mplgreen} \depthcor^{-1 / 2}}$
    & $0.05 \, {\color{mplgreen} \depthcor^{-1}}$
    & $0.05 \, {\color{mplgreen} \depthcor^{-1}}$
    \\
    $\alpha_{\text{$A$,scale}}$
    & $\width^{-1 / 2}$
    & ${\color{mplpurple} 0.03}$
    & ${\color{mplpurple} 0.03}$
    & ${\color{mplpurple} 0.03}$
    \\
    $\alpha_{\text{$M$,init}}$
    & $0.05$
    & $0.05 \, {\color{mplgreen} \depthcor^{-1 / 2}}$
    & $0.05 \, {\color{mplgreen} \depthcor^{-1}}$
    & $0.05 \, {\color{mplgreen} \depthcor^{-1}}$
    \\
    $\alpha_{\text{$M$,scale}}$
    & $\width^{-1 / 2}$
    & ${\color{mplpurple} 0.03}$
    & ${\color{mplpurple} 0.03}$
    & ${\color{mplpurple} 0.03}$
    \\
    \midrule
    $s_{\text{$qk$,init}}$
    & 1
    & 1
    & 1
    & 1
    \\
    $s_{\text{$qk$,scale}}$
    & $\width^{-1 / 2}$
    & ${\color{mplpurple} 0.03}$
    & ${\color{mplpurple} 0.03}$
    & ${\color{mplpurple} 0.03}$
    \\
    \midrule
    $s_{\text{$u$,init}} = s_{\text{$\nu$,init}}$
    & 1
    & 1
    & 1
    & 1
    \\
    $s_{\text{$u$,scale}} = s_{\text{$\nu$,scale}}$
    & 1
    & 1
    & 1
    & 1
    \\
    \midrule
    $s_{\text{$z$,init}}$
    & 1
    & 1
    & 1
    & $\widthcor^{1 / 2}$
    \\
    $s_{\text{$z$,scale}}$
    & $\width^{-1 / 2}$
    & ${\color{mplpurple} 0.03}$
    & ${\color{mplpurple} 0.03}$
    & ${\color{mplpurple} 0.03}$
    \\
    \midrule
    $\sigma^2_{\text{output}}$
    & arbitrary
    & arbitrary
    & arbitrary
    & arbitrary
    \\
    $\eta_{\text{output}}$
    & $\eta_{\text{base}}$
    & $\eta_{\text{base}} {\color{mplorange} \widthcor^{- 1 / 2}}$
    & $\eta_{\text{base}} {\color{mplorange} \widthcor^{- 1 / 2}}$
    & $\eta_{\text{base}} {\color{mplorange} \widthcor^{- 3 / 4}}$
    \\
    \bottomrule
  \end{NiceTabular}
  \caption{
    Our re-parametrization of nGPT (defined in \cref{sec:ngpt-def}).
    \Cref{sec:summary-of-changes} contains the text version of this table and defines
    \textcolor{mplblue}{data $\datacor$}, \textcolor{mplorange}{width $\widthcor$} and \textcolor{mplgreen}{depth $\depthcor$} ratios.
    Other notations: $\eta_{\text{base}}$ is the learning rate used by the optimizer,
    $\eta_{\text{input}}$ is the learning rate of $\mE_{\text{input}}$ weights ($\sigma^2_{\text{input}}$ their initialization variance),
    $\eta_{\text{output}}$ of $\mE_{\text{output}}$ weights ($\sigma^2_{\text{output}}$ their initialization variance),
    $\eta_{\text{hidden}}$ of all linear layers in Transformer blocks ($\sigma^2_{\text{hidden}}$ their initialization variance).
    \label{tab:our-param}
  }
\end{table}

Let us fix some \textbf{base} number of iterations, depth (number of layers) and width (model dimension). These will be constant throughout, and we define
\begin{equation*}
\datacor := \frac{\text{iter. count}}{\text{base iter. count}}, \quad \widthcor := \frac{\text{width}}{\text{base width}}, \quad \depthcor := \frac{\text{depth}}{\text{base depth}}
\end{equation*}
for the data, width, and depth multipliers of the target model that we wish to train. Our $\nu$GPT parameterization (\cref{tab:our-param}) asks the following:

\begin{roundnote}
\begin{enumerate}
\item Multiply the global learning rate by $\datacor^{-1/3}$.
\item Multiply further the learning rate of embedding weights
  $\eta_{\text{input}}$
  by $\widthcor^{-1 / 2}$.
\item Multiply the learning rate $\eta_{\text{hidden}}$ for weights matrices in MLP and Attention blocks and the learning rate $\eta_{\text{output}}$ for unembedding weights by $\widthcor^{-3 / 4}$.
\item Optionally, multiply $\eta_{\text{input}}$, $\eta_{\text{output}}$ by additional constant factors tuned on the base model.
\item Do not scale $\alpha_{\text{$A$,scale}}$, $\alpha_{\text{$M$,scale}}$, $s_{\text{$qk$,scale}}$, $s_{\text{$z$,scale}}$: take them to be constant $0.03$ rather
  than $\dmodel^{- 1 / 2}$ as in original nGPT.
\item Scale $\alpha_{\text{$A$,init}}$ and $\alpha_{\text{$M$,init}}$ with depth: put them at $0.05 \, \depthcor^{-1}$ rather than $0.05$ as in original nGPT.
\end{enumerate}
\end{roundnote}

The initialization of embeddings, unembeddings and hidden weights is Gaussian with variances
$\sigma^2_{\text{input}}$, $\sigma^2_{\text{output}}$ and $\sigma^2_{\text{hidden}}$
respectively. It is not important what the actual values of these variances are because the corresponding matrices are normalized
before starting each optimization step (including the first one), which is why we write ``arbitrary'' for their values in the table.

\subsection{Width and depth corrections}\label{sec:width-depth-corr}

In this section we provide theoretically grounded heuristics for transferring learning rates across width and depth. As in prior work \cite{pmlr-v139-yang21c} and \cite{dey2025dont} we proceed by formulating two intuitive desiderata constraining initialization and learning rates as functions of depth and width,
meant to give a simplified picture of what a correct parametrization must look like.
We then give heuristic computations that explain how these desiderata can be fulfilled, leading to the $\nu$GPT prescription described above.
Detailed estimation of scales of initial values and updates inside each block are deferred to \cref{sec:width-deriv}.

\paragraph{Notation.}
For any vector or matrix in the network $\mP(t)$, we denote by $\mP$ its value $\mP(0)$  at initialization and by $\Delta \mP (t)$ or just $\Delta \mP$ its change from initialization $\mP(t) - \mP(0)$. In this section, the number of optimization steps $t$ is assumed to not exceed a universal constant (we relax this in \cref{sec:token-horizon-corr}). For quantities $a$ and $b$ depending on width and depth, we write $a = O(b)$, or $a \lesssim b$, or $b = \Omega(a)$ to mean that the ratio $a / b$ ``does not explode''
(does not grow unbounded at growing width and depth). The notation $a = \Theta(b)$ and $a \asymp b$ both mean that simultaneously $a \lesssim b$ and $a \gtrsim b$.

\paragraph{HP Transfer Desiderata.}

Our first desideratum for HP transfer is borrowed from \cite{pmlr-v139-yang21c} and asks that all residual layer preactivations are order-$1$.

\begin{desideratum}[Stability at initialization]\label{des:stab-init}
  We call a parametrization \textbf{stable at initialization} if
  no preactivations blow up or become trivial,
  and the logits do not blow up\footnote{It is a technicality of the \mup initialization that logits converge to zero at initialization, though they evolve non-trivially, thus becoming $\Theta(1)$ during training. We do not recommend $\mu$P as our primary choice, but we do not have reason to exclude it.}:
  $\norm{\vh^{\ell}_n} = \Theta(1)$ for each $\ell \in \range{1}{L}$ and $\norm{\vz_n} = O(1)$.
\end{desideratum}

\Cref{des:stab-init} is essentially automatic for nGPT because of the
$\text{Norm}(\cdot)$ operations in the forward pass and normalization of matrices at each step.

For our second desideratum we again follow ideas from \cite{pmlr-v139-yang21c} and seek parameterizations in which preactivations change order-$1$ from each step of training, independent of depth and width.

\begin{desideratum}[Stable non-trivial feature learning]\label{des:feature-learn}
  We say that a parametrization achieves \textbf{stable non-trivial feature learning} if the logits and all preactivations
  evolve non-trivially without blowing up:
  $\norm{\Delta \vh_n^{\ell}} = \Theta(1)$ and $\norm{\Delta \vz_n} = \Theta(1)$,
  and the learning rates $\eta_{\text{input}}$, $\eta_{\text{hidden}}$, $\eta_{\text{output}}$ have the largest scales possible without
  breaking this constraint.
\end{desideratum}

It is a robust empirical observation that desiderata such as the ones above lead to learning rate transfer, even though it is not completely clear why \citep{dey2025dont, pmlr-v139-yang21c, everett2024scalingexponentsacross}.
We present a simple if tedious calculation in \cref{sec:width-deriv} for why \cref{des:feature-learn} holds in our $\nu$GPT parameterization at growing model width and fixed depth. While we defer the full details to the appendix, we focus here on illustrating the main ideas, especially the somewhat unusual $m_{\text{width}}^{-3/4}$ scaling of the hidden layer and unembedding learning rates.

\paragraph{Summary derivations of width corrections}
We start with recalling the important notion of alignment exponents introduced in \cite{everett2024scalingexponentsacross}.

\begin{definition}[Alignment exponents]\label{def:alignment-exponents}
Define $\alpha_{\text{output}}, \omega_{\text{output}}, \nu_{\text{output}} \in [0, 1]$ to be such that
\begin{align*}
  \frac{\norm{\Delta \mE_{\text{output}} \vh}}{\sqrt{V}}
  &\asymp d_{\text{model}}^{\alpha_{\text{output}}} \frac{\norm{\Delta \mE_{\text{output}}}_F}{\sqrt{V d_{\text{model}}}} \frac{\norm{\vh}}{\sqrt{d_{\text{model}}}},\\
  \frac{\norm{\mE_{\text{output}} \Delta \vh}}{\sqrt{V}}
  &\asymp d_{\text{model}}^{\omega_{\text{output}}} \frac{\norm{\mE_{\text{output}}}_F}{\sqrt{V d_{\text{model}}}} \frac{\norm{\Delta \vh}}{\sqrt{d_{\text{model}}}},\\
  \frac{\norm{\Delta \mE_{\text{output}} \Delta \vh}}{\sqrt{V}}
  &\asymp d_{\text{model}}^{\nu_{\text{output}}} \frac{\norm{\Delta \mE_{\text{output}}}_F}{\sqrt{V d_{\text{model}}}} \frac{\norm{\Delta \vh}}{\sqrt{d_{\text{model}}}}.
\end{align*}

Additionally, define $\alpha_{\text{hidden}}, \omega_{\text{hidden}}, \nu_{\text{hidden}} \in [0, 1]$ to be such that
\begin{align*}
  \frac{\norm{\Delta \mW \vh}}{\sqrt{d_{\text{out}}}}
  &\asymp d_{\text{in}}^{\alpha_{\text{hidden}}} \frac{\norm{\Delta \mW}_F}{\sqrt{d_{\text{out}} \, d_{\text{in}}}}  \frac{\norm{\vh}}{\sqrt{d_{\text{in}}}},\\
  \frac{\norm{\mW \Delta \vh}}{\sqrt{d_{\text{out}}}}
  &\asymp d_{\text{in}}^{\omega_{\text{hidden}}} \frac{\norm{\mW}_F}{\sqrt{d_{\text{out}} \, d_{\text{in}}}}  \frac{\norm{\Delta \vh}}{\sqrt{d_{\text{in}}}},\\
  \frac{\norm{\Delta \mW \Delta \vh}}{\sqrt{d_{\text{out}}}}
  &\asymp d_{\text{in}}^{\nu_{\text{hidden}}} \frac{\norm{\Delta \mW}_F}{\sqrt{d_{\text{out}} \, d_{\text{in}}}}  \frac{\norm{\Delta \vh}}{\sqrt{d_{\text{in}}}}
\end{align*}
for any hidden weight $\mW \in \mathbb{R}^{d_{\text{out}} \times d_{\text{in}}}$.

Although we suppress this from the notation,  alignment exponents in principle depend on step number, token position, layer number.
\end{definition}

The actual dynamics of the alignment exponents in large nGPT models is complicated and will be explored in \cref{sec:alignment-exponents}.
Unsurprisingly, some of them change during training. However, in order for the theory to be tractable, we assume them to take fixed values.
We postpone further discussion of this, and for now just list the recommended values. It is observed decisively that $\omega_{\text{hidden}} = \omega_{\text{output}} = 1 / 2$ (see \cref{fig:alignment_val_loss_weighted}). Further, we introduce the following definition for clarity.

\begin{definition}[Full, no, mid alignment]\label{def:full-no-mid-align}
  We introduce the following assumptions:
\begin{itemize}
\item \textbf{no alignment assumption:} $\max\crl{\alpha_{\text{hidden}}, \nu_{\text{hidden}}} = \max \crl{\alpha_{\text{output}}, \nu_{\text{output}}} = 1 / 2$,
\item \textbf{full alignment assumption:} $\max\crl{\alpha_{\text{hidden}}, \nu_{\text{hidden}}} = \max \crl{\alpha_{\text{output}}, \nu_{\text{output}}} = 1$,
\item \textbf{mid alignment assumption:}
  $\max\crl{\alpha_{\text{hidden}}, \nu_{\text{hidden}}} = \max \crl{\alpha_{\text{output}}, \nu_{\text{output}}} = 3 / 4$.
\end{itemize}
\end{definition}

Our parametrization in \cref{tab:our-param} assumes \textbf{mid alignment} because it achieves perfect transfer over width and is consistent with measurements (\cref{sec:alignment-exponents}). Moreover, we make predictions for HP transfer in Adam by studying signGD. Our motivation is that the Adam update for a scalar weight component $w$ takes the form
\begin{equation*}
w(t + 1) = w(t) - \eta \frac{\frac{1 - \beta_1}{1 - \beta_1^{t + 1}} \sum_{\tau = 0}^t \beta_1^{t - \tau} g(\tau)}{\sqrt{\frac{1 - \beta_2}{1 - \beta_2^{t + 1}} \sum_{\tau = 0}^t \beta_2^{t - \tau} g^2(\tau)} + \epsilon},
\end{equation*}
where $g$ the corresponding gradient component, $\eta$ the corresponding learning rate, $t$ the time step, $\beta_1, \beta_2, \epsilon$ the Adam's hyperparameters. This can be viewed as a smoothed version of the signGD update
\begin{equation*}
w(t + 1) = w(t) - \eta \frac{g(t)}{\sqrt{g^2(t)}} = w(t) - \eta \sign g(t).
\end{equation*}
In particular, we assume ``signGD-like'' scaling of updates:
\begin{equation*}
w(t + 1) - w(t) \asymp \eta.
\end{equation*}
For small enough updates (which will be the case in our derivations), the renormalization of matrices at each step does not influence this scale.

The following derivations, designed to avoid violating \cref{des:feature-learn}, motivate the choice of width corrections.

\begin{itemize}
\item What should we do to make
  the embedding update contribute non-trivially without blowup, that is,
  for
  $\norm{\Delta \mE_{\text{input}} \vx_n}$ to have scale $\Theta(1)$?
Since $\vx_n$ is one-hot, $\Delta \mE_{\text{input}} \vx_n$ is a column of $\Delta \mE_{\text{input}}$, so this norm has
scale $\sqrt{d_{\text{model}}} \eta_{\text{input}}$ because Adam updates have signGD-like scaling (with each component almost $\pm \eta_{\text{input}}$).
This is why we need $\eta_{\text{input}}$ to be proportional to $d_{\text{model}}^{- 1 / 2}$, justifying the $\widthcor^{-1 / 2}$ correction in $\eta_{\text{input}}$.

\item What should we do to ensure $\norm{\Delta (\mW \vh_n)}$ have scale $\Theta(1)$ for some hidden weight matrix $\mW \in \mathbb{R}^{\dout \times \din}$ inside one of the Transformer blocks, where $\dout \asymp \din \asymp \dmodel$? Using the definition of $\alpha_{\text{hidden}}$, we have
\begin{align*}
  \frac{\norm{\Delta \mW \vh_n}}{\sqrt{\dmodel}}
  \asymp \dmodel^{\alpha_{\text{hidden}}} \frac{\norm{\Delta \mW}_F}{\dmodel} \frac{\norm{\vh_n}}{\sqrt{\dmodel}}
  \asymp \dmodel^{\alpha_{\text{hidden}} - 1 / 2} \eta_{\text{hidden}},
\end{align*}
where the last equivalence is true because the quadratic mean of $\Delta \mW$ components $\norm{\Delta \mW}_F / \dmodel$ scales like $\eta_{\text{hidden}}$ (Adam updates have signGD-like scaling). Similarly,
\begin{align*}
  \frac{\norm{\mW \Delta \vh_n}}{\sqrt{\dmodel}}
  \asymp \dmodel^{\omega_{\text{hidden}}} \frac{\norm{\mW}_F}{\dmodel} \frac{\norm{\Delta \vh_n}}{\sqrt{\dmodel}}
  \asymp \dmodel^{- 1 / 2},
\end{align*}
where in the last equivalence we used $\omega_{\text{hidden}} = 1 / 2$, the fact that rows of $\mW$ are of norm 1 and assumed that
the movement of the previous hidden state matches itself in scale, that is, $\norm{\Delta \vh_n} \asymp \norm{\vh_n} \asymp 1$. Finally,
\begin{align*}
\frac{\norm{\Delta \mW \Delta \vh_n}}{\sqrt{\dmodel}} \asymp \dmodel^{\nu_{\text{hidden}} - 1 / 2} \eta_{\text{hidden}}
\end{align*}
by similar arguments. Combining, we obtain the maximal hidden learning rate $\eta_{\text{hidden}} \asymp \dmodel^{- \max\crl{\alpha_{\text{hidden}}, \nu_{\text{hidden}}}} = \dmodel^{- 3 / 4}$,
justifying the $\widthcor^{- 3 / 4}$ (mid alignment) correction in $\eta_{\text{hidden}}$.

\item
  What should we do if we want to ensure $\norm{\Delta (\mE_{\text{output}} \vh_n)} \asymp \norm{\mE_{\text{output}} \vh_n}$?
  The alignment exponent at initialization is always $1 / 2$ by independence:
  \begin{equation*}
    \frac{\norm{\mE_{\text{output}} \vh_n}}{\sqrt{V}}
    \asymp \sqrt{d_{\text{model}}} \frac{\norm{\mE_{\text{output}}}_F}{\sqrt{d_{\text{model}} V}} \frac{\norm{\vh_n}}{\sqrt{d_{\text{model}}}}
    = \frac{\norm{\mE_{\text{output}}}_F}{\sqrt{d_{\text{model}} V}} = d_{\text{model}}^{- 1 / 2}.
  \end{equation*}
  Next, similarly to the above, we have
\begin{align*}
  &\frac{\norm{\Delta \mE_{\text{output}} \vh_n}}{\sqrt{V}} \asymp d_{\text{model}}^{\alpha_{\text{output}}} \eta_{\text{output}} \frac{\norm{\vh_n}}{\sqrt{d_{\text{model}}}} = d_{\text{model}}^{\alpha_{\text{output}} - 1 / 2} \eta_{\text{output}};\\
  &\frac{\norm{\mE_{\text{output}} \Delta \vh_n}}{\sqrt{V}} \asymp d_{\text{model}}^{\omega_{\text{output}}} \frac{\norm{\mE_{\text{output}}}_F}{\sqrt{V d_{\text{model}}}} \frac{\norm{\Delta \vh_n}}{\sqrt{d_{\text{model}}}}
  \asymp d_{\text{model}}^{\omega_{\text{output}} - 1 / 2} \frac{\norm{\mE_{\text{output}}}_F}{\sqrt{V d_{\text{model}}}} \asymp d_{\text{model}}^{\omega_{\text{output}} - 1};\\
  &\frac{\norm{\Delta \mE_{\text{output}} \Delta \vh_n}}{\sqrt{V}}
  \lesssim d_{\text{model}}^{\nu_{\text{output}}} \eta_{\text{output}} \frac{\norm{\Delta \vh_n}}{\sqrt{d_{\text{model}}}}
  = d_{\text{model}}^{\nu_{\text{output}} - 1 / 2} \eta_{\text{output}}.
\end{align*}
Recalling that $\omega_{\text{output}} = 1  /2$, we should set
$\eta_{\text{output}} \lesssim \dmodel^{- \max\crl{\alpha_{\text{output}}, \nu_{\text{output}}}}$,
justifying the width correction $\widthcor^{-3 / 4}$ (mid alignment) in $\eta_{\text{output}}$.
\end{itemize}

\paragraph{Depth corrections}
The preceding discussion concerned HP transfer across width in nGPT. Let us also briefly discuss how to scale learning rates and parameterization with growing depth. It is common in the literature to consider a residual network like
\begin{equation}\label{eq:TTiUGU}
\vh^{\ell + 1} = \vh^{\ell} + n_{\text{layers}}^{-\depthalph} \mathcal{F}_{\ell}(\vh^{\ell}),
\end{equation}
where $\mathcal{F}_{\ell}$ is the $\ell$th block, and $\depthalph$ is the \textbf{depth alpha}. The updates
\begin{equation*}
\Delta \vh^{\ell + 1} = \Delta \vh^{\ell} + n_{\text{layers}}^{-\depthalph} \Delta \mathcal{F}_{\ell}(\vh^{\ell})
\end{equation*}
accumulate over $\depth$ layers, so one can require
\begin{equation*}
n_{\text{layers}}^{-\depthalph} \Delta \mathcal{F}_{\ell}(\vh^{\ell}) \lesssim \frac{1}{\depth}
\end{equation*}
to prevent updates from blowing up (\cref{des:feature-learn}). If $\Delta \mathcal{F}_{\ell}(\vh^{\ell}) \asymp \eta_{\text{hidden}}$ which is typically the case for Adam
because of signGD-like scaling,
this means we may need to introduce a depth correction into the hidden learning rate:
\begin{equation}\label{eq:LWphMK}
\eta_{\text{hidden}} \asymp n_{\text{layers}}^{\depthalph - 1},
\end{equation}
unless $\depthalph = 1$ (in which case no such correction is required).
The Depth-$\mu$P \citep{yang2024tensor} style parametrization sets $\depthalph = 1 / 2$ (and therefore $\eta_{\text{hidden}} \asymp n_{\text{layers}}^{- 1 / 2}$)
based on additional desiderata such as feature diversity,
whereas CompleteP \citep{dey2025dont} sets $\depthalph = 1$ based on a different additional desideratum
(``complete'' feature learning: hidden layers and model output are not near-linear in any of the model parameters).

The Normalized Transformer differs significantly from \eqref{eq:TTiUGU},
so we conduct our own theoretical analysis in \cref{sec:depth-scaling-deriv}.
We find that for a simple network closer to nGPT, the safe hidden learning rate is also given by \eqref{eq:LWphMK}.
To keep things simple, we just introduce a correction $\depthcor^{-1}$ into $\alpha_{\text{$A$,init}}$ and $\alpha_{\text{$M$,init}}$,
which roughly corresponds to $\depthalph = 1$ (at least at initialization).
Interestingly, our ablations in \cref{sec:experiments} show that,
(a)~baseline nGPT already shows decent transfer over depth
(perhaps because $\valpha_A$ and $\valpha_M$ are trainable and initialized with a small enough value, which means $\depth$ needs to be extremely large for depth corrections to matter),
(b)~the hidden learning rate should not be changed even if a $\depthcor^{- 1 / 2}$ correction in $\alpha_{\text{$A$,init}}$ and $\alpha_{\text{$M$,init}}$ is used instead. Explaining this phenomenon theoretically can be an interesting future direction.

\subsection{Token horizon corrections}\label{sec:token-horizon-corr}

In \cref{sec:width-depth-corr}, the number of steps was bounded.
In practice, however, the token horizon is typically scaled along with the model size; hence,
we need to account for the growing number of steps.
It is a robust finding in recent literature that the optimal learning rates decrease as the token horizon increases
at a fixed batch size
\citep{everett2024scalingexponentsacross,bjorck2025scaling,mlodozeniec2025completedhyperparametertransfer}.
There is no well-known and well-tested theoretical framework predicting the correct scaling. The work \citet{mlodozeniec2025completedhyperparametertransfer} uses an SDE perspective from \citet{malladi2022on} to derive the data correction $\datacor^{- 1 / 2}$.
\Citet{shulgin2026deriving} also predict this data correction via a different theoretical method.
However, we find that the exponent $1 / 2$ is too large.
Instead, we base our choice $\datacor^{- 1 / 3}$ on empirical results from \citet{bjorck2025scaling}
(cf. $\beta = 0.32$ there) and our own independent power law fits (\cref{sec:token-count}).
The exponent $1 / 3$ is also concurrently confirmed in \citet{ren2026rethinking} in a very different setting.


\section{Experiments}\label{sec:experiments}

We train nGPT as defined in \cref{sec:ngpt-def} with cross-entropy loss on the FineWeb-Edu dataset \citep{penedo2024the} with sequence length 4096,
fixed batch size 64.
We use the OLMo 2 \citep{olmo20252olmo2} tokenizer\footnote{\url{https://huggingface.co/allenai/OLMo-2-0425-1B}} with vocabulary size $100\,352$.
The ``\textbf{fixed iteration count}'' experiments all use $80\,000$ steps, or $4096 \times 64 \times 80\,000 \approx 21B$ tokens.
We use our own implementation of nGPT in our fork of TorchTitan \citep{liang2025torchtitan}.
The ``\textbf{20 tokens per parameter}'' experiments all use the token count which is $20$ times the number of non-embedding parameters\footnote{or, more specifically, the number of steps corresponding to this token count and rounded up to a multiple of 250}.
This is heuristically treated as the compute-optimal number of tokens in the literature \citep{dey2025dont,wen2025fantasticpretrainingoptimizers}, although
\citet{hoffmann2022empiricalanalysiscompute} include embedding parameters when fitting power laws (see also \citet{pearce2024reconciling} on this topic).
We use Adam (AdamW with weight decay 0.0)
with $\beta_1 = 0.9$, $\beta_2 = 0.95$, $\epsilon = 10^{- 16}$ (following \citet{dey2025dont}), no warm-up, the learning rate decayed to $10\%$ of its peak (initial) value using a cosine schedule. Our base depth and width are $10$.
Unless averaging over seeds, we use validation loss EMA with $\beta = 0.95$.

\subsection{Growing $\nheads$ at a fixed iteration count}\label{sec:growing-nheads-at-fixed-iteration}

We first fix the head dimension at\footnote{$102 = \lfloor 1024 / 10 \rfloor$, where $10$ is the base depth and width} 102 and scale the number of heads,
fixing the number of blocks (depth) at $\depth = 12$. As the number of heads grows from 8 to 40,
the total number of parameters grows from 0.26\,B to 3.22\,B\internalComment{!}.
The baseline nGPT implementation without our reparametrization does \textit{not} show transfer of the learning rate over width
(\cref{fig:width-fixed-steps-baseline}). In the same setting,
our parametrization shows good transfer over the number of heads (\cref{fig:width-fixed-steps-mupsamescalelogits-tuned-midalign}).

\begin{roundnote}
Width corrections (powers of $\widthcor$) are important for width transfer in $\nu$GPT.
\end{roundnote}

\subsection{Growing $\depth$ at a fixed iteration count}\label{sec:depth-corr-ablations}

\begin{figure}[t]
  \centering
  \begin{tabular}{cc}
    \begin{subfigure}[h]{0.48\textwidth}
      \includegraphics[width=\linewidth]{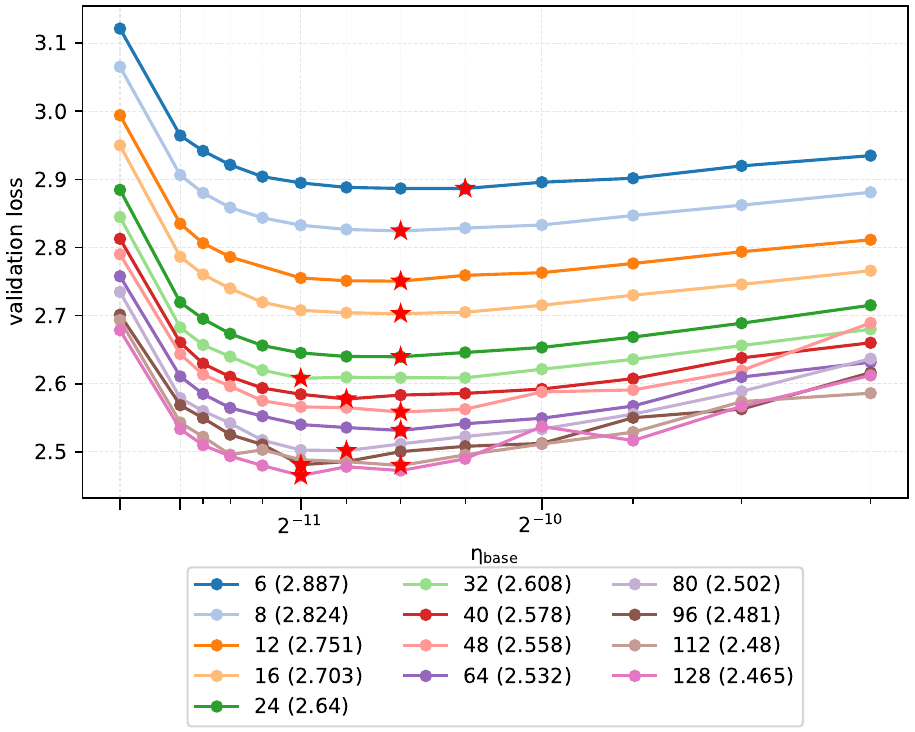}
      \caption{nGPT (baseline)\label{fig:depth-sweep-ngpt}}
    \end{subfigure}
    &
      \begin{subfigure}[h]{0.48\textwidth}
        \includegraphics[width=\linewidth]{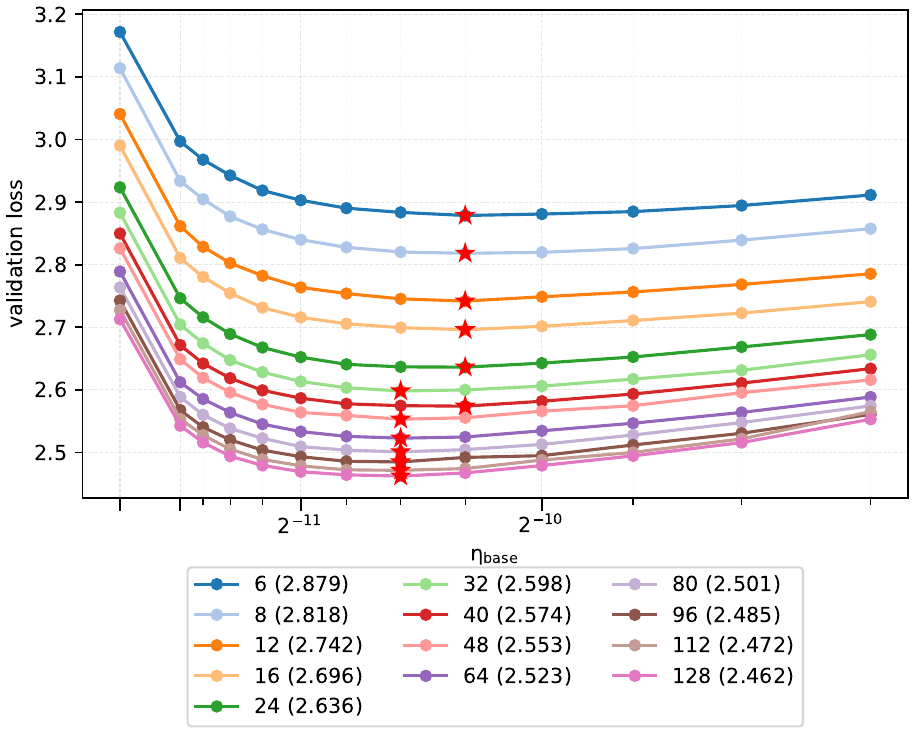}
        \caption{$\nu$GPT (ours)\internalComment{(mid align) with $\eta_{\text{output}}$ multiplied by $2^{-1}$.}\label{fig:depth-sweep-ngpt-samescalelogits-tuned-midalign}}
      \end{subfigure}
  \end{tabular}
  \caption{
    \textbf{Depth} sweeps with $\depth$ increasing (in the legend, with the best loss in parentheses), head dimension and $\nheads = 12$ fixed, 80\,000 iterations (about 21\,B tokens).
    Both baseline nGPT and our CompleteP-inspired \citep{dey2025dont} parametrization show decent transfer, with ours winning slightly in stability and performance.
    \label{fig:depth-sweeps}
    \internalComment{See \cref{exp:depth-sweeps-at-fixed-step-count}.}}
\end{figure}

In this set of sweeps, we fix the head dimension (at the same value 102 as above) and the number of heads at 12 (so that $d_{\text{model}} = 1224$), and scale the number of blocks (depth) $\depth$. As we vary $\depth$ from 8 to 128, the number of parameters in the model grows from 0.39\,B to 2.55\,B.\internalComment{!}

Interestingly, baseline nGPT already shows decent transfer over depth at a fixed iteration count (\cref{fig:depth-sweep-ngpt}),
although models with the highest depth become somewhat unstable.
Therefore, the minimal depth corrections in \cref{sec:mungpt}, although theoretically well-motivated,
may not be necessary. In addition, we observe that the average $\valpha_A$ and $\valpha_M$ components
(over all components and all blocks) at the end of training decrease at power laws in depth close to $\depth^{-0.5}$
(see \cref{fig:depth-sweep-baseline-alpha-power-law} for precise formulas),
which suggests that among $\depthalph \in \crl{0, 1 / 2, 1}$ (defined in \cref{sec:width-depth-corr}),
the trained baseline Normalized Transformer prefers $\depthalph = 1 / 2$ on average (it should be noted that the variance across components and blocks is very high,
as seen from error bars).

We also verify that our parametrization gives good transfer over depth in \cref{fig:depth-sweep-ngpt-samescalelogits-tuned-midalign}.

\begin{roundnote}
  Both baseline nGPT and our reparametrization show good learning rate transfer over depth.
  Our parametrization shows lower learning rate sensitivity
  at large depths.
\end{roundnote}

\subsection{Alignment exponents}\label{sec:alignment-exponents}

We show in \cref{sec:width-deriv} (sketch in \cref{sec:width-depth-corr}) that if we assume fixed values for alignment exponents (\cref{def:alignment-exponents}), then the maximal stable learning rates should scale as
\begin{equation*}
\eta_{\text{hidden}} \asymp \dmodel^{- \max\crl{\alpha_{\text{hidden}}, \nu_{\text{hidden}}}}, \quad \eta_{\text{output}} \lesssim \dmodel^{- \max\crl{\alpha_{\text{output}}, \nu_{\text{output}}}}.
\end{equation*}
Recall\internalComment{!} the ``full alignment'' regime where $\max\crl{\alpha_{\text{hidden}}, \nu_{\text{hidden}}} = \max\crl{\alpha_{\text{output}}, \nu_{\text{output}}} = 1$ and the ``no alignment''
regime where $\max\crl{\alpha_{\text{hidden}}, \nu_{\text{hidden}}} = \max\crl{\alpha_{\text{output}}, \nu_{\text{output}}} = 1 / 2$.
We measure the alignment exponents in a representative run of $\nu$GPT with $\depth = \nheads = 12$
(see \cref{sec:alignment-exponents-for-wider-model} for a different size) and
see that neither of these two regimes perfectly match observation.
\Cref{fig:alignment_val} shows that all $\alpha$ and $\nu$ vary significantly during training, with $\alpha_{\text{hidden}}, \nu_{\text{hidden}}$ and $\alpha_{\text{output}}, \nu_{\text{output}}$ much higher during a short period at the beginning than at the end (where they settle close to 0.5--0.6).

We find that the exact middle of these two regimes gives essentially perfect transfer over width:
\begin{equation*}
\max\crl{\alpha_{\text{hidden}}, \nu_{\text{hidden}}} = \max\crl{\alpha_{\text{output}}, \nu_{\text{output}}} = 3 / 4,
\end{equation*}
and we make this choice that in our final recommendation.
This setting is close to observed average alignment exponents if we weight them by the loss decrease (\cref{fig:alignment_val_loss_weighted}).
This naturally gives more weight to the beginning of training. Although this is not obviously the right weighting, it is reasonable because loss decrease may parametrize the importance of updates, and it is consistent with the $\mu$P methodology which is focused on the first one or two steps.

In addition, the plots reported here show definitively that the correct value of $\omega_{\text{hidden}}$ and $\omega_{\text{output}}$ is one half:
\begin{equation*}
\omega_{\text{hidden}} = \omega_{\text{output}} = 1 / 2.
\end{equation*}
This confirms the point made theoretically in \citet{everett2024scalingexponentsacross}
that the assumption $\omega_{\text{output}} = 1$ (uniquely motivating $\mu$P)
is too conservative.

\begin{figure}[t]
  \centering
  \begin{subfigure}[t]{0.32\textwidth}
    \centering
    \includegraphics[width=\textwidth]{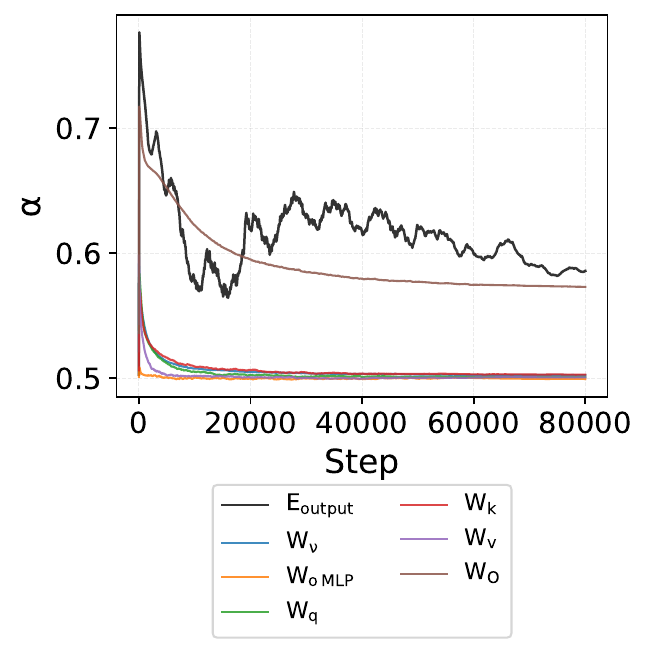}
    \caption{Alignment exponent $\alpha$}
  \end{subfigure}
  \hfill
  \begin{subfigure}[t]{0.32\textwidth}
    \centering
    \includegraphics[width=\textwidth]{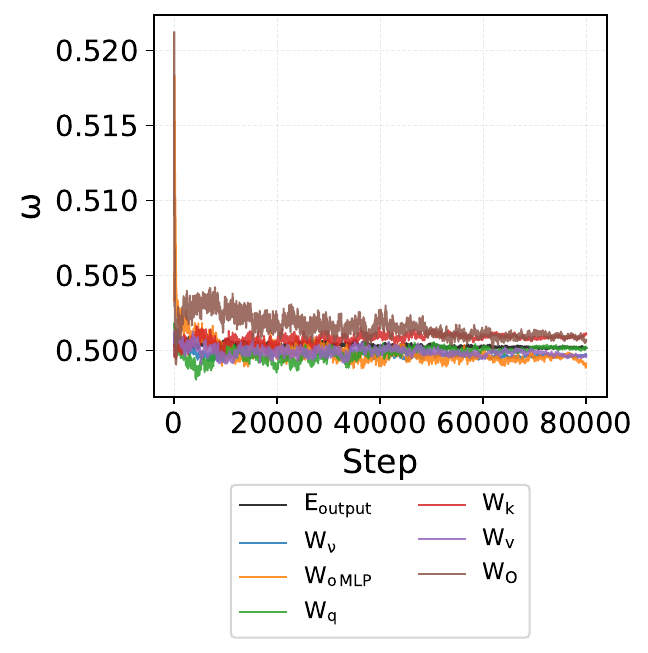}
    \caption{Alignment exponent $\omega$}
  \end{subfigure}
  \hfill
  \begin{subfigure}[t]{0.32\textwidth}
    \centering
    \includegraphics[width=\textwidth]{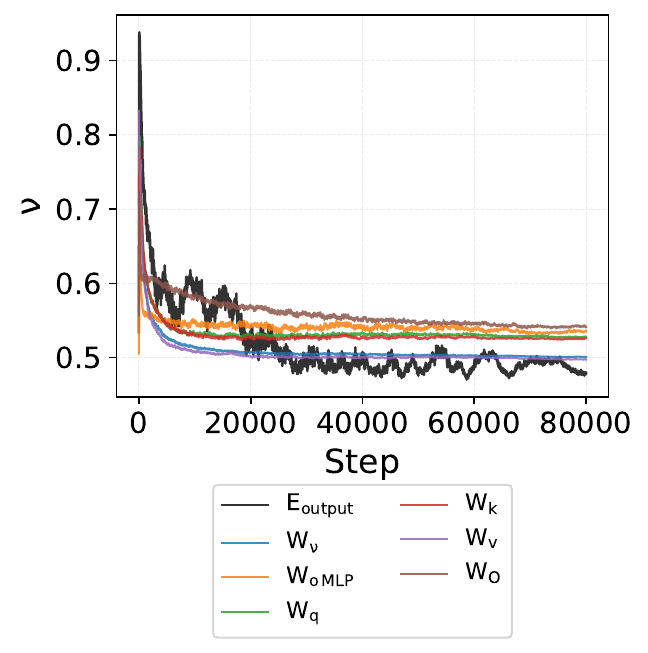}
    \caption{Alignment exponent $\nu$}
  \end{subfigure}
  \hfill
  \caption{Alignment exponents (\cref{def:alignment-exponents}) of $\nu$GPT with $\depth = \nheads = 12$ on a fixed validation batch, averaged over layers.}
  \label{fig:alignment_val}
\end{figure}

\begin{figure}[t]
  \centering
  \begin{subfigure}[t]{0.32\textwidth}
    \centering
    \includegraphics[width=\textwidth]{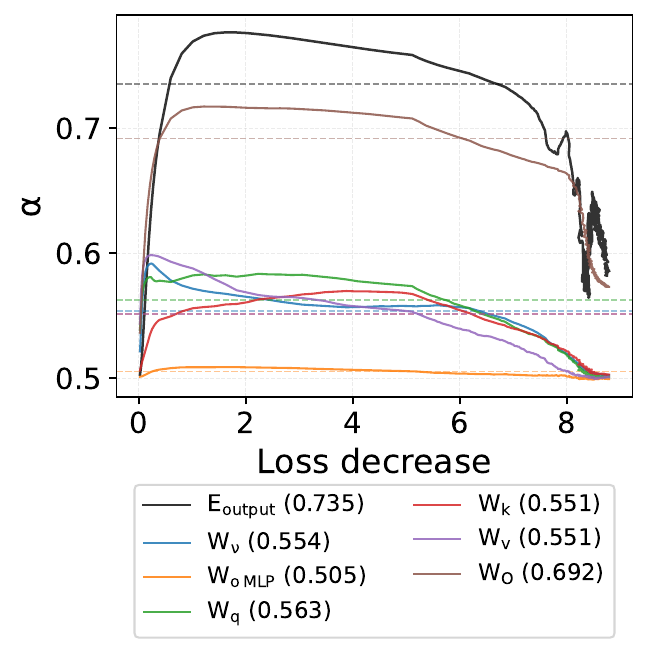}
    \caption{Alignment exponent $\alpha$}
  \end{subfigure}
  \hfill
  \begin{subfigure}[t]{0.32\textwidth}
    \centering
    \includegraphics[width=\textwidth]{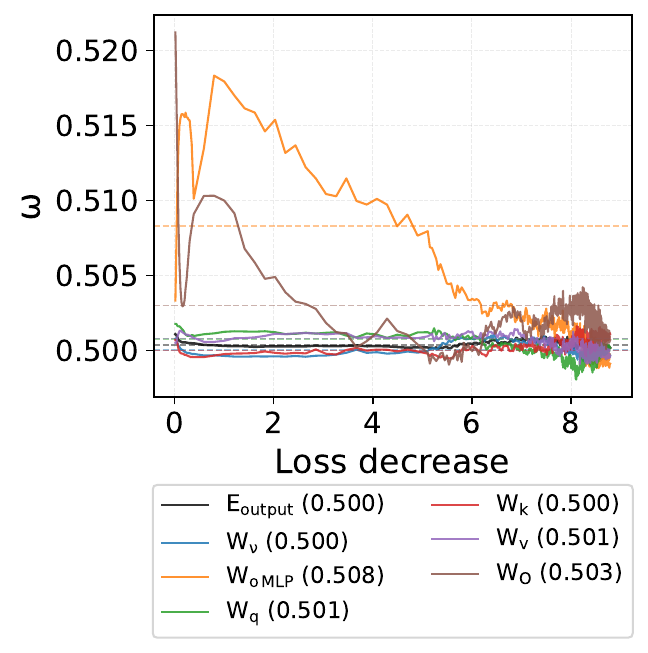}
    \caption{Alignment exponent $\omega$}
  \end{subfigure}
  \hfill
  \begin{subfigure}[t]{0.32\textwidth}
    \centering
    \includegraphics[width=\textwidth]{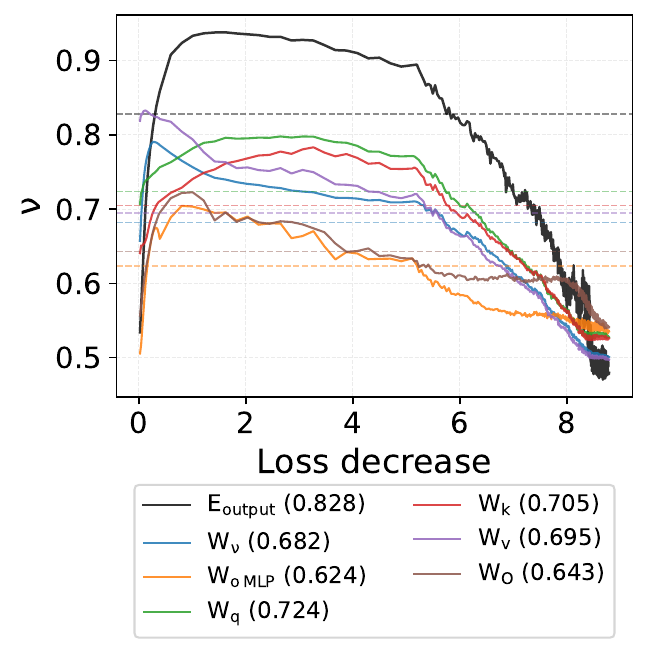}
    \caption{Alignment exponent $\nu$}
  \end{subfigure}
  \hfill
  \caption{Alignment exponents (\cref{def:alignment-exponents}) of $\nu$GPT with $\depth = \nheads = 12$ on a fixed validation batch, averaged over layers,
   viewed as a function of loss decrease. The mid alignment assumption (\cref{def:full-no-mid-align}) matches the observations well.}
  \label{fig:alignment_val_loss_weighted}
\end{figure}

\begin{figure}[t]
  \centering
  \begin{subfigure}[t]{0.32\textwidth}
    \centering
    \includegraphics[width=\textwidth]{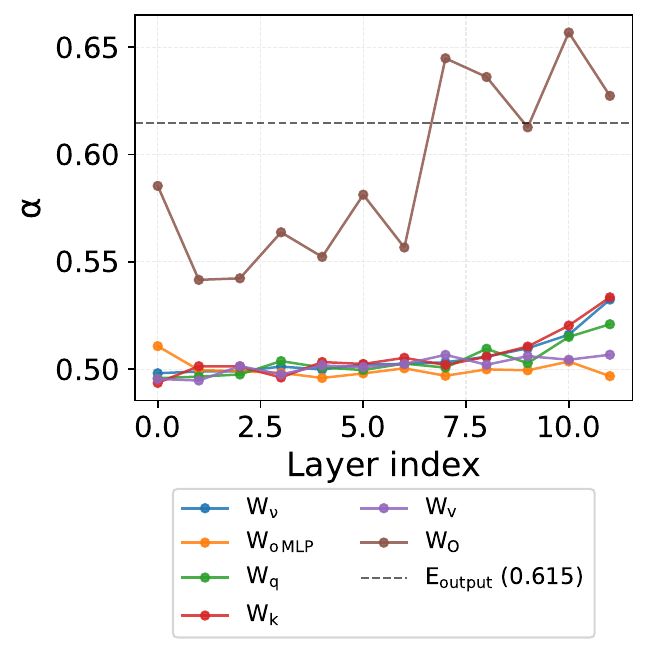}
    \caption{Alignment exponent $\alpha$}
  \end{subfigure}
  \hfill
  \begin{subfigure}[t]{0.32\textwidth}
    \centering
    \includegraphics[width=\textwidth]{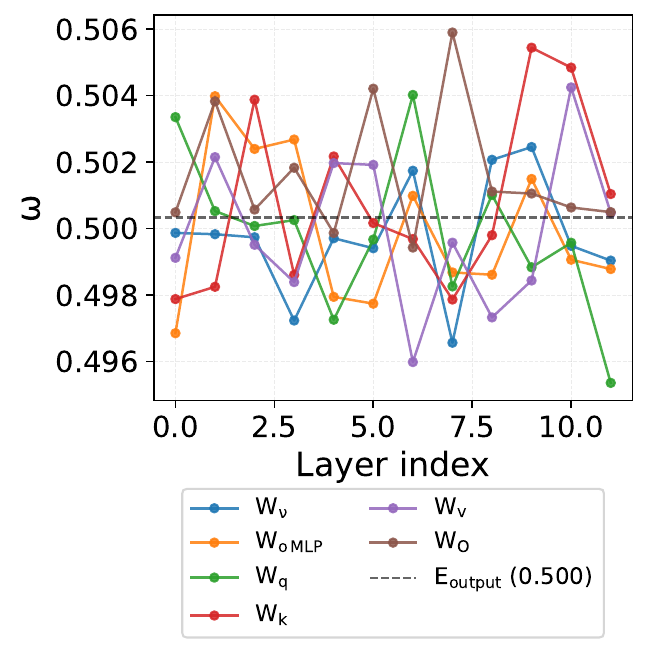}
    \caption{Alignment exponent $\omega$}
  \end{subfigure}
  \hfill
  \begin{subfigure}[t]{0.32\textwidth}
    \centering
    \includegraphics[width=\textwidth]{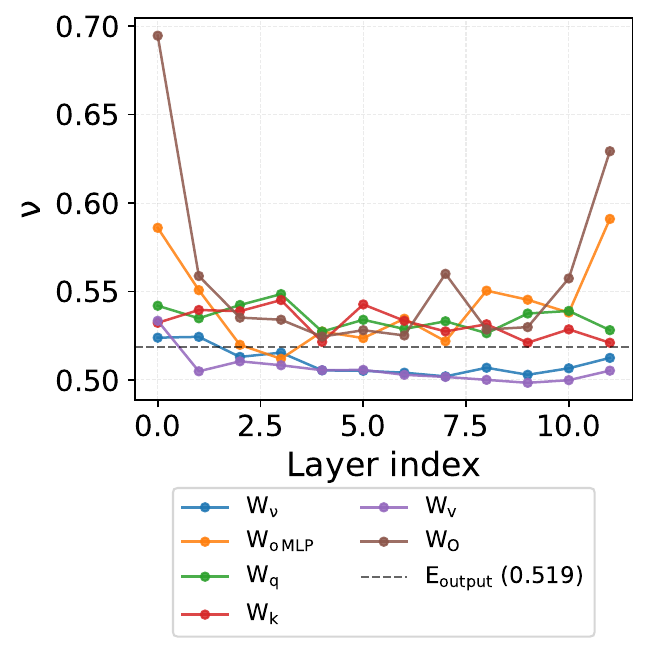}
    \caption{Alignment exponent $\nu$}
  \end{subfigure}
  \hfill
  \caption{Alignment exponents (\cref{def:alignment-exponents}) of $\nu$GPT with $\depth = \nheads = 12$ on a fixed validation batch vs. layer index, averaged over steps.}
  \label{fig:alignment_val_layer_dependence}
\end{figure}


\subsection{$\nu$GPT gives better transfer over width than $\mu$P}\label{sec:nu-gpt-gives-better-transfer-over-width-than-mup}

We ablate our carefully tuned alignment exponents by comparing with standard $\mu$P-style
width corrections from prior literature. \Cref{fig:mup-drifts-ours-not} shows that $\nu$GPT gives somewhat better transfer over width.
We do not interpret this as falsifying $\mu$P (it still gives decent transfer)
but rather making some corrections in the theory. In particular, we confirm the point made in \citet{everett2024scalingexponentsacross}
that $\mu$P is by no means the unique maximal update parametrization achieving feature learning and giving
learning rate transfer (in particular, our parametrization is somewhat different from any parametrization in prior work).

\begin{figure}[h!]
  \centering

  \begin{subfigure}[h]{0.48\textwidth} \includegraphics[width=\linewidth]{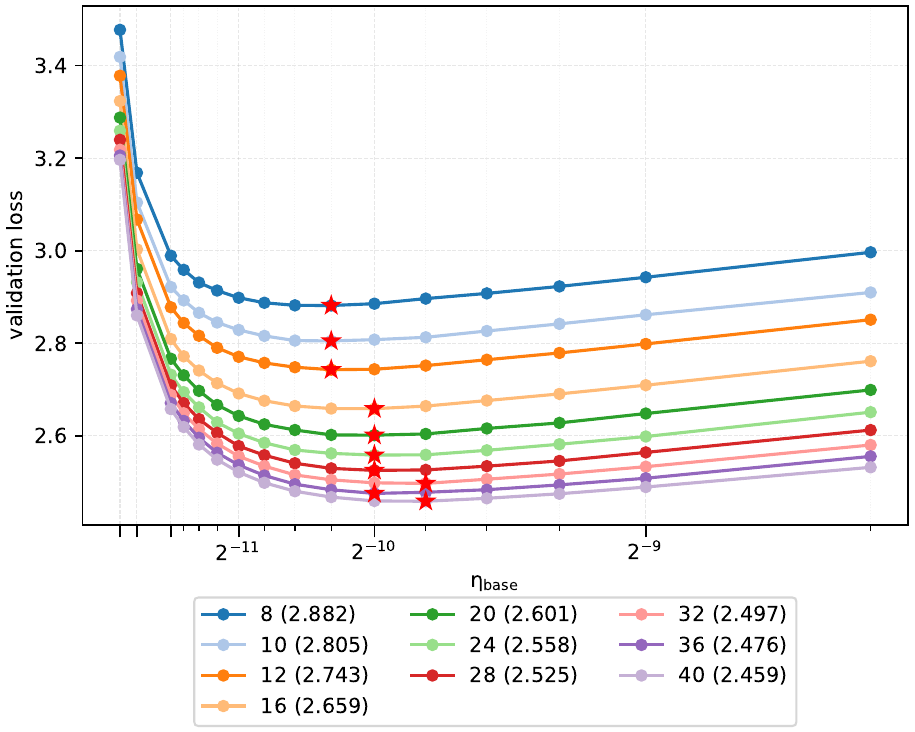}
\caption{nGPT with $\mu$P-style width corrections (``CompleteP'' in \cref{tab:our-param})\internalComment{$\mu$nGPT (full align) with $\eta_{\text{output}}$ multiplied by $2^{-1.5}$}}
\end{subfigure}
\hfill
\begin{subfigure}[h]{0.48\textwidth} \includegraphics[width=\linewidth]{width_steps80000_sqkBaseScale0.03_szScale0.03_alphaScale0.03_eigenBaseInit0.05_mup_sameScaleLogits_etaOut0.5_etaOutPow0.75_etaHidPow0.75_val_loss.pdf}
  \caption{$\nu$GPT (ours)\internalComment{(mid align) with $\eta_{\text{output}}$ multiplied by $2^{-1}$}}
\end{subfigure}

\caption{The number of heads is swept as in \cref{fig:width-fixed-steps}.
  For nGPT with $\mu$P-style width corrections (``CompleteP'' in \cref{tab:our-param}) the learning rate slightly drifts to the right.
  Our parametrization provides essentially perfect transfer of the learning rate over width.
The multipliers for $(\eta_{\text{input}}, \eta_{\text{output}})$ are tuned for both parametrizations separately.
\label{fig:ours-vs-mup-width}
}

\label{fig:mup-drifts-ours-not}
\end{figure}

\begin{roundnote}
  Measuring the alignment exponents leads to corrections into the $\mu$P-type theory,
  providing somewhat better learning rate transfer over width at a fixed token horizon.
\end{roundnote}

The correct value $\omega_{\text{output}} = 1 / 2$ instead of $\omega_{\text{output}} = 1$ allows the logits to be order-1
both at initialization and during training, rather than of order $1 / \sqrt{\dmodel}$ at initialization as in $\mu$P.
An ablation of the logits scaling specifically (with other choices fixed) is provided in \cref{sec:mup-or-not-mup}.

\subsection{Depth-$\mu$P style correction in $\eta_{\text{hidden}}$ would break transfer over depth}

We observe in \cref{fig:depth-mup-depth-sweep} that a depth correction in $\eta_{\text{hidden}}$
motivated by \eqref{eq:LWphMK} with $\depthalph = 1 / 2$, as done in Depth-$\mu$P,
does not show transfer over depth: the learning rate strongly drifts to the right suggesting
that $\eta_{\text{hidden}}$ should not be changed.
This is likely an artifact of the architecture: $\valpha_A$, $\valpha_M$ are trainable and we only control their (small) initial values
and learning rate.
This finding does not confirm or refute \citep{dey2025dont} that ``Depth-$\mu$P'' style $\depthalph = 1 / 2$ is incorrect
but it does suggest that depth transfer in normalized models
(with or without trainable convex combinations)
is a fruitful subject of theoretical research.

\begin{figure}[h!]
  \centering

  \begin{tabular}{c c}
    \begin{subfigure}[h]{0.48\textwidth} \includegraphics[width=\textwidth]{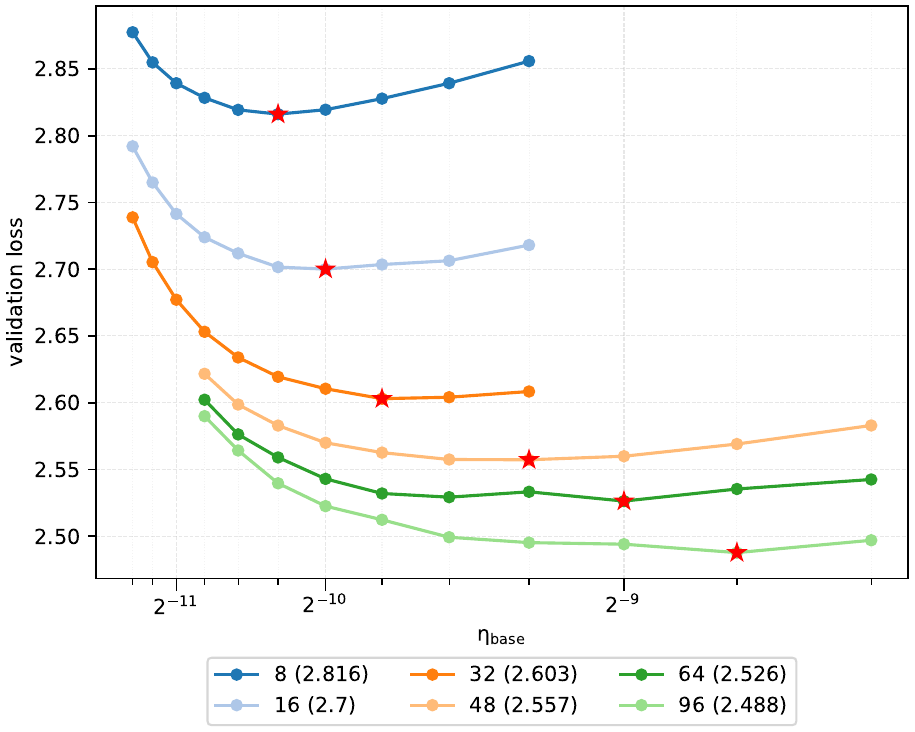}
      \caption{Depth sweep of ``Depth-$\mu$P'' in \cref{tab:our-param}.}
    \end{subfigure}
    &
      \begin{subfigure}[h]{0.48\textwidth}
        \includegraphics[width=\linewidth]{depth_steps80000_sqkBaseScale0.03_szScale0.03_alphaScale0.03_eigenBaseInit0.05_mup_sameScaleLogits_etaOut0.5_etaOutPow0.75_etaHidPow0.75_val_loss.pdf}
        \caption{$\nu$GPT (ours)\internalComment{(mid align) with $\eta_{\text{output}}$ multiplied by $2^{-1}$.}}
      \end{subfigure}

  \end{tabular}

  \caption{
    The number of layers is swept as in \cref{fig:depth-sweeps}.
    A $\depthcor^{-1 / 2}$ correction to $\eta_{\text{hidden}}$ like in Depth-$\mu$P would break transfer over depth.
  }

  \label{fig:depth-mup-depth-sweep}
\end{figure}

\subsection{Token count correction}\label{sec:token-count}

We train a fixed $\nu$GPT model ($\depth = \nheads = 12$)
with token count corrections removed and fit a power law describing how the optimal learning rate
decreases with the number of iterations.
We see that it decreases roughly as $(\text{iter. count})^{- 1 / 3}$
(\cref{fig:steps-sweep-ngpt-ours}).
\begin{figure}[h!]
  \centering

  \begin{tabular}{c c}
    \begin{subfigure}[h]{0.48\textwidth}
      \includegraphics[width=\linewidth]{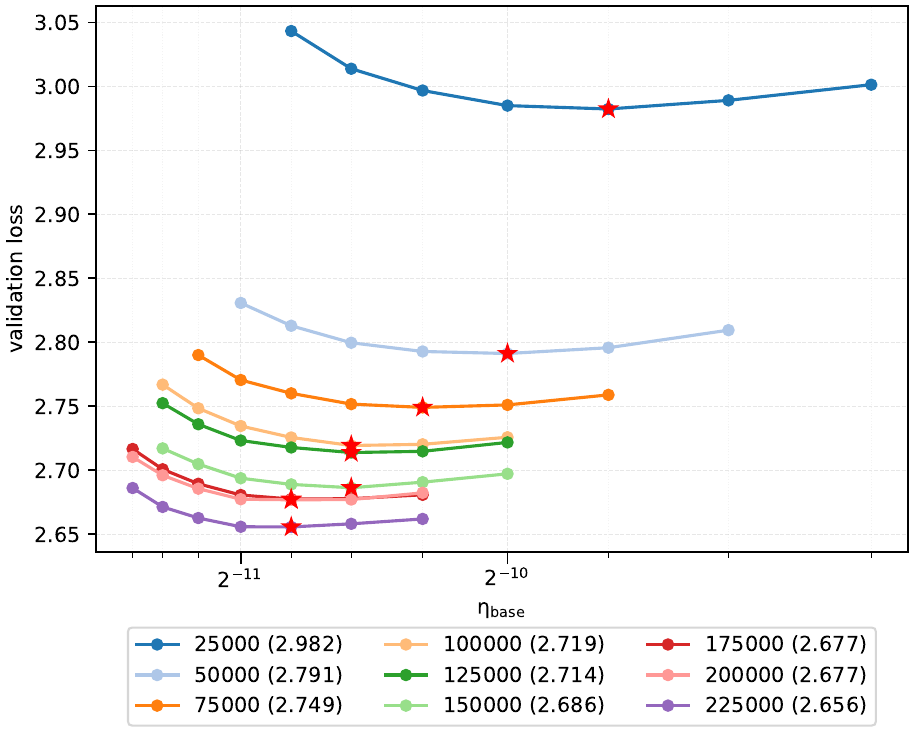}
    \end{subfigure}
    &
      \begin{subfigure}[h]{0.48\textwidth}
        \includegraphics[width=\linewidth]{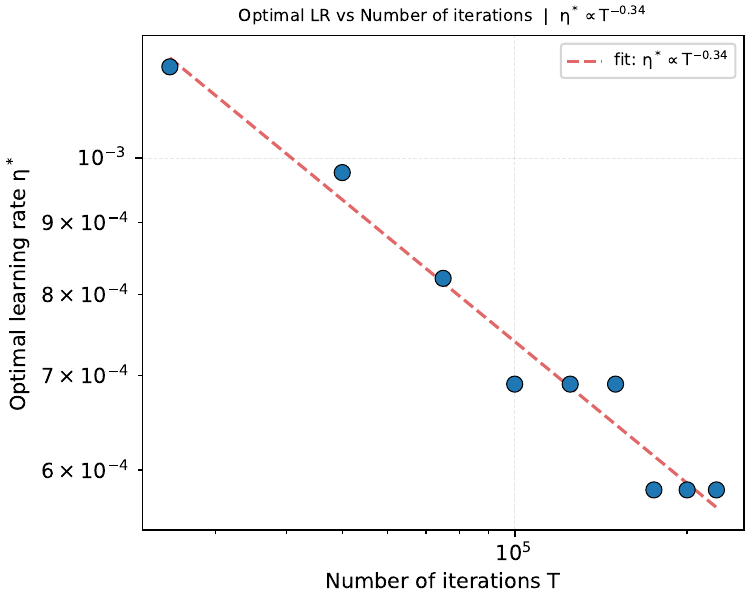}
      \end{subfigure}
  \end{tabular}
  \caption{A sweep of a fixed $\nu$GPT model with different number of iterations (in the legend, with the best loss in parentheses) but without token count corrections: optimal peak learning rate decreases at about $(\text{iter. count})^{- 1 / 3}$.
  }
  \label{fig:steps-sweep-ngpt-ours}
\end{figure}
This matches \citet{bjorck2025scaling} who used non-normalized architectures
(cf. $\beta = 0.32$ there).
An additional sweep with different model sizes is provided in \cref{sec:width-sweep-at-20-tpp}.

\begin{roundnote}
  It is a robust finding that the learning rate decreases proportionally to $(\text{iter. count})^{- 1 / 3}$,
  calling for an iteration count correction $\datacor^{- 1 / 3}$ (as in \cref{tab:our-param}).
\end{roundnote}


\section{Concluding remarks}

We have obtained a parametrization $\nu$GPT of the Normalized Transformer that shows good transfer of the learning rate
over model width, depth and token count as well as their combinations, at no loss of performance.
An important feature of the normalized setting is that it is more controlled, reducing the number of assumptions
(such as matrix norms)
that are to be taken on faith.
For example, the models achieve impressive performance without using weight decay, which is a confounder
for HP transfer theory (e.\,g. \cite{kosson2026weight}).
We have further made significant effort to stay empirically grounded, e.\,g. by measuring weight-activation
alignment exponents,
avoiding $\mu$P-type assumptions commonly used in the literature, and questioning
somewhat imprecise token count power laws theoretically predicted in recent works.
This motivates a few theoretical future directions,
such as a more complete understanding of depthwise
transfer in normalized models, the dynamic understanding of weight decay when it is, in fact, used,
explaining analytically the breakdown of the $\mu$P assumptions in toy models,
and predicting the $\sim 1 / 3$ exponent in the token count corrections.

\section*{Funding Acknowledgments}
BH is supported by a 2024 Sloan Fellowship in Mathematics, NSF grant DMS-2143754, DMS-2133806, and DARPA AIQ grant (HR001124S0029).

\clearpage
\newpage
\beginappendix

\section{Proofs and derivations}\label{sec:proofs-derivations}

\subsection{Width scaling}\label{sec:width-deriv}

In this \lcnamecref{sec:width-deriv}, we will show that the width corrections
described in \cref{sec:summary-of-changes} satisfy \cref{des:stab-init,des:feature-learn}
in the setting of fixed depth.
Specifically, in \cref{lem:Tgznub} we show that the first hidden state $\vh_n^1 = \mE_{\text{input}} \vx_n$ is initialized stably
and evolves stably and non-trivially.
In \cref{lem:DUvOrU}, we show that if this was true of the hidden state before an attention block,
it is also true of the hidden state after this block, and so on. Then, since depth $\depth$ is fixed
and the number of iterations is bounded, by induction, we can conclude that \cref{des:stab-init,des:feature-learn} are satisfied
during the whole training.

\paragraph{Notation}
We ask the reader to recall the notation from \cref{sec:width-depth-corr}.\internalComment{!}
In this section, depth is considered fixed but the number of heads can vary (and we think of it as large).
We will sometimes use $N = \dmodel$ to declutter notations.
We will follow \citet{everett2024scalingexponentsacross} in using non-standard $O(\cdot)$ and similar notation:
specifically, $a_N \lesssim b_N$ or $a_N = O(b_N)$ means $\abs{a_N / b_N}$ is allowed to grow sub-polynomially, that is,
almost surely $\limsup_{N \to \infty} \log_N \abs{a_N / b_N} \leq 0$, whereas
$a_N \asymp b_N$ or $a_N = \Theta(b_N)$ means $a_N \lesssim b_N$ and $b_N \lesssim a_N$.
Note that this definition requires care because properties that more standard $O(\cdot)$ notation
has can be false. For example, it is false that $\sum_{k = 1}^N O(1) = O(N)$ (a counter-example is $a_N^{(k)} \equiv k$ for $k \in \mathbb{Z}_{> 0}$).

\paragraph{Alignment}
Recall the definition of alignment exponents given in \cref{def:alignment-exponents}.

We make the following classical \citep{everett2024scalingexponentsacross}
simple assumption that is definitively consistent with observation (as shown in \cref{sec:alignment-exponents}):

\begin{assumption}\label{ass:alignment}
Assume $\omega_{\text{hidden}} = 1 / 2$ (no alignment).
\end{assumption}

Requiring additionally $\omega_{\text{output}} = 1$ corresponds to classical \mup,
but $\omega_{\text{output}} = 1 / 2$ is consistent with observation (\cref{sec:alignment-exponents})
and allows for non-equivalent maximal-update parametrizations
that achieve feature learning at high widths \citep{everett2024scalingexponentsacross}.

\paragraph{Learning rate of  ``rescalers''}
Recall that trainable ``rescalers'' such as $\frac{s_{\text{$qk$,init}}}{s_{\text{$qk$,scale}}} \vs_{q k}$
have two hyperparameters $s_{\text{$qk$,init}}$ and $s_{\text{$qk$,scale}}$.
Since each component of $\vs_{q k}$ is initialized with $s_{\text{$qk$,scale}}$, each component of
$\frac{s_{\text{$qk$,init}}}{s_{\text{$qk$,scale}}} \vs_{q k}$ at initialization is $s_{\text{$qk$,init}}$.
Thus, the hyperparameter $s_{\text{$qk$,init}}$ controls the initialization value, whereas $s_{\text{$qk$,scale}}$ controls
the learning rate. So, for certainty, we will leave the learning rate of the raw $\vs_{q k}$ vector unmodified
(equal to $\eta_{\text{base}}$). In other words, $\Delta s_{q k, i} = \Theta(1)$ for each coordinate $i$.
We do not lose generality, because if we needed the learning rate of $\frac{s_{\text{$qk$,init}}}{s_{\text{$qk$,scale}}} \vs_{q k}$ to scale with width or depth,
we would modify $s_{\text{$qk$,scale}}$.

\begin{assumption}\label{ass:rescalers-movement}
  At each step $t = O(1)$ we have
\begin{equation*}
\Delta s_{q k, i}(t) \asymp \Delta s_{u, i}(t) \asymp \Delta s_{\nu, i}(t) \asymp \Delta s_{z, i}(t) \asymp \Delta \alpha_{A, i}(t) \asymp \Delta \alpha_{M, i}(t) \asymp 1.
\end{equation*}
\end{assumption}

We now proceed with proving that neither block of $\nu$GPT breaks the stability and non-triviality requirements.

\begin{lemma}[Input]\label{lem:Tgznub}
  Suppose \cref{ass:alignment} holds,
  and $\eta_{\text{input}} \asymp d_{\text{model}}^{- 1 / 2}$.
  Then, for one-hot $\vx_n$, the vector $\vh_n^1(t) := \mE_{\text{input}}(t) \vx_n$ satisfies
  $\norm{\vh_n^1} = \Theta(1)$
  and $\norm{\Delta \vh_n^1} = \Theta(1)$.
\end{lemma}

\begin{proof}
Since $\vx_n$ is one-hot, $\vh_n^1$ is a column of $\mE_{\text{input}}$,
so it is normalized: $\norm{\vh_n^1} = 1$.
Next, $\Delta \vh_n^1$ is a column of $\Delta \mE_{\text{input}}$. Therefore,
\begin{equation*}
\norm{\Delta \vh_n^1} \asymp \eta_{\text{input}} \sqrt{\dmodel} \asymp 1.\qedhere
\end{equation*}
\end{proof}

The next \lcnamecref{lem:mlp} deals with the stability of forward pass and updates in the MLP block.

\begin{lemma}[MLP]\label{lem:mlp}
  Suppose \cref{ass:alignment,ass:rescalers-movement} hold,
  the hyperparameters
  $s_{\text{$u$,init}}$ $s_{\text{$u$,scale}}$, $s_{\text{$\nu$,init}}$, $s_{\text{$\nu$,scale}}$
  are positive constants,
  and $\eta_{\text{hidden}} \lesssim \dmodel^{- \max\crl{\alpha_{\text{hidden}}, \nu_{\text{hidden}}}}$.
  In addition, suppose $\norm{\Delta \vh_n(t)} \asymp 1$.
  Then the components\footnote{We assume implicitly that components of general vectors are of the same scale. (They in fact become i.\,i.\,d. random variables in the infinite-width limit, see e.\,g. \citet{pmlr-v139-yang21c}).} of $\text{MLP}(\vh_n)$ (at time $0$) and $\Delta \text{MLP}(\vh_n) (t)$ are both $\Theta(d_{\text{model}}^{- 1 / 2})$.
\end{lemma}

\begin{proof}
The scale of $\mW_u \vh_n$ at initialization (no alignment) is given by
\begin{equation*}
  \frac{\norm{\mW_u \vh_n}}{\sqrt{d_{\text{MLP}}}}
  \asymp \sqrt{d_{\text{model}}} \frac{\norm{\mW_u}_F}{\sqrt{d_{\text{MLP}} d_{\text{model}}}} \frac{\norm{\vh_n}}{\sqrt{d_{\text{model}}}}
  = d_{\text{model}}^{- 1 / 2},
\end{equation*}
whereas the movement is
\begin{equation*}
\Delta (\mW_u \vh_n) = \Delta \mW_u \vh_n + \mW_u \Delta \vh_n + \Delta \mW_u \Delta \vh_n,
\end{equation*}
and
\begin{equation*}
  \frac{\norm{\Delta \mW_u \vh_n}}{\sqrt{d_{\text{MLP}}}}
  \asymp d_{\text{model}}^{\alpha_{\text{hidden}}} \eta_{\text{hidden}} \frac{\norm{\vh_n}}{\sqrt{d_{\text{model}}}}
  \asymp d_{\text{model}}^{\alpha_{\text{hidden}} - 1 / 2} \eta_{\text{hidden}};
\end{equation*}
by $\omega_{\text{hidden}} = 1 / 2$ (and using $\norm{\Delta \vh_n} \asymp 1$)
\begin{equation*}
  \frac{\norm{\mW_u \Delta \vh_n}}{\sqrt{d_{\text{MLP}}}}
  \asymp \sqrt{d_{\text{model}}} \frac{\norm{\mW_u}_F}{\sqrt{d_{\text{MLP}} d_{\text{model}}}} \frac{\norm{\Delta \vh_n}}{\sqrt{d_{\text{model}}}}
  \asymp d_{\text{model}}^{- 1 / 2};
\end{equation*}
and
\begin{equation*}
  \frac{\norm{\Delta \mW_u \Delta \vh_n}}{\sqrt{d_{\text{MLP}}}} \asymp d_{\text{model}}^{\nu_{\text{hidden}}} \eta_{\text{hidden}} \frac{\norm{\Delta \vh_n}}{\sqrt{d_{\text{model}}}}
  \asymp d_{\text{model}}^{\nu_{\text{hidden}} - 1 / 2} \eta_{\text{hidden}}.
\end{equation*}
Since $\eta_{\text{hidden}} \lesssim \dmodel^{- \max\crl{\alpha_{\text{hidden}}, \nu_{\text{hidden}}}}$,
we see that the scales of components of $\mW_u \vh_n$ and $\Delta (\mW_u \vh_n)$ stay $\Theta(d_{\text{model}}^{- 1 / 2})$ during training.

The movement of $\vu$ is given by
\begin{equation*}
  \Delta u_i
  = \Delta (\mW_u \vh_n)_i \frac{s_{\text{$u$,init}}}{s_{\text{$u$,scale}}} s_{u, i}
  + (\mW_u \vh_n)_i \frac{s_{\text{$u$,init}}}{s_{\text{$u$,scale}}} \Delta s_{u, i}
  + \Delta (\mW_u \vh_n)_i \frac{s_{\text{$u$,init}}}{s_{\text{$u$,scale}}} \Delta s_{u, i}.
\end{equation*}
Since at initialization $s_{u, i} = s_{\text{$u$,scale}}$,
the first term is of scale $d_{\text{model}}^{- 1 / 2} s_{\text{$u$,init}} \asymp d_{\text{model}}^{- 1 / 2}$;
the second and third are of scale $\frac{s_{\text{$u$,init}}}{s_{\text{$u$,scale}}} d_{\text{model}}^{- 1 / 2} \asymp d_{\text{model}}^{- 1 / 2}$.

By the same logic as above, the scales of components of $\mW_{\nu} \vh_n$ and $\Delta (\mW_{\nu} \vh_n)$ stay $\Theta(d_{\text{model}}^{- 1 / 2})$ during training.
The update of $\vnu$ is given by
\begin{equation*}
\Delta \nu_i = \Delta (\mW_{\nu} \vh_n)_i \frac{s_{\text{$\nu$,init}}}{s_{\text{$\nu$,scale}}} d_{\text{model}}^{1 / 2} s_{\nu, i} + (\mW_{\nu} \vh_n)_i \frac{s_{\text{$\nu$,init}}}{s_{\text{$\nu$,scale}}} d_{\text{model}}^{1 / 2} \Delta s_{\nu, i} + \Delta (\mW_{\nu} \vh_n)_i \frac{s_{\text{$\nu$,init}}}{s_{\text{$\nu$,scale}}} d_{\text{model}}^{1 / 2} \Delta s_{\nu, i},
\end{equation*}
and the scale of each term is $\Theta(1)$.

We see that components of $\vnu$ and hence $\text{SiLU}(\vnu)$ are $\Theta(1)$, so the components of
$\text{SiLU}(\vnu) \odot \vu$ are $\Theta(d_{\text{model}}^{- 1 / 2})$ along with their updates during training.

Proceeding with the same logic (and the same alignment assumptions), we see that
the components of $\text{MLP}(\vh_n)$ along with their updates stay
$\Theta(d_{\text{model}}^{- 1 / 2})$ during training.
\end{proof}

\begin{remark}
  For baseline nGPT, the same calculation shows that components of $\vu$ are $\Theta(\dmodel^{- 1 / 2})$
  at initialization, components of $\Delta \vu$ are $\Theta\prn[\big]{\max\crl{\dmodel^{1 / 2} \eta_{\text{global}}, \dmodel^{- 1 / 2}}}$ during training,
  whereas $\frac{s_{\text{$u$,init}}}{s_{\text{$u$,scale}}} \Delta s_{u, i}$ move at scale $\eta_{\text{global}}$.
  This means that the movement of $\frac{s_{\text{$u$,init}}}{s_{\text{$u$,scale}}} \Delta s_{u, i}$ becomes negligible at large width.
  The same can be said about $\frac{s_{\text{$\nu$,init}}}{s_{\text{$\nu$,scale}}} \Delta s_{\nu, i}$.
\end{remark}

The next \lcnamecref{lem:DUvOrU} checks the scaling inside attention blocks.

\begin{lemma}[Attention]\label{lem:DUvOrU}
  Suppose \cref{ass:alignment,ass:rescalers-movement} hold,
  the hyperparameters
  $s_{\text{$qk$,init}}$, $s_{\text{$qk$,scale}}$
  are positive constants,
  and $\eta_{\text{hidden}} \lesssim \dmodel^{- \max\crl{\alpha_{\text{hidden}}, \nu_{\text{hidden}}}}$.
  In addition, suppose $\norm{\Delta \vh_n(t)} \asymp 1$.
  Assume that $d_{\text{key}}$ is large enough so that the approximation $\Delta \frac{q_{n, i}}{\norm{\vq_n}} \asymp \frac{\Delta q_{n, i}}{\norm{\vq_n}}$ holds.
  Then the components of $\vq'_n$ (at time $0$) and $\Delta \vq_n' (t)$ are both $\Theta(d_{\text{key}}^{- 1 / 2})$, where $\vq'_n$ is given by \cref{eq:xgSNnZ};
  the components of $\vk'_m$ and $\Delta \vk_m' (t)$ are both $\Theta(d_{\text{key}}^{- 1 / 2})$, where $\vk'_m$ is given by \cref{eq:XBxnKv};
  the components of $\vv_m$ and $\Delta \vv_m(t)$ are both $\Theta(d_{\text{model}}^{- 1 / 2})$, where $\vv_m$ is given by \cref{eq:MNAzYI}.
  Further, the quantity $\sqrt{d_{\text{key}}} \vq^{\prime \transt}_n \vk'_m$ is $\Theta(1)$ at initialization,
  whereas its movement $\sqrt{d_{\text{key}}} \Delta \prn[\big]{\vq^{\prime \transt}_n \vk'_m}$ is $O(\sqrt{d_{\text{key}}})$.
  Finally, if $d_{\text{key}}$ is constant, the components of the attention output
  $\text{Attention}_n(\crl{\vh_m}_{m = 1}^{\text{SeqLen}})$ (given by \cref{eq:JSpBcR})
  and of their movements $\Delta \text{Attention}_n(\crl{\vh_m}_{m = 1}^{\text{SeqLen}})(t)$
  are both $\Theta(\dmodel^{- 1 / 2})$.
\end{lemma}

\begin{remark*}
  In $\mu$P \cite{NEURIPS2021_8df7c2e3}, the multiplier of key-query scalar product is changed because of alignment between keys and queries during training.
  Here, such alignment would correspond to $\Delta \prn[\big]{\vq^{\prime \transt}_n \vk'_m} = \Theta(1)$.
  If $d_{\text{key}}$ were not constant, we would need to remove the multiplier $\sqrt{d_{\text{key}}}$
  in attention
  to prevent blowup.
\end{remark*}

\begin{proof}
For simplicity, we ignore rotary embedding maps here because they do not change the relevant scales.
The scale of a query vector at initialization (no alignment) is
\begin{equation*}
  \frac{\norm{\vq_n}}{\sqrt{d_{\text{key}}}}
  = \sqrt{d_{\text{model}}} \frac{\norm{\mW_q}_F}{\sqrt{d_{\text{key}} d_{\text{model}}}}
  \frac{\norm{\vh_n}}{\sqrt{d_{\text{model}}}}
  = d_{\text{model}}^{- 1 / 2}.
\end{equation*}
The movement at time $t$ is given by
\begin{equation*}
\Delta \vq_n = \Delta \mW_q\trans \vh_n + \mW_q\trans \Delta \vh_n + \Delta \mW_q\trans \Delta \vh_n,
\end{equation*}
and by the definition of $\alpha_{\text{hidden}}$
\begin{equation*}
  \frac{\norm{\Delta \mW_q\trans \vh_n}}{\sqrt{d_{\text{key}}}}
  \asymp d_{\text{model}}^{\alpha_{\text{hidden}}} \eta_{\text{hidden}} \frac{\norm{\vh_n}}{\sqrt{d_{\text{model}}}}
  = d_{\text{model}}^{\alpha_{\text{hidden}} - 1 / 2} \eta_{\text{hidden}};
\end{equation*}
by $\omega_{\text{hidden}} = 1 / 2$ (and using $\norm{\Delta \vh_n} \asymp 1$)
\begin{equation*}
  \frac{\norm{\mW_q\trans \Delta \vh_n}}{\sqrt{d_{\text{key}}}}
  \asymp \sqrt{d_{\text{model}}} \frac{\norm{\mW_q}_F}{\sqrt{d_{\text{key}} d_{\text{model}}}} \frac{\norm{\Delta \vh_n}}{\sqrt{d_{\text{model}}}}
  \asymp d_{\text{model}}^{- 1 / 2};
\end{equation*}
by the definition of $\nu_{\text{hidden}}$ (again using $\norm{\Delta \vh_n} \asymp 1$)
\begin{equation*}
  \frac{\norm{\Delta \mW_q\trans \Delta \vh_n}}{\sqrt{d_{\text{key}}}} \asymp d_{\text{model}}^{\nu_{\text{hidden}}} \eta_{\text{hidden}} \frac{\norm{\Delta \vh_n}}{\sqrt{d_{\text{model}}}}
  \asymp d_{\text{model}}^{\nu_{\text{hidden}} - 1 / 2} \eta_{\text{hidden}}.
\end{equation*}
Since $\eta_{\text{hidden}} \lesssim \dmodel^{- \max\crl{\alpha_{\text{hidden}}, \nu_{\text{hidden}}}}$,
we see that the scales of components of $\vq_n$ and $\Delta \vq_n$ are $\Theta(d_{\text{model}}^{- 1 / 2})$ during training.

The update of $\vq'_n$ is given by
\begin{equation}\label{eq:SgIxkS}
  \Delta q_{n, i}' =
  \prn[\bigg]{\Delta \frac{q_{n, i}}{\norm{\vq_n}}} \frac{s_{\text{$qk$,init}}}{s_{\text{$qk$,scale}}} s_{q k, i}
  + \frac{q_{n, i}}{\norm{\vq_n}} \frac{s_{\text{$qk$,init}}}{s_{\text{$qk$,scale}}} \Delta s_{q k, i}
  + \prn[\bigg]{\Delta \frac{q_{n, i}}{\norm{\vq_n}}} \frac{s_{\text{$qk$,init}}}{s_{\text{$qk$,scale}}} \Delta s_{q k, i},
\end{equation}
where
\begin{equation*}
\Delta q_{n, i} \asymp d_{\text{model}}^{- 1 / 2}, \quad \norm{\vq_n} \asymp d_{\text{key}}^{1 / 2} d_{\text{model}}^{- 1 / 2} \quad \Rightarrow \quad \frac{\Delta q_{n, i}}{\norm{\vq_n}} = d_{\text{key}}^{- 1 / 2}
\end{equation*}
and
\begin{equation*}
\frac{q_{n, i}}{\norm{\vq_n}} \asymp d_{\text{key}}^{- 1 / 2}.
\end{equation*}
Then, using $\Delta \frac{q_{n, i}}{\norm{\vq_n}} \asymp \frac{\Delta q_{n, i}}{\norm{\vq_n}}$, we see that each of the three terms in \cref{eq:SgIxkS} is $\Theta(d_{\text{key}}^{- 1 / 2})$.

The claims about $\vk_m'$ and $\vv_m$ are proven similarly.

At initialization, $\vq'_n$ and $\vk'_m$ are unaligned, so that
\begin{equation*}
\abs{\vq^{\prime \transt}_n \vk'_m} \asymp \frac{\norm{\vq'_n} \norm{\vk'_m}}{\sqrt{d_{\text{key}}}} \asymp \frac{1}{\sqrt{d_{\text{key}}}}.
\end{equation*}
During training, they may become aligned, and the upper bound
$\abs[\big]{\Delta \prn[\big]{\vq^{\prime \transt}_n \vk'_m}} \lesssim 1$
follows from the Cauchy-Schwarz inequality.

We conclude that, if $d_{\text{key}}$ is constant, the components of each head's output and their movements
are $\Theta(\dmodel^{- 1 / 2})$. Hence, the components of $\text{Attention}_n(\crl{\vh_m}_{m = 1}^{\text{SeqLen}})$ and $\Delta \text{Attention}_n(\crl{\vh_m}_{m = 1}^{\text{SeqLen}})(t)$ are both $\Theta(\dmodel^{- 1 / 2})$ by the same argument as in \cref{lem:mlp}.
\end{proof}

\internalComment{
  I speculate that their intuition is as follows.
  Think of $q_{n, i} / \norm{\vq_n}$ as one parameter whose magnitude is $d_{\text{model}}^{- 1 / 2}$
  (not distinguishing between $d_{\text{key}}$ and $d_{\text{model}}$).
  Say we want to match the scales in the update above:
  \begin{equation*}
  \frac{\Delta q_{n, i}}{\norm{\vq_n}} \frac{s_{\text{$qk$,init}}}{s_{\text{$qk$,scale}}} s_{q k, i} \asymp \frac{q_{n, i}}{\norm{\vq_n}} \frac{s_{\text{$qk$,init}}}{s_{\text{$qk$,scale}}} \Delta s_{q k, i}
\end{equation*}
Assuming $s_{q k, i} \asymp s_{\text{$qk$,scale}}$ (hence the name) leads to
\begin{equation*}
\Delta \frac{q_{n, i}}{\norm{\vq_n}} \asymp d_{\text{model}}^{- 1 / 2} \frac{1}{s_{\text{$qk$,scale}}} \Delta s_{q k, i}.
\end{equation*}
Assuming (incorrectly) that this normalized parameter $q_{n, i} / \norm{\vq_n}$ is directly updated with one global learning rate, we \textit{would have}
\begin{equation*}
\eta_{\text{global}} \asymp d_{\text{model}}^{- 1 / 2} \frac{1}{s_{\text{$qk$,scale}}} \eta_{\text{global}},
\end{equation*}
motivating $s_{\text{$qk$,scale}} = d_{\text{model}}^{- 1 / 2}$.
}

The next lemma deals with the LERP after Attention or MLP blocks.

\begin{lemma}[LERP; learning rates of $\valpha_A$ and $\valpha_M$]
  Suppose \cref{ass:alignment,ass:rescalers-movement} hold and the hyperparameter $\alpha_{\text{$A$,scale}}$
  is a positive constant. In addition, suppose $\Delta h_{A, n, i}(t) \asymp d_{\text{model}}^{-1 / 2}$.
  Then each component of $\frac{\alpha_{\text{$A$,init}}}{\alpha_{\text{$A$,scale}}} \valpha_A \odot \vh_{A, n}$
  (at initialization)
  and its movement at time $t$ are both of scale $\Theta(\alpha_{\text{$A$,init}} d_{\text{model}}^{- 1 / 2})$.
\end{lemma}

\begin{proof}
The scale of each component in $\frac{\alpha_{\text{$A$,init}}}{\alpha_{\text{$A$,scale}}} \valpha_A \odot \vh_{A, n}$ at initialization
is $\alpha_{\text{$A$,init}} d_{\text{model}}^{- 1 / 2}$.
The update is given by
\begin{equation*}
  \frac{\alpha_{\text{$A$,init}}}{\alpha_{\text{$A$,scale}}} \Delta \alpha_{A, i} h_{A, n, i}
  + \frac{\alpha_{\text{$A$,init}}}{\alpha_{\text{$A$,scale}}} \alpha_{A, i} \Delta h_{A, n, i}
  + \frac{\alpha_{\text{$A$,init}}}{\alpha_{\text{$A$,scale}}} \Delta \alpha_{A, i} \Delta h_{A, n, i}
\end{equation*}
where
\begin{equation*}
\Delta h_{A, n, i} \asymp d_{\text{model}}^{- 1 / 2}, \quad h_{A, n, i} \asymp d_{\text{model}}^{- 1 / 2}.
\end{equation*}
Assuming $\alpha_{A, i} \asymp \alpha_{\text{$A$,scale}}$,
we have
\begin{align*}
  &\frac{\alpha_{\text{$A$,init}}}{\alpha_{\text{$A$,scale}}} \Delta \alpha_{A, i} h_{A, n, i}
    \asymp \frac{\alpha_{\text{$A$,init}}}{\alpha_{\text{$A$,scale}}} d_{\text{model}}^{- 1 / 2} \Delta \alpha_{A, i},\\
  &\frac{\alpha_{\text{$A$,init}}}{\alpha_{\text{$A$,scale}}} \alpha_{A, i} \Delta h_{A, n, i} \asymp \alpha_{\text{$A$,init}} d_{\text{model}}^{- 1 / 2},\\
  &\frac{\alpha_{\text{$A$,init}}}{\alpha_{\text{$A$,scale}}} \Delta \alpha_{A, i} \Delta h_{A, n, i} \asymp \frac{\alpha_{\text{$A$,init}}}{\alpha_{\text{$A$,scale}}} d_{\text{model}}^{- 1 / 2} \Delta \alpha_{A, i}.\qedhere
\end{align*}
\end{proof}

For logits scaling,
there are two possible choices.
If we assume $\omega_{\text{output}} = 1 / 2$ \citep{everett2024scalingexponentsacross},
it is possible for logits to be the same scale at initialization
as during training, although this requires $s_{\text{$z$,init}} \asymp d_{\text{model}}^{1 / 2}$.
On the other hand, in classical $\mu$P parametrizations, the logits scale as $1 / \sqrt{\dmodel}$ at initialization but move $\Theta(1)$.

\begin{lemma}[Same scaling of logits at initialization and during training]\label{lem:logits-under-align0.5}
  Suppose \cref{ass:alignment,ass:rescalers-movement} hold,
  $\omega_{\text{output}} = 1 / 2$, $\eta_{\text{output}} \lesssim \dmodel^{- \max\crl{\alpha_{\text{output}}, \nu_{\text{output}}}}$, $s_{\text{$z$,init}} \asymp d_{\text{model}}^{1 / 2}$,
  and $s_{\text{$z$,scale}}$ is a positive constant.
  Suppose further $\norm{\Delta \vh_n(t)} \asymp 1$.
  Then each component of $\vz_n$ (at initialization) and $\Delta \vz_n(t)$ are both of scale $\Theta(1)$, where $\vz_n$ is defined by \cref{eq:oDPOKs}.
\end{lemma}

\begin{proof}
The scale of $\hat{\vz}_n \in \mathbb{R}^V$ at initialization (no alignment) is given by
\begin{equation}\label{eq:gVMQEG}
  \hat{z}_{n, i} \asymp \frac{\norm{\hat{\vz}_n}}{\sqrt{V}} \asymp \sqrt{d_{\text{model}}} \frac{\norm{\mE_{\text{output}}}_F}{\sqrt{d_{\text{model}} V}} \frac{\norm{\vh_n}}{\sqrt{d_{\text{model}}}}
  = \frac{\norm{\mE_{\text{output}}}_F}{\sqrt{d_{\text{model}} V}} = d_{\text{model}}^{- 1 / 2}.
\end{equation}
The update of $\hat{\vz}_n$ is given by
\begin{equation*}
\Delta \hat{\vz}_n = \Delta \mE_{\text{output}} \vh_n + \mE_{\text{output}} \Delta \vh_n + \Delta \mE_{\text{output}} \Delta \vh_n.
\end{equation*}
The scale of the first perturbation term is
\begin{equation}\label{eq:NGyVlR2}
  \frac{\norm{\Delta \mE_{\text{output}} \vh_n}}{\sqrt{V}} \asymp d_{\text{model}}^{\alpha_{\text{output}}} \eta_{\text{output}} \frac{\norm{\vh_n}}{\sqrt{d_{\text{model}}}} = d_{\text{model}}^{\alpha_{\text{output}} - 1 / 2} \eta_{\text{output}};
\end{equation}
of the second perturbation term (no alignment $\omega_{\text{output}} = 1 / 2$ and $\norm{\Delta \vh_n} \asymp 1$)
\begin{equation*}
  \frac{\norm{\mE_{\text{output}} \Delta \vh_n}}{\sqrt{V}} \asymp \sqrt{d_{\text{model}}} \frac{\norm{\mE_{\text{output}}}_F}{\sqrt{V d_{\text{model}}}} \frac{\norm{\Delta \vh_n}}{\sqrt{d_{\text{model}}}}
  = \frac{\norm{\mE_{\text{output}}}_F}{\sqrt{V d_{\text{model}}}} = d_{\text{model}}^{- 1 / 2};
\end{equation*}
and of the third perturbation term (using $\norm{\Delta \vh_n} \asymp 1$)
\begin{equation*}
  \frac{\norm{\Delta \mE_{\text{output}} \Delta \vh_n}}{\sqrt{V}}
  \asymp d_{\text{model}}^{\nu_{\text{output}}} \eta_{\text{output}} \frac{\norm{\Delta \vh_n}}{\sqrt{d_{\text{model}}}}
  = d_{\text{model}}^{\nu_{\text{output}} - 1 / 2} \eta_{\text{output}}.
\end{equation*}
Then, each component of both $\hat{\vz}_n$ (at initialization) and $\Delta \hat{\vz}_n(t)$ is of scale $d_{\text{model}}^{- 1 / 2}$ during training.

The update of $\hat{\vz}_n \odot \frac{s_{\text{$z$,init}}}{s_{\text{$z$,scale}}} \vs_z$ is given by
\begin{equation*}
  \Delta \prn[\bigg]{\hat{z}_{n, i} \frac{s_{\text{$z$,init}}}{s_{\text{$z$,scale}}} s_{z, i}} = \Delta \hat{z}_{n, i} \frac{s_{\text{$z$,init}}}{s_{\text{$z$,scale}}} s_{z, i} + \hat{z}_{n, i} \frac{s_{\text{$z$,init}}}{s_{\text{$z$,scale}}} \Delta s_{z, i} + \frac{s_{\text{$z$,init}}}{s_{\text{$z$,scale}}} \Delta \hat{z}_{n, i} \Delta s_{z, i}.
\end{equation*}
Each term is of scale $s_{\text{$z$,init}} d_{\text{model}}^{- 1 / 2} \asymp 1$.
\end{proof}

Finally,
classical $\mu$P leads to logits being negligible at initialization
but not during training.
Moreover, $\mu$P would be the only possible (stable and nontrivial) choice if we assumed $\omega_{\text{output}} = 1$.
Width-wise learning rate transfer in \mup is a robust empirical observation,
making it a safe choice.

\begin{lemma}[Logits are smaller at initialization than during training, e.\,g. classical \mup]\label{lem:logits-under-align1}
  Suppose \cref{ass:alignment,ass:rescalers-movement} hold,
  $\eta_{\text{output}} \asymp \dmodel^{1 / 2 - \max\crl{\alpha_{\text{output}}, \nu_{\text{output}}}}$, $s_{\text{$z$,init}}$
  and $s_{\text{$z$,scale}}$ are positive constants.
  Suppose further $\norm{\Delta \vh_n(t)} \asymp 1$.
  Then, with no assumption on $\omega_{\text{output}}$, each component of $\vz_n$ (at initialization)
  is $\Theta(d_{\text{model}}^{- 1 / 2})$
  but $\Delta \vz_n(t)$ is of scale $\Theta(1)$, where $\vz_n$ is defined by \cref{eq:oDPOKs}.
\end{lemma}

\begin{proof}
The scale of $\hat{\vz}_n \in \mathbb{R}^V$ at initialization (no alignment) is given by
\begin{equation}\label{eq:gVMQEGhhh}
  \hat{z}_{n, i} \asymp \frac{\norm{\hat{\vz}_n}}{\sqrt{V}} \asymp \sqrt{d_{\text{model}}} \frac{\norm{\mE_{\text{output}}}_F}{\sqrt{d_{\text{model}} V}} \frac{\norm{\vh_n}}{\sqrt{d_{\text{model}}}}
  = \frac{\norm{\mE_{\text{output}}}_F}{\sqrt{d_{\text{model}} V}} = d_{\text{model}}^{- 1 / 2}.
\end{equation}
The update of $\hat{\vz}_n$ is given by
\begin{equation*}
\Delta \hat{\vz}_n = \Delta \mE_{\text{output}} \vh_n + \mE_{\text{output}} \Delta \vh_n + \Delta \mE_{\text{output}} \Delta \vh_n.
\end{equation*}
The scale of the first perturbation term is
\begin{equation}\label{eq:NGyVlR}
  \frac{\norm{\Delta \mE_{\text{output}} \vh_n}}{\sqrt{V}} \asymp d_{\text{model}}^{\alpha_{\text{output}}} \eta_{\text{output}} \frac{\norm{\vh_n}}{\sqrt{d_{\text{model}}}} = d_{\text{model}}^{\alpha_{\text{output}} - 1 / 2} \eta_{\text{output}};
\end{equation}
of the second perturbation term (using $\norm{\Delta \vh_n} \asymp 1$)
\begin{equation*}
  \frac{\norm{\mE_{\text{output}} \Delta \vh_n}}{\sqrt{V}} \asymp d_{\text{model}}^{\omega_{\text{output}}} \frac{\norm{\mE_{\text{output}}}_F}{\sqrt{V d_{\text{model}}}} \frac{\norm{\Delta \vh_n}}{\sqrt{d_{\text{model}}}}
  \asymp d_{\text{model}}^{\omega_{\text{output}} - 1 / 2} \frac{\norm{\mE_{\text{output}}}_F}{\sqrt{V d_{\text{model}}}} \asymp d_{\text{model}}^{\omega_{\text{output}} - 1};
\end{equation*}
and of the third perturbation term ($\norm{\Delta \vh_n} \asymp 1$)
\begin{equation*}
  \frac{\norm{\Delta \mE_{\text{output}} \Delta \vh_n}}{\sqrt{V}}
  \lesssim d_{\text{model}}^{\nu_{\text{output}}} \eta_{\text{output}} \frac{\norm{\Delta \vh_n}}{\sqrt{d_{\text{model}}}}
  = d_{\text{model}}^{\nu_{\text{output}} - 1 / 2} \eta_{\text{output}}.
\end{equation*}
Thus, the (pre-)logits start much smaller than they are updated.
From the second forward pass onwards, the scale of (pre-)logits becomes
the (dominant) scale of their updates, that is,
\begin{equation*}
\hat{z}_{n, i} \asymp d_{\text{model}}^{- 1 / 2} \quad \text{but} \quad \hat{z}_{n, i} + \Delta \hat{z}_{n, i} \asymp 1.
\end{equation*}
The movement of $\hat{\vz}_n \odot \frac{s_{\text{$z$,init}}}{s_{\text{$z$,scale}}} \vs_z$ is given by
\begin{equation*}
  \Delta \prn[\bigg]{\hat{z}_{n, i} \frac{s_{\text{$z$,init}}}{s_{\text{$z$,scale}}} s_{z, i}} = \Delta \hat{z}_{n, i} \frac{s_{\text{$z$,init}}}{s_{\text{$z$,scale}}} s_{z, i} + \hat{z}_{n, i} \frac{s_{\text{$z$,init}}}{s_{\text{$z$,scale}}} \Delta s_{z, i} + \frac{s_{\text{$z$,init}}}{s_{\text{$z$,scale}}} \Delta \hat{z}_{n, i} \Delta s_{z, i}.
\end{equation*}
The first term is of scale $s_{\text{$z$,init}}$,
the second of scale $d_{\text{model}}^{- 1 / 2} \frac{s_{\text{$z$,init}}}{s_{\text{$z$,scale}}}$,
and the third of scale $\frac{s_{\text{$z$,init}}}{s_{\text{$z$,scale}}}$ (always dominating the second term).
The fact that $s_{\text{$z$,scale}} \asymp 1$ ensures matching scales of the first and third terms
(ensuring that the movement in $\vs_z(t)$ is not negligible but does not dominate the dynamics).
\end{proof}




\subsection{Depth scaling}\label{sec:depth-scaling-deriv}

As common in the literature \citep{yang2024tensor,dey2025dont}, we consider a simplification of the network to derive depth
corrections.

\begin{definition}[Simple normalized network]
The linear normalized network maps a one-hot vector $\vx \in \mathbb{R}^V$ to logits $\vz(\vx; t) \in \mathbb{R}^V$ as follows:
\begin{align*}
  \vh^1(\vx; t) &= \mE_{\text{input}}(t) \vx,\\
  \hat{\vh}^{\ell + 1}(\vx; t) &= \prn[\big]{1 - L^{-\alpha}} \vh^{\ell}(\vx; t) + L^{-\alpha} \text{Norm} \prn[\big]{\mW^{\ell}(t) \vh^{\ell}(\vx; t)},\\
  \vh^{\ell + 1}(\vx; t) &= \text{Norm} \prn[\big]{\hat{\vh}^{\ell + 1}(\vx; t)}, \quad \ell = 1, \ldots, L - 1,\\
  \vz(\vx; t) &= \mE_{\text{output}}(t) \vh^L(\vx; t)
\end{align*}
with each matrix $\mE_{\text{input}}(t) \in \mathbb{R}^{N \times V}$, $\mW^{\ell}(t) \in \mathbb{R}^{N \times N}$ for $\ell \in \range{1}{L - 1}$, $\mE_{\text{output}}(t) \in \mathbb{R}^{V \times N}$ initialized with i.\,i.\,d. Gaussian components and normalized along the embedding dimension (columns for $\mE_{\text{input}}$, rows for $\mW^{\ell}$ and $\mE_{\text{output}}$) before each training step, where $\alpha > 0$ is the ``depth alpha'' exponent.

The signGD updates $\Delta \mE_{\text{input}}$, $\Delta \mW^{\ell}$ and $\Delta \mE_{\text{output}}$ consist of components $\pm \eta_{\text{input}}$, $\pm \eta_{\text{hidden}}$ and $\pm \eta_{\text{output}}$ respectively.
\end{definition}

\paragraph{Notation}
We will treat the width $N \equiv d_{\text{model}}$ and depth $L$ as varying (and think of them as large), whereas $V$ is a fixed constant dimension.
Similarly to \cref{sec:width-deriv},
the notation $a_{N, L} \lesssim b_{N, L}$ or $a_{N, L} = O(b_{N, L})$
means $\limsup_{N \to \infty} \log_N \abs{a_{N, L} / b_{N, L}} \leq 0$ for all $L$ (!),
whereas $a_{N, L} \asymp b_{N, L}$ or $a_{N, L} = \Theta(b_{N, L})$ means $a_{N, L} \lesssim b_{N, L}$ and $b_{N, L} \lesssim a_{N, L}$.

First, stability at initialization (\cref{des:stab-init}) is automatic because the network is normalized. Indeed, since $\vx$ is one-hot, $\mE_{\text{input}} \vx$
is a column of $\mE_{\text{input}}$, which is normalized before each training step, so $\norm{\vh^1} = 1$.
Next, $\norm{\vh^{\ell}} = 1$ for $\ell \in \range{2}{L}$ by definition.
Finally, at initialization $\mE_{\text{output}}$ is unaligned with $\vh^L$, so
\begin{equation*}
  \norm{\vz} \asymp \frac{\norm{\vz}}{\sqrt{V}} \asymp \sqrt{N} \frac{\norm{\mE_{\text{output}}}_F}{\sqrt{N V}} \frac{\norm{\vh^L}}{\sqrt{N}}
  = \frac{\norm{\mE_{\text{output}}}_F}{\sqrt{N V}} = N^{- 1 / 2} = O(1).
\end{equation*}

Next, we deal with \cref{des:feature-learn}. Consider the linearized update
\begin{equation}\label{eq:BYaiHs}
  \Delta h_k^{\ell + 1} = \sum_{l = 1}^{\ell} \sum_{i j} \frac{\partial h_k^{\ell + 1}}{\partial W_{i j}^l} \Delta W_{i j}^l
  + \sum_{i j} \frac{\partial h_k^{\ell + 1}}{\partial E_{\text{input}, i j}} \Delta E_{\text{input}, i j} + O(\eta_{\text{hidden}}^2 + \eta_{\text{input}}^2).
\end{equation}
Note that\internalComment{MbTODO say along with \eqref{eq:lAvyNo} that we use $a_k \asymp \norm{\va} / \sqrt{N}$ in general}
\begin{align*}
  \sum_{i j} \frac{\partial h_k^{\ell + 1}}{\partial W_{i j}^l} \Delta W^l_{i j}
  &= \sum_{i j} \sum_m \frac{\partial h_k^{\ell + 1}}{\partial h_m^{l + 1}}
  \frac{\partial h_m^{l + 1}}{\partial W_{i j}^l} \Delta W_{i j}^l\\
  &= \sum_{i j} \sum_m \frac{\partial h_k^{\ell + 1}}{\partial h_m^{l + 1}}
    \frac{\partial h_m^{l + 1}}{\partial (\mW^l \vh^l)_i} h_j^l \Delta W_{i j}^l\\
  &= \sum_m \frac{\partial h_k^{\ell + 1}}{\partial h_m^{l + 1}}
    \brk[\bigg]{\frac{\partial \vh^{l + 1}}{\partial (\mW^l \vh^l)} (\Delta \mW^l \vh^l)}_m\\
  &= \brk[\bigg]{\frac{\partial \vh^{\ell + 1}}{\partial \vh^{l + 1}}
    \frac{\partial \vh^{l + 1}}{\partial (\mW^l \vh^l)} (\Delta \mW^l \vh^l)}_k.
\end{align*}

\begin{lemma}\label{lem:QmUoXJ}
  Suppose \cref{ass:alignment} holds, $L^{1 - 2 \alpha} \lesssim N$, and the matrix $\frac{\partial \vh^{\ell + 1}}{\partial \vh^{l + 1}}$ is unaligned with the vector $\frac{\partial \vh^{l + 1}}{\partial (\mW^l \vh^l)} \Delta \mW^l \vh^l$: specifically,
  \begin{equation*}
    \max_{\ell} \max_{1 \leq l \leq \ell} \prn[\bigg]{\Expectlet \norm[\bigg]{\frac{\partial \vh^{\ell + 1}}{\partial \vh^{l + 1}} \frac{\partial \vh^{l + 1}}{\partial (\mW^l \vh^l)} \Delta \mW^l \vh^l}}
    \prn[\bigg]{N^{-1 / 2} \Expectlet \norm[\bigg]{\frac{\partial \vh^{\ell + 1}}{\partial \vh^{l + 1}}}_F \Expectlet \norm[\bigg]{\frac{\partial \vh^{l + 1}}{\partial (\mW^l \vh^l)} \Delta \mW^l \vh^l}}^{-1}
    \lesssim 1.
  \end{equation*}
  Then
  \begin{equation*}
    N^{-1 / 2} \max_{\ell} \max_{1 \leq l \leq \ell} \Expectlet \norm[\bigg]{\frac{\partial \vh^{\ell + 1}}{\partial \vh^{l + 1}} \frac{\partial \vh^{l + 1}}{\partial (\mW^l \vh^l)} \Delta \mW^l \vh^l}
    \lesssim L^{- \alpha} N^{1 / 2} \eta_{\text{hidden}}.
  \end{equation*}
\end{lemma}

\internalComment{
  The form of the alignment assumption
  is such that Cauchy-Schwarz goes through without using fourth moments.
  I was trying to be somewhat rigorous here,
  so if I want to bound the scale of, say, $\sum_{l = 1}^L \xi_l$, I need e.g. $\max_{1 \leq l \leq L} \Expectlet \abs{\xi}_l = O(1)$.
  Just having $\abs{\xi_l} = O_{\Problet}(1)$ is not enough because it is not true that $\sum_{l = 1}^L O_{\Problet}(1) = O_{\Problet}(L)$.
}

\begin{proof}
We start with
\begin{equation}\label{eq:qjFwas}
\begin{multlined}
  N^{-1 / 2} \norm[\bigg]{\frac{\partial \vh^{l + 1}}{\partial \hat{\vh}^{l + 1}} \frac{\partial \hat{\vh}^{l + 1}}{\partial (\mW^l \vh^l)} (\Delta \mW^l \vh^l)}
  \labrel{1}[\leq] L^{- \alpha} N^{- 1 / 2} \frac{\norm{\Delta \mW^l \vh^l}}{\norm[\big]{\hat{\vh}^{l + 1}} \norm{\mW^l \vh^l}}\\
  \labrel{2}[\leq] (1 + O(L^{-\alpha})) L^{-\alpha} N^{-1 / 2} \frac{\norm{\Delta \mW^l \vh^l}}{\norm{\mW^l \vh^l}}
\end{multlined}
\end{equation}
where \eqrefpt{eq:qjFwas}{1} is by
\begin{equation}\label{eq:lAvyNo}
  \norm[\bigg]{\frac{\partial \text{Norm}(\va)}{\partial \va} \vb}
  = \frac{1}{\norm{\va}} \norm[\bigg]{\prn[\bigg]{\mI - \frac{\va \va\trans}{\norm{\va}^2}} \vb}
  \leq \frac{\norm{\vb}}{\norm{\va}}
\end{equation}
for any two (non-zero) vectors $\va$ and $\vb$, and \eqrefpt{eq:qjFwas}{2} is because
$\norm{(1 - L^{-\alpha}) \va + L^{-\alpha} \vb} = 1 - O(L^{-\alpha})$ for unit vectors $\va$ and $\vb$ regardless of their alignment.

Now, consider the matrix
\begin{align*}
  \frac{\partial \vh^{\ell + 1}}{\partial \vh^{\ell}} &= \frac{\partial \vh^{\ell + 1}}{\partial \hat{\vh}^{\ell + 1}} \frac{\partial \hat{\vh}^{\ell + 1}}{\partial \vh^{\ell}}\\
                                  &= \frac{\partial \vh^{\ell + 1}}{\partial \hat{\vh}^{\ell + 1}} \prn[\bigg]{(1 - L^{-\alpha}) \mI + L^{-\alpha} \frac{\partial \text{Norm}(\mW^{\ell} \vh^{\ell})}{\partial \mW^{\ell} \vh^{\ell}} \mW^{\ell}}\\
  &= \frac{1}{\norm[\big]{\hat{\vh}^{\ell + 1}}} \prn[\big]{\mI - \vh^{\ell + 1} (\vh^{\ell + 1})\trans} \prn[\bigg]{(1 - L^{-\alpha}) \mI + L^{-\alpha} \frac{\partial \text{Norm}(\mW^{\ell} \vh^{\ell})}{\partial \mW^{\ell} \vh^{\ell}} \mW^{\ell}}.
\end{align*}
Notice that
\begin{align*}
  &\Expectlet \norm[\bigg]{\prn[\big]{\mI - \vh^{\ell + 1} (\vh^{\ell + 1})\trans} \prn[\bigg]{(1 - L^{-\alpha}) \mI + L^{-\alpha} \frac{\partial \text{Norm}(\mW^{\ell} \vh^{\ell})}{\partial \mW^{\ell} \vh^{\ell}} \mW^{\ell}} \mA}_F^2\\
  &\quad \leq \Expectlet \norm[\bigg]{\prn[\bigg]{(1 - L^{-\alpha}) \mI + L^{-\alpha} \frac{\partial \text{Norm}(\mW^{\ell} \vh^{\ell})}{\partial \mW^{\ell} \vh^{\ell}} \mW^{\ell}} \mA}_F^2\\
  &\quad = \Expectlet \trace \mA\trans \prn[\bigg]{(1 - L^{-\alpha}) \mI + L^{-\alpha} \frac{\partial \text{Norm}(\mW^{\ell} \vh^{\ell})}{\partial \mW^{\ell} \vh^{\ell}} \mW^{\ell}}\trans \prn[\bigg]{(1 - L^{-\alpha}) \mI + L^{-\alpha} \frac{\partial \text{Norm}(\mW^{\ell} \vh^{\ell})}{\partial \mW^{\ell} \vh^{\ell}} \mW^{\ell}} \mA\\
  &\quad \leq \brk[\big]{(1 - L^{-\alpha})^2 + L^{- 2 \alpha} (1 + O(N^{-1}))} \norm{\mA}_F^2\\
  &\quad = \brk[\big]{(1 - L^{-\alpha})^2 + L^{- 2 \alpha} } \prn[\big]{1 + O(L^{-2 \alpha} N^{-1})} \norm{\mA}_F^2.
\end{align*}
where $\mA$ is assumed independent of $\mW^{\ell}$ and the expectation is over $\mW^{\ell}$.
Also,
\begin{align*}
  \Expectlet \frac{1}{\norm[\big]{\hat{\vh}^{\ell + 1}}^2} &= \Expectlet \frac{1}{(1 - L^{-\alpha})^2 + L^{-2 \alpha} + 2 L^{-\alpha} (1 - L^{-\alpha}) (\vh^{\ell})\trans \text{Norm}(\mW^{\ell} \vh^{\ell})}\\
                                                        &= \frac{1}{(1 - L^{-\alpha})^2 + L^{-2 \alpha}} \Expectlet \frac{1}{1 + 2 \frac{L^{-\alpha} (1 - L^{-\alpha})}{(1 - L^{-\alpha})^2 + L^{-2 \alpha}} (\vh^{\ell})\trans \text{Norm}(\mW^{\ell} \vh^{\ell})}\\
                                                        &=: \frac{1}{(1 - L^{-\alpha})^2 + L^{-2 \alpha}} \Expectlet \frac{1}{1 + s}
                                                          = \frac{1}{(1 - L^{-\alpha})^2 + L^{-2 \alpha}} \Expectlet \prn[\bigg]{1 - s + \frac{s^2}{1 + s}}\\
                                                        &\leq \frac{1}{(1 - L^{-\alpha})^2 + L^{-2 \alpha}} \Expectlet (1 - s + 2 s^2)\\
                                                        &= \frac{1}{(1 - L^{-\alpha})^2 + L^{-2 \alpha}} \prn[\bigg]{1 + 8 \frac{L^{-2 \alpha} (1 - L^{-\alpha})^2}{N \prn[\big]{(1 - L^{-\alpha})^2 + L^{-2 \alpha}}^2}}\\
  &\leq \frac{1}{(1 - L^{-\alpha})^2 + L^{-2 \alpha}} \prn[\big]{1 + O(L^{-2 \alpha} N^{-1})}.
\end{align*}
By the Cauchy-Schwarz inequality, we can now infer
\begin{multline*}
  \Expectlet \frac{1}{\norm[\big]{\hat{\vh}^{\ell + 1}}} \norm[\bigg]{\prn[\big]{\mI - \vh^{\ell + 1} (\vh^{\ell + 1})\trans} \prn[\bigg]{(1 - L^{-\alpha}) \mI + L^{-\alpha} \frac{\partial \text{Norm}(\mW^{\ell} \vh^{\ell})}{\partial \mW^{\ell} \vh^{\ell}} \mW^{\ell}} \mA}_F\\
  \leq \prn[\big]{1 + O(L^{- 2 \alpha} N^{-1})} \norm{\mA}_F.
\end{multline*}
Repeated conditioning leads to the estimate of the expected Frobenius norm of the matrix $\frac{\partial \vh^{\ell + 1}}{\partial \vh^{l + 1}}$ for $\ell$ and $l$ arbitrarily distant from each other:
\begin{equation*}
\Expectlet \norm[\bigg]{\frac{\partial \vh^{\ell + 1}}{\partial \vh^{l + 1}}}_F \leq \prn[\big]{1 + O(L^{- 2 \alpha} N^{-1})}^L \sqrt{N} = O(\sqrt{N}),
\end{equation*}
where in the last bound we use\internalComment{!} that $L^{- 2 \alpha} N^{-1} \lesssim L^{-1}$.

Then, using low alignment\internalComment{!} of $\frac{\partial \vh^{\ell + 1}}{\partial \vh^{l + 1}}$ and $\frac{\partial \vh^{l + 1}}{\partial (\mW^l \vh^l)} \Delta \mW^l \vh^l$, we see that
uniformly in $l$ and $\ell$,
\begin{equation*}
  \frac{1}{\sqrt{N}} \Expectlet \norm[\bigg]{\frac{\partial \vh^{\ell + 1}}{\partial \vh^{l + 1}} \frac{\partial \vh^{l + 1}}{\partial (\mW^l \vh^l)} \Delta \mW^l \vh^l}
  \lesssim \frac{1}{\sqrt{N}} \Expectlet \norm[\bigg]{\frac{\partial \vh^{l + 1}}{\partial (\mW^l \vh^l)} \Delta \mW^l \vh^l}
  \lesssim \frac{L^{-\alpha}}{\sqrt{N}} \Expectlet \frac{\norm{\Delta \mW^l \vh^l}}{\norm{\mW^l \vh^l}},
\end{equation*}
where the last inequality is by \cref{eq:qjFwas}.
The scale of $\frac{\norm{\Delta \mW^l \vh^l}}{\norm{\mW^l \vh^l}}$ is $\eta_{\text{hidden}} N$ (recall that $\alpha_{\text{hidden}} = 1$ by \cref{ass:alignment}).
\end{proof}

By \cref{lem:QmUoXJ},
the first term in \cref{eq:BYaiHs} is of scale no more than $\eta_{\text{hidden}} \ell L^{-\alpha} N^{1 / 2}$.
Taking into account that $\vx$ is one-hot, it is easy to see similarly that $\sum_{i j} \frac{\partial h_k^{\ell + 1}}{\partial E_{\text{input}, i j}} \Delta E_{\text{input}, i j}$
is of scale not exceeding $\eta_{\text{input}}$.
So, the maximal stable learning rates are found from
\begin{equation*}
\eta_{\text{hidden}} L^{1 - \alpha} N^{1 / 2} \asymp N^{-1 / 2}, \quad \eta_{\text{input}} \asymp N^{- 1 / 2},
\end{equation*}
that is,
\begin{equation*}
\eta_{\text{hidden}} \asymp L^{\alpha - 1} N^{-1}, \quad \eta_{\text{input}} \asymp N^{- 1 / 2}.
\end{equation*}

\clearpage

\section{Additional experiments}\label{sec:additional-exps}

\subsection{LERP parameters}

\begin{figure}[h!]
  \centering

  \begin{tabular}{c c}
    \begin{subfigure}[t]{0.48\textwidth}
      \includegraphics[width=\linewidth]{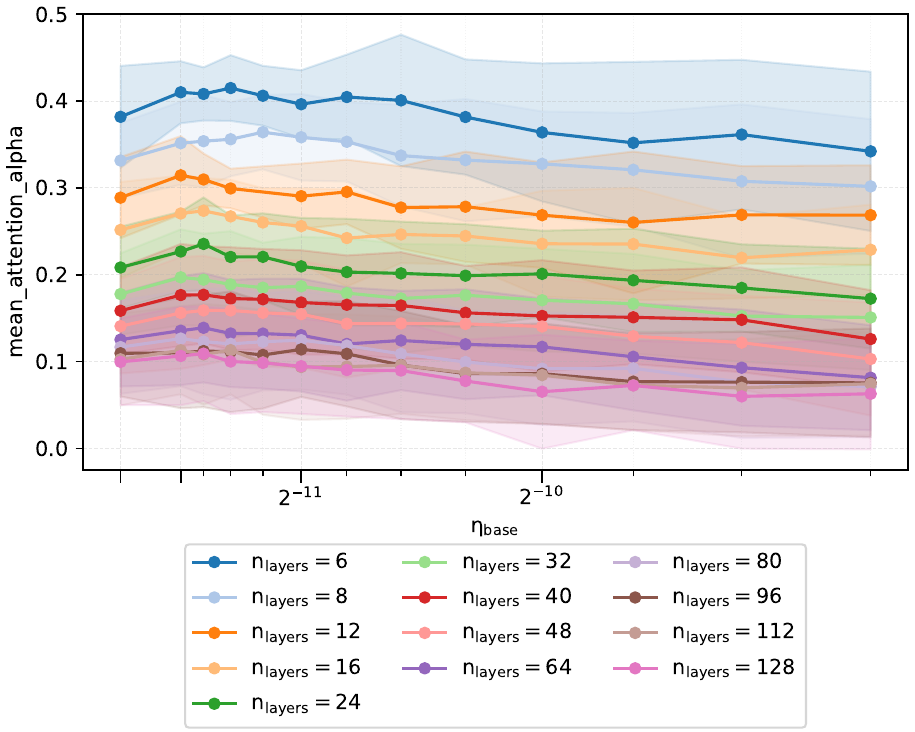}
    \end{subfigure}
    &
      \begin{subfigure}[t]{0.48\textwidth}
        \includegraphics[width=\linewidth]{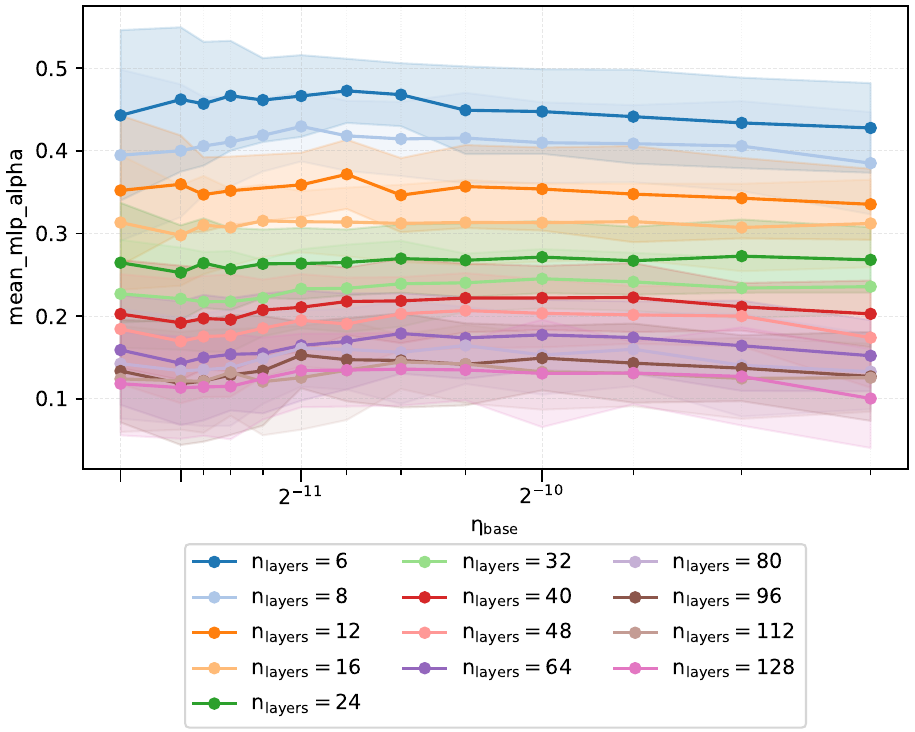}
      \end{subfigure}
    \\
    \begin{subfigure}[t]{0.48\textwidth}
      \includegraphics[width=\linewidth]{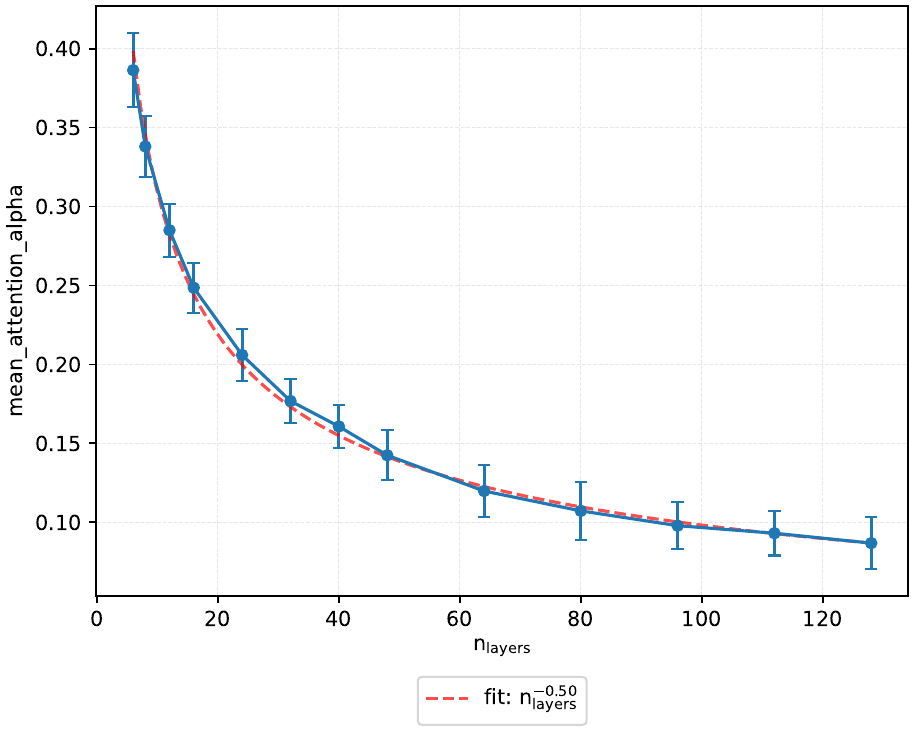}
    \end{subfigure}
    &
      \begin{subfigure}[t]{0.48\textwidth}
        \includegraphics[width=\linewidth]{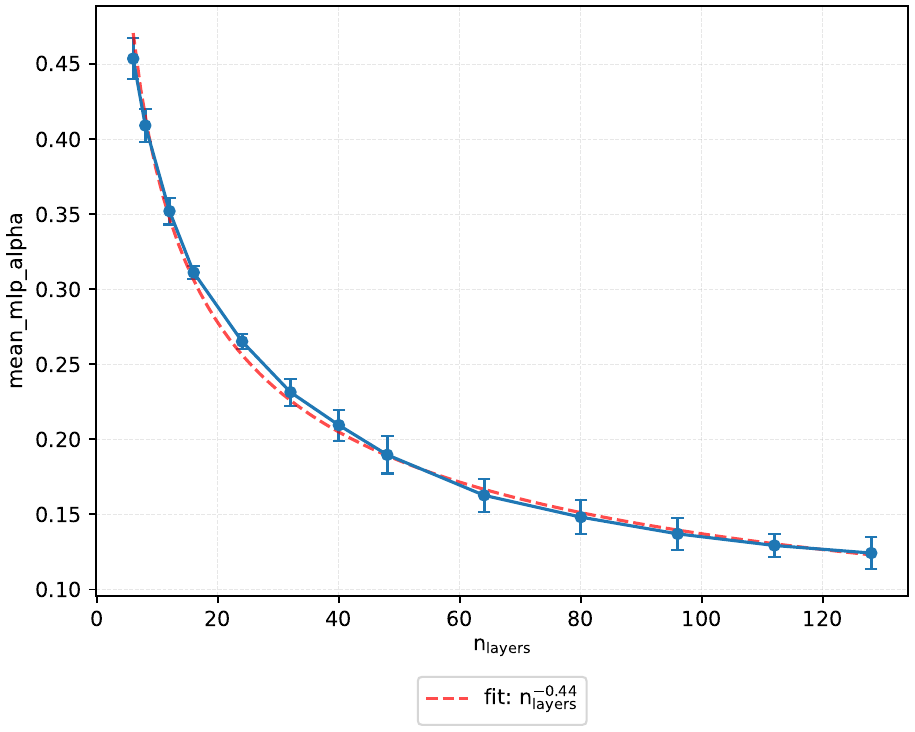}
      \end{subfigure}
  \end{tabular}
  \caption{
    Depth sweep of nGPT (baseline) at a fixed iteration count, power law fit of $\valpha_A$ and $\valpha_M$ components
    (averaged over layers and component indices).
  }
  \label{fig:depth-sweep-baseline-alpha-power-law}
\end{figure}

\begin{figure}[h!]
  \centering

  \begin{tabular}{c c}
      \begin{subfigure}[h]{0.48\textwidth}
        \includegraphics[width=\linewidth]{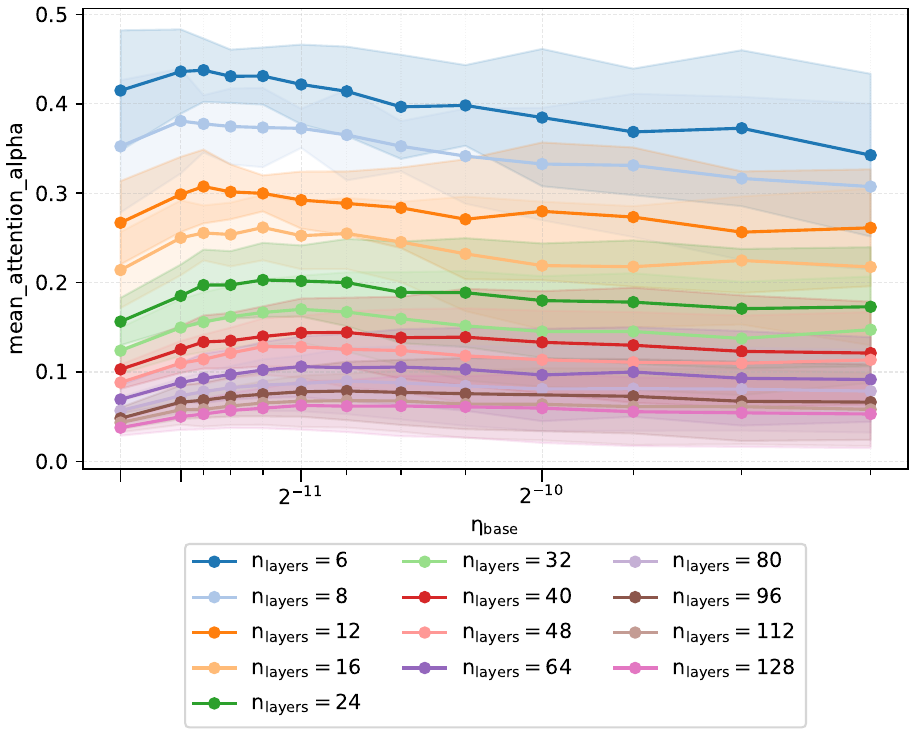}
      \end{subfigure}
             &
               \begin{subfigure}[h]{0.48\textwidth}
                 \includegraphics[width=\linewidth]{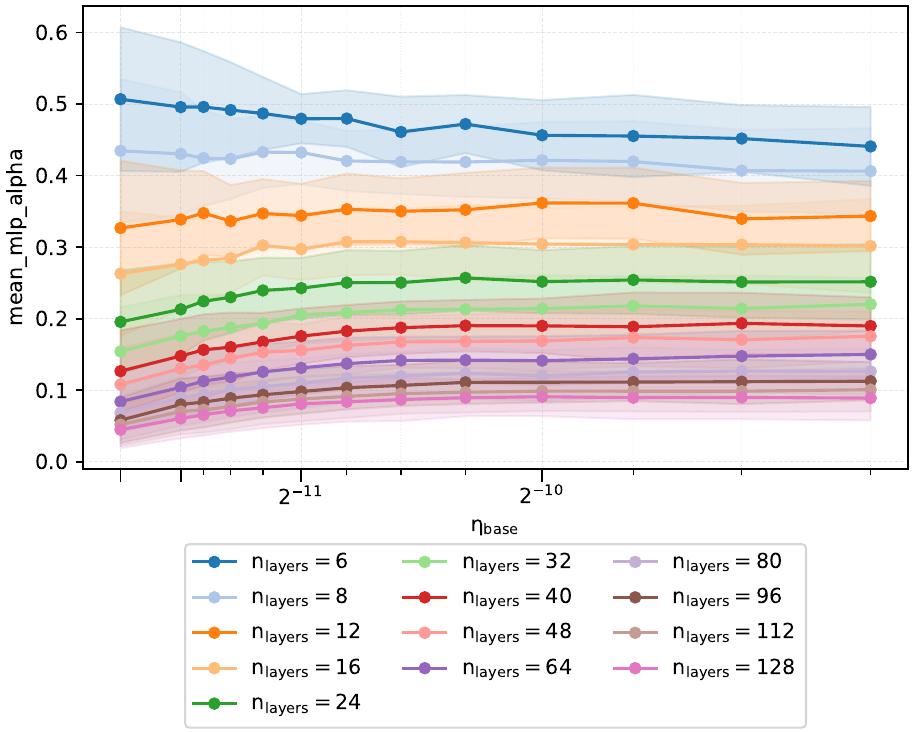}
               \end{subfigure}
      \\
      \begin{subfigure}[h]{0.48\textwidth}
        \includegraphics[width=\linewidth]{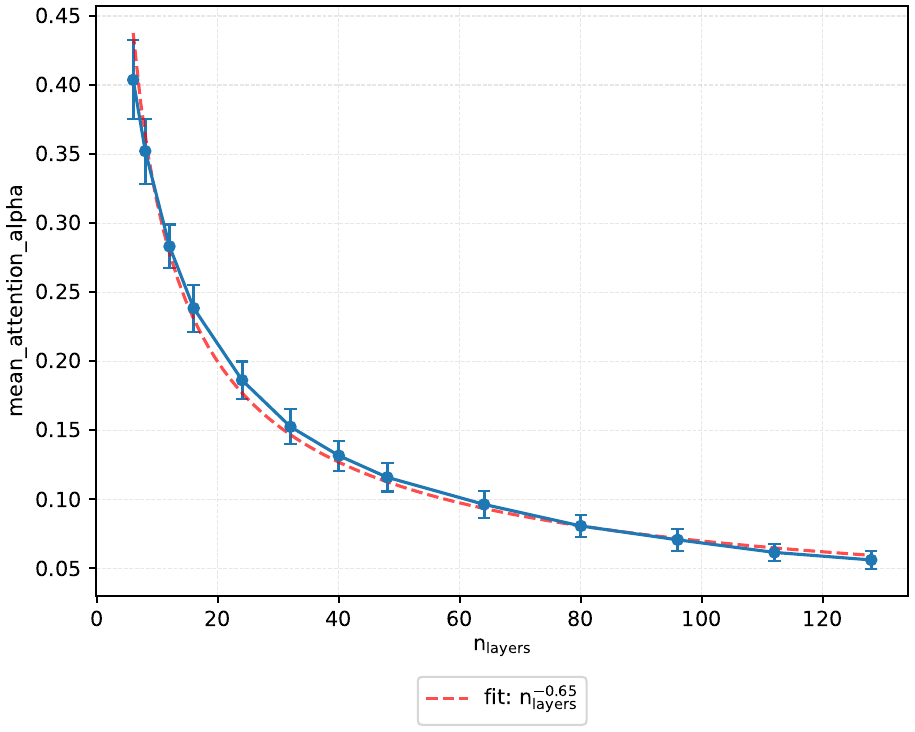}
      \end{subfigure}
             &
               \begin{subfigure}[h]{0.48\textwidth}
                 \includegraphics[width=\linewidth]{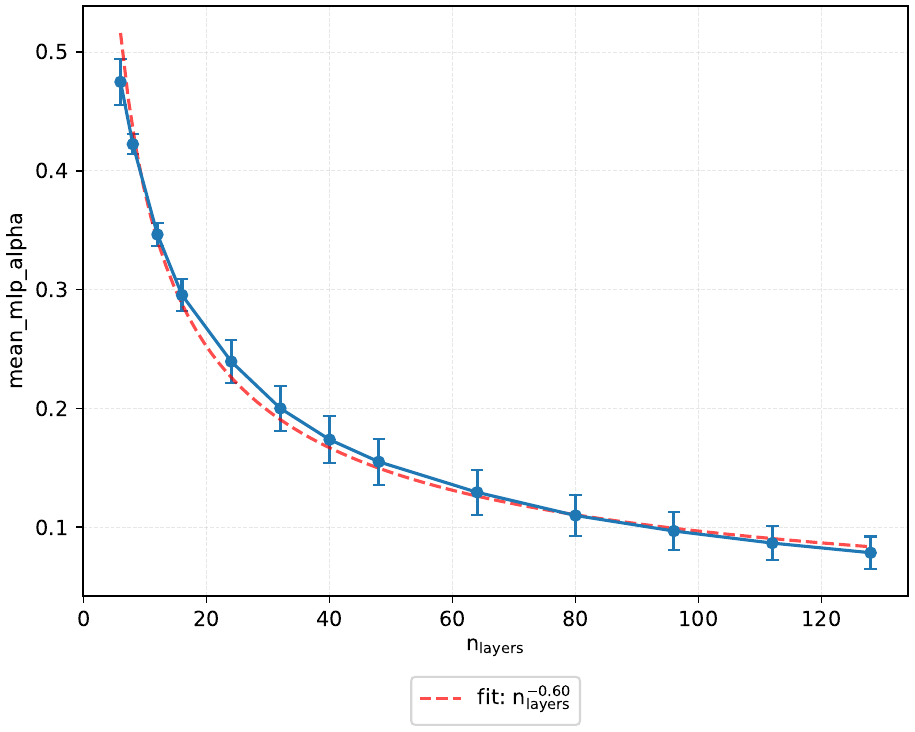}
               \end{subfigure}
    \end{tabular}

    \caption{
    Depth sweep of nGPT with $\mu$P-style corrections at a fixed iteration count, power law fit of $\valpha_A$ and $\valpha_M$ components
    (averaged over layers and component indices).
    }
    \label{fig:depth-sweep-ours-alpha-power-law}
  \end{figure}

\clearpage

\subsection{Alignment exponents for a wider model}\label{sec:alignment-exponents-for-wider-model}

Analogously to \cref{sec:alignment-exponents}, we plot in \cref{fig:alignment_val_l12h28,fig:alignment_val_loss_weighted_l12h28,fig:alignment_val_layer_dependence_l12h28}
the alignment exponents for a model with $\depth = 12$ and $\nheads = 28$.

\begin{figure}[t]
  \centering
  \begin{subfigure}[t]{0.32\textwidth}
    \centering
    \includegraphics[width=\textwidth]{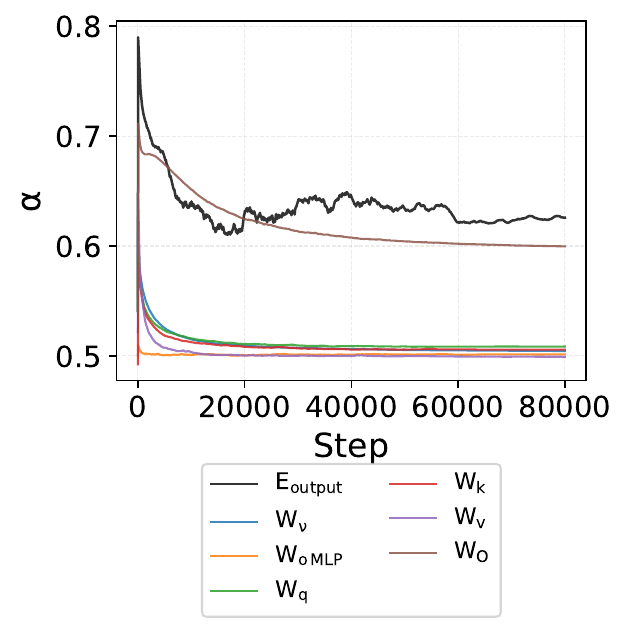}
    \caption{Alignment exponent $\alpha$}
  \end{subfigure}
  \hfill
  \begin{subfigure}[t]{0.32\textwidth}
    \centering
    \includegraphics[width=\textwidth]{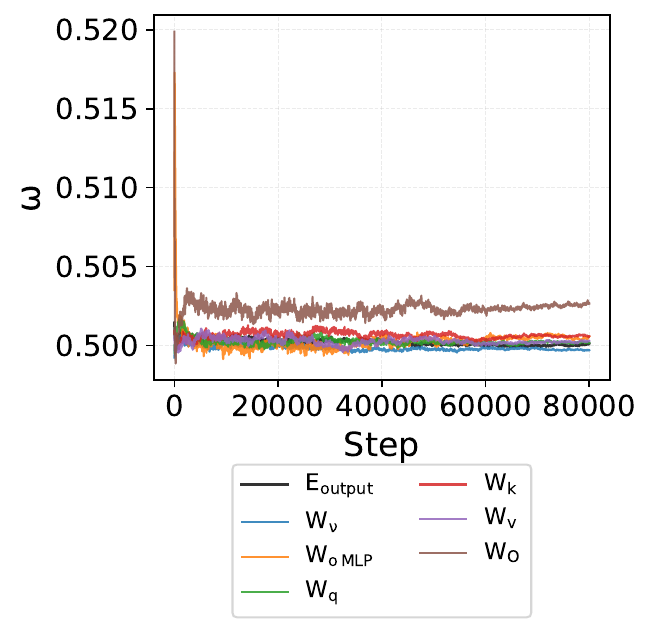}
    \caption{Alignment exponent $\omega$}
  \end{subfigure}
  \hfill
  \begin{subfigure}[t]{0.32\textwidth}
    \centering
    \includegraphics[width=\textwidth]{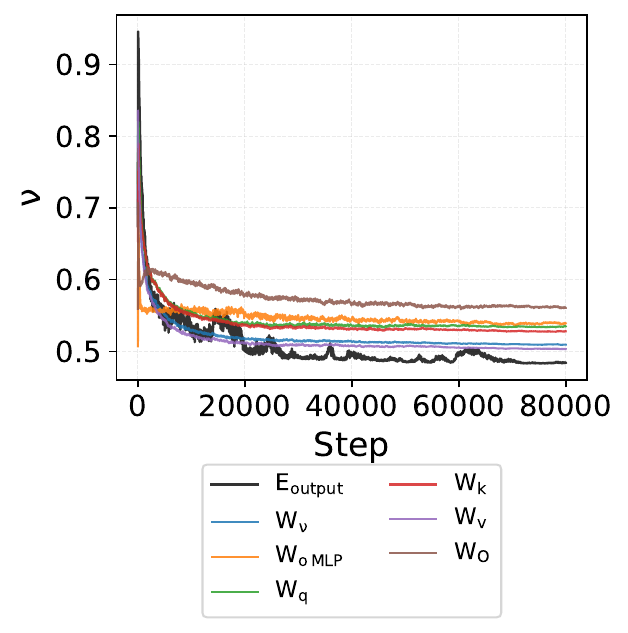}
    \caption{Alignment exponent $\nu$}
  \end{subfigure}
  \hfill
  \caption{Alignment exponents (\cref{def:alignment-exponents}) of $\nu$GPT with $\depth = 12$, $\nheads = 28$ on a fixed validation batch, averaged over layers.}
  \label{fig:alignment_val_l12h28}
\end{figure}

\begin{figure}[t]
  \centering
  \begin{subfigure}[t]{0.32\textwidth}
    \centering
    \includegraphics[width=\textwidth]{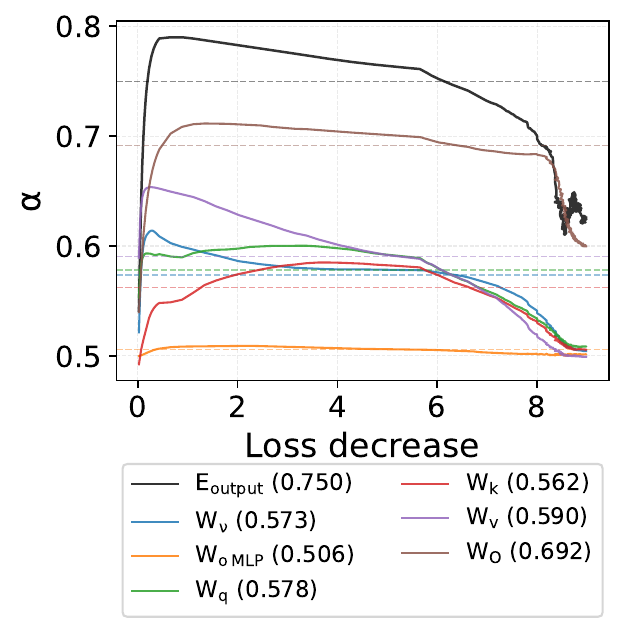}
    \caption{Alignment exponent $\alpha$}
  \end{subfigure}
  \hfill
  \begin{subfigure}[t]{0.32\textwidth}
    \centering
    \includegraphics[width=\textwidth]{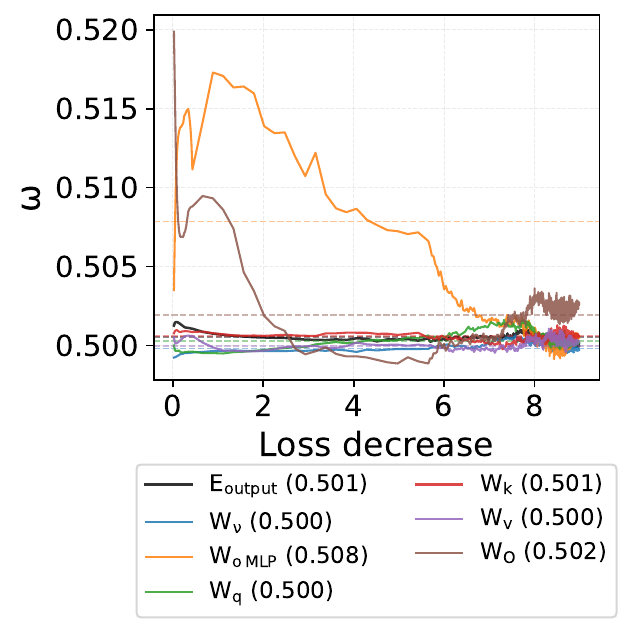}
    \caption{Alignment exponent $\omega$}
  \end{subfigure}
  \hfill
  \begin{subfigure}[t]{0.32\textwidth}
    \centering
    \includegraphics[width=\textwidth]{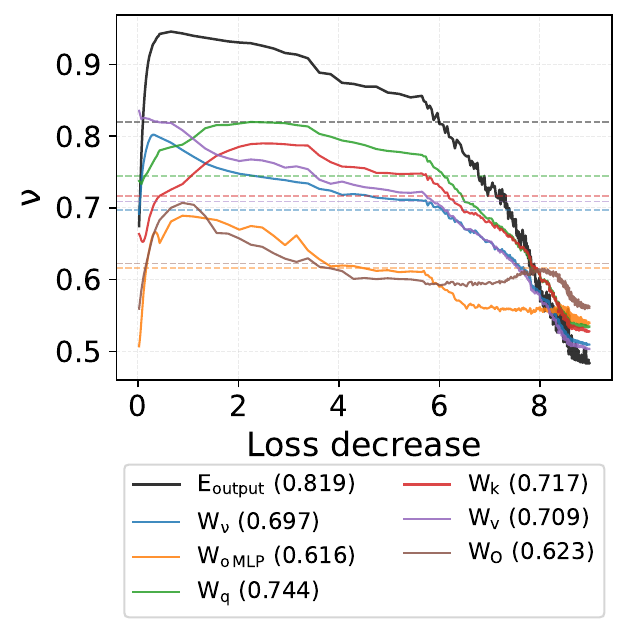}
    \caption{Alignment exponent $\nu$}
  \end{subfigure}
  \hfill
  \caption{Alignment exponents (\cref{def:alignment-exponents}) of $\nu$GPT with $\depth = 12$, $\nheads = 28$ on a fixed validation batch, averaged over layers,
   viewed as a function of loss decrease. The mid alignment assumption (\cref{def:full-no-mid-align}) matches the observations well.}
  \label{fig:alignment_val_loss_weighted_l12h28}
\end{figure}

\begin{figure}[t]
  \centering
  \begin{subfigure}[t]{0.32\textwidth}
    \centering
    \includegraphics[width=\textwidth]{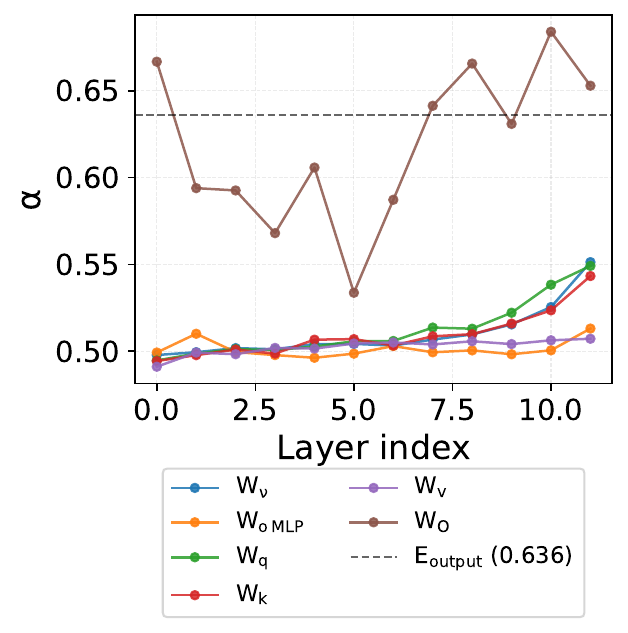}
    \caption{Alignment exponent $\alpha$}
  \end{subfigure}
  \hfill
  \begin{subfigure}[t]{0.32\textwidth}
    \centering
    \includegraphics[width=\textwidth]{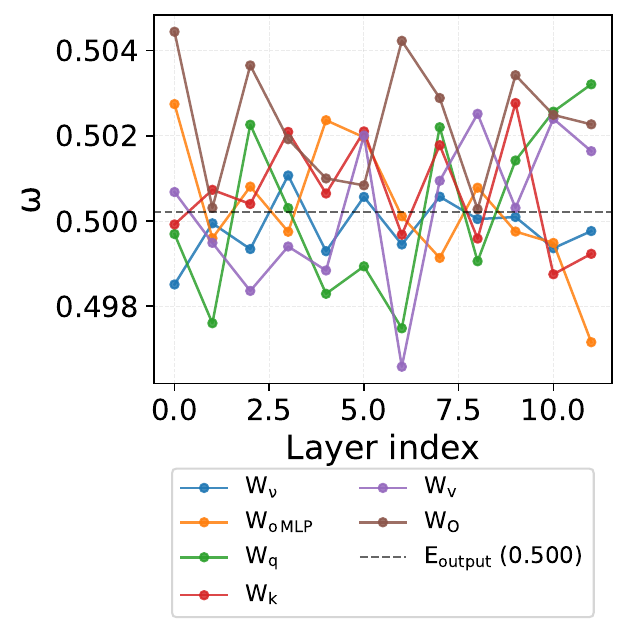}
    \caption{Alignment exponent $\omega$}
  \end{subfigure}
  \hfill
  \begin{subfigure}[t]{0.32\textwidth}
    \centering
    \includegraphics[width=\textwidth]{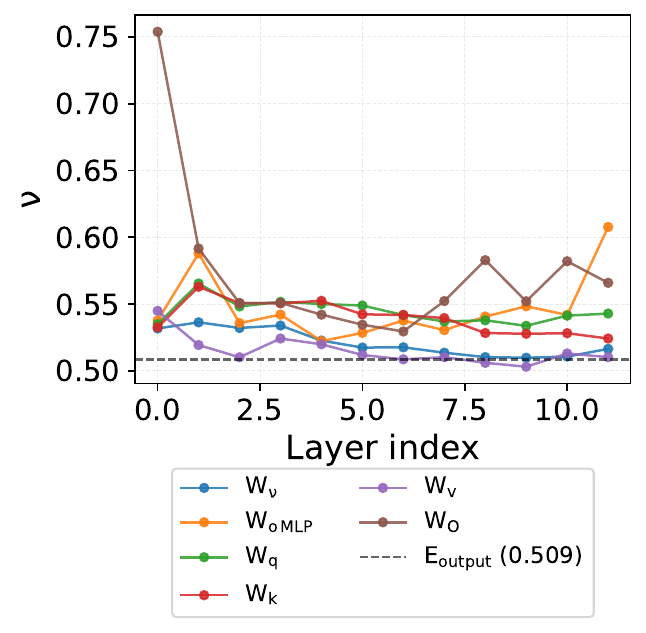}
    \caption{Alignment exponent $\nu$}
  \end{subfigure}
  \hfill
  \caption{Alignment exponents (\cref{def:alignment-exponents}) of $\nu$GPT with $\depth = 12$, $\nheads = 28$ on a fixed validation batch vs. layer index, averaged over steps.}
  \label{fig:alignment_val_layer_dependence_l12h28}
\end{figure}

\subsection{Width sweep at 20 tokens per parameter}\label{sec:width-sweep-at-20-tpp}
We train our version of $\nu$GPT at 20 tokens per parameter with token count corrections removed, scaling the number of heads.
The power law is again close to $(\text{iter. count})^{- 1 / 3}$ like in the sweep of a fixed model
(\cref{fig:width-sweep-ngpt-ours-without-tc-corrections}).

\begin{figure}[h!]
  \centering

  \begin{tabular}{c c}
    \begin{subfigure}[h]{0.48\textwidth}
      \includegraphics[width=\linewidth]{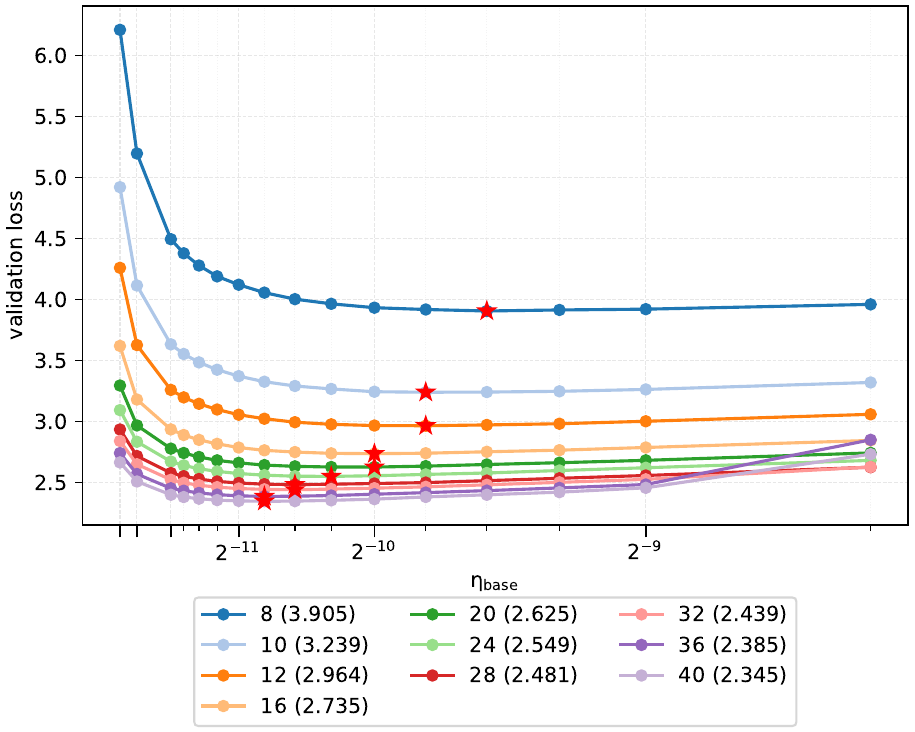}
    \end{subfigure}
    &
      \begin{subfigure}[h]{0.48\textwidth}
        \includegraphics[width=\linewidth]{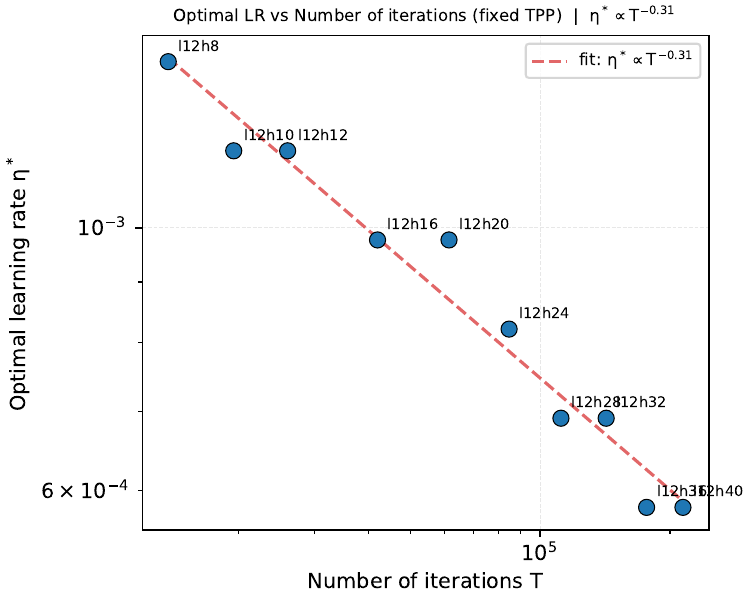}
      \end{subfigure}
  \end{tabular}
  \caption{A width sweep of $\nu$GPT models with different $\nheads$ (in the legend) at 20 tokens per parameter but without token count corrections: optimal peak learning rate also decreases at a law close to $(\text{iter. count})^{- 1 / 3}$, as in \cref{fig:steps-sweep-ngpt-ours}.
  }
  \label{fig:width-sweep-ngpt-ours-without-tc-corrections}
\end{figure}

\subsection{$\eta_{\text{output}}$ is too large}

The performance of each parametrization is quite sensitive to the learning rate applied to unembedding weights $\eta_{\text{output}}$.
Thus, consistently with \citet{everett2024scalingexponentsacross}, we find that significant improvements can be obtained by tuning the ratios
between $\eta_{\text{input}}$, $\eta_{\text{hidden}}$, $\eta_{\text{output}}$.
This tuning can be done on a small model and is not resource intensive (though we conduct and report experiments with different sizes).
The results of such tuning,
reported in \cref{fig:both-eta-input-output-sweep},
suggest that $\eta_{\text{output}}$ is too large in our setting, which does not necessarily impact learning transfer but does hurt performance.
Based on this, we propose a small modification (designed to be as simple as we could make it): multiplying $\eta_{\text{output}}$ by a tuned coefficient ($2^{-1}$ in our case).

\begin{figure}
  \centering
  \begin{tabular}{cc}
  \begin{subfigure}[h]{0.48\textwidth}
    \includegraphics[width=\linewidth]{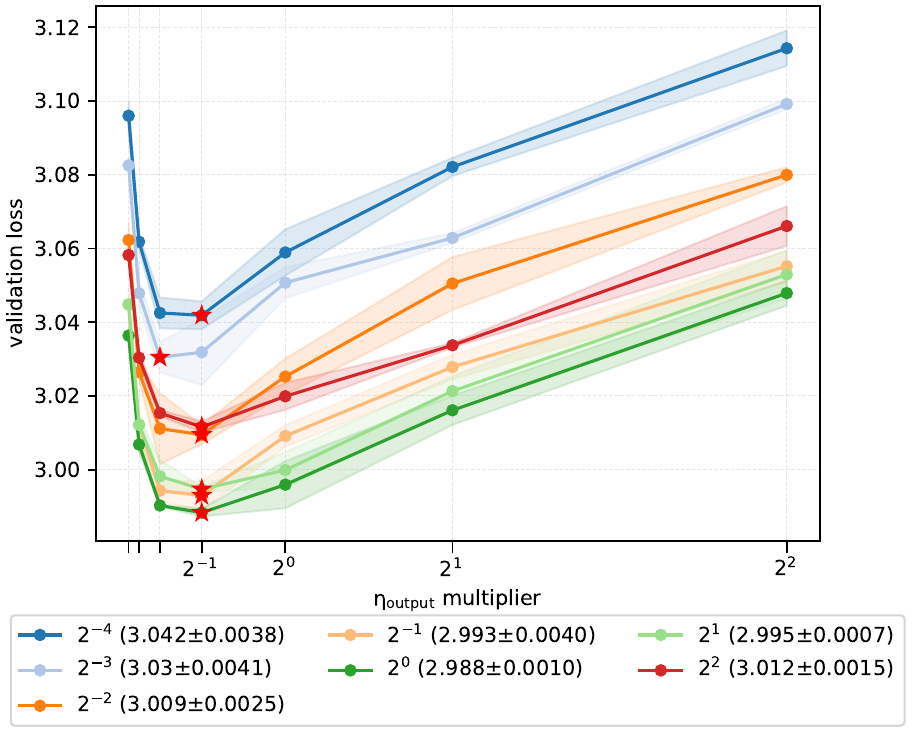}
    \caption{$\depth = \nheads = 10$}
  \end{subfigure}

  \begin{subfigure}[h]{0.48\textwidth} \includegraphics[width=\linewidth]{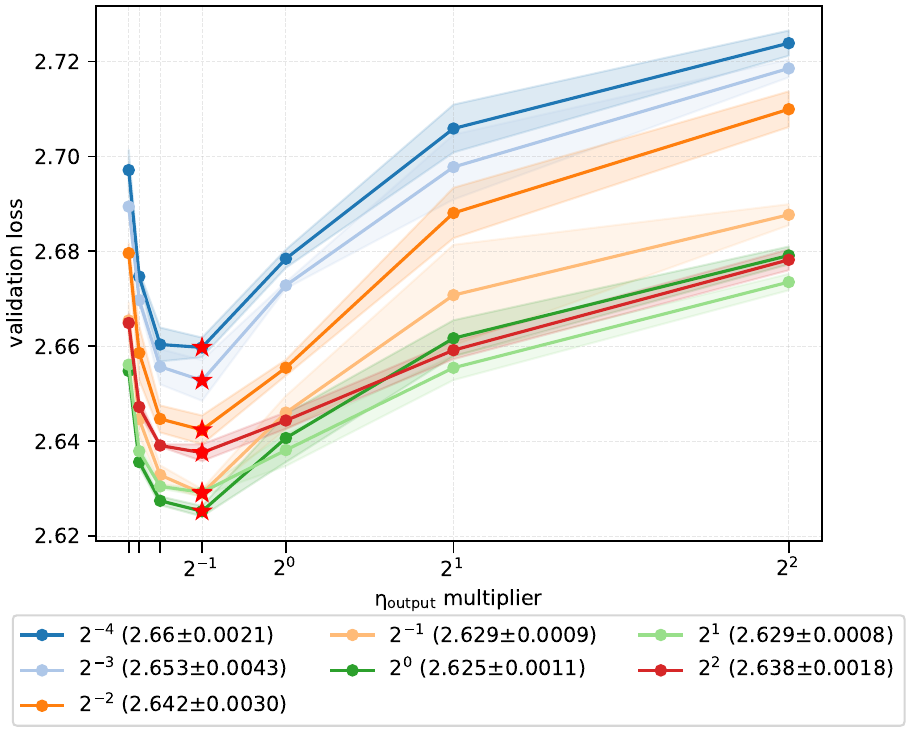}
    \caption{$\depth = \nheads = 16$}
  \end{subfigure}
  \end{tabular}
  \caption{Sweep of optimal $\eta_{\text{input}}$ (different curves) and $\eta_{\text{output}}$ (x-axis) multipliers for $\nu$GPT,
      showing that $\eta_{\text{output}}$ is too large; we multiply it by $2^{-1}$ for performance.}
  \label{fig:both-eta-input-output-sweep}
\end{figure}

\subsection{\mup or not \mup, that is the question (logits scaling)}\label{sec:mup-or-not-mup}

In this ablation, we investigate whether 'tis nobler\internalComment{joke alert} to have logits be $\Theta(\dmodel^{- 1 / 2})$ at the beginning of training with $\Theta(1)$ updates, as done
in \mup, or if they should be $\Theta(1)$ both at initialization and during training.
Specifically, we compare the two parametrizations in \cref{tab:mup-not-mup-param}.
Recall that under $\omega_{\text{output}} = 1$, \mup is the only parametrization satisfying
\cref{des:stab-init,des:feature-learn} for large-width neural networks;
however, our experiments (\cref{sec:alignment-exponents}) decisively show that the correct assumption is
$\omega_{\text{output}} = 1 / 2$, allowing for a different parametrization
in which $\eta_{\text{output}} \asymp \eta_{\text{hidden}}$ but $s_{\text{$z$,init}}$ needs to be scaled with width.
Although mid alignment (\cref{def:full-no-mid-align})
better matches measured average alignment exponents
and transfers better over width,
here we fix the full alignment assumption as in ``CompleteP'' to isolate the effect of only changing logits scaling,
and refer to this parametrization $\nu$GPT (full align), as opposed to $\nu$GPT.
The analogue of this parametrization is called ``standard'' and ``NTK''
(not $\mu$P) in \citet{everett2024scalingexponentsacross}\footnote{
  For clarity, the parametrizations they call ``standard'' and ``NTK'' are theoretically (under infinite precision) equivalent, and both names are unusual with respect to the literature so we avoid them.
}. They observe that such a choice outperforms \mup in terms of the best validation loss.
In \cref{fig:width-sweep-ngpt-mup-vs-muppr,fig:width-sweep-ngpt-mup-vs-muppr-tuned},
we observe that performance essentially matches in our setting when either not tuning $\eta_{\text{output}}$ for either parametrization or
tuning it separately for both.
Hence, we take no position as to which choice is better but choose
to scale logits $\Theta(1)$ at initialization (same as during training)
because the assumption $\omega_{\text{output}} = 1$ leading to $\mu$P is inconsistent with experiment, and hence there is no principled reason
to scale differently.

\begin{table}
  \centering
  \begin{NiceTabular}{c c c}
    \toprule
    \multirow{2}{*}{\textbf{Hyperparameter}} & \textbf{``CompleteP''} & \multirow{2}{*}{\textbf{$\nu$GPT (full align)}}\\
    & {\small with $\datacor$ corrections} & \\
    \midrule
    $\eta_{\text{base}}$ & $\eta_{\text{global}} {\color{mplblue} \datacor^{- \dataalph}}$ & $\eta_{\text{global}} {\color{mplblue} \datacor^{- \dataalph}}$ \\
    \midrule
    $\sigma^2_{\text{input}}$ & arbitrary & arbitrary \\
    $\eta_{\text{input}}$ & $\eta_{\text{base}} {\color{mplorange} \widthcor^{- 1 / 2}}$ & $\eta_{\text{base}} {\color{mplorange} \widthcor^{- 1 / 2}}$ \\
    \midrule
    $\sigma^2_{\text{hidden}}$ & arbitrary & arbitrary \\
    $\eta_{\text{hidden}}$ & $\eta_{\text{base}} {\color{mplorange} \widthcor^{-1}}$ & $\eta_{\text{base}} {\color{mplorange} \widthcor^{-1}}$ \\
    \midrule
    $\alpha_{\text{$A$,init}}$ & $0.05 \, {\color{mplgreen} \depthcor^{-1}}$ & $0.05 \, {\color{mplgreen} \depthcor^{-1}}$ \\
    $\alpha_{\text{$A$,scale}}$ & ${\color{mplpurple} 0.03}$ & ${\color{mplpurple} 0.03}$\\
    $\alpha_{\text{$M$,init}}$ & $0.05 \, {\color{mplgreen} \depthcor^{-1}}$ & $0.05 \, {\color{mplgreen} \depthcor^{-1}}$ \\
    $\alpha_{\text{$M$,scale}}$ & ${\color{mplpurple} 0.03}$ & ${\color{mplpurple} 0.03}$ \\
    \midrule
    $s_{\text{$qk$,init}}$ & 1 & 1 \\
    $s_{\text{$qk$,scale}}$ & ${\color{mplpurple} 0.03}$ & ${\color{mplpurple} 0.03}$ \\
    \midrule
    $s_{\text{$u$,init}} = s_{\text{$\nu$,init}}$ & 1 & 1 \\
    $s_{\text{$u$,scale}} = s_{\text{$\nu$,scale}}$ & 1 & 1 \\
    \midrule
    $s_{\text{$z$,init}}$ & 1 & ${\color{red} \widthcor^{1 / 2}}$ \\
    $s_{\text{$z$,scale}}$ & ${\color{mplpurple} 0.03}$ & ${\color{mplpurple} 0.03}$ \\
    \midrule
    $\sigma^2_{\text{output}}$ & arbitrary & arbitrary \\
    $\eta_{\text{output}}$ & $\eta_{\text{base}} {\color{mplorange} \widthcor^{- 1 / 2}}$ & $\eta_{\text{base}} {\color{mplorange} \widthcor^{{\color{red}- 1}}}$ \\
    \bottomrule
  \end{NiceTabular}
  \caption{
    The correct value $\omega_{\text{output}} = 1 / 2$ leads to two different ways to scale logits and $\eta_{\text{output}}$.
    \label{tab:mup-not-mup-param}
  }
\end{table}

\begin{figure}[h]
  \centering

  \begin{tabular}{c c}
  \begin{subfigure}[h]{0.48\textwidth}
    \includegraphics[width=\linewidth]{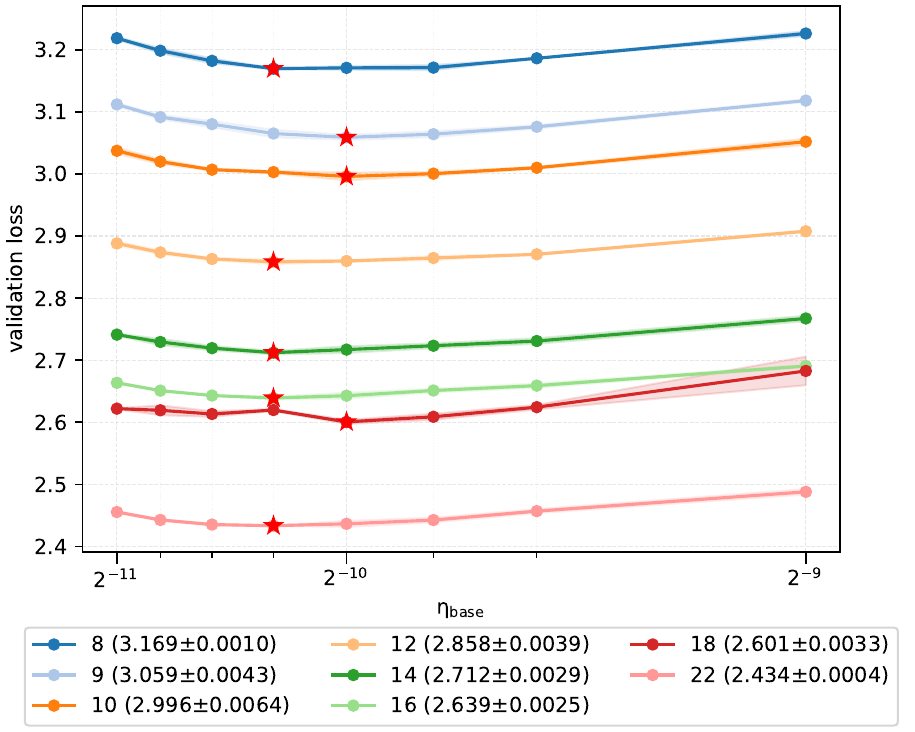}
    \caption{``CompleteP''\label{fig:width-sweep-ngpt-mup-vs-muppr-mup}}
  \end{subfigure}
    &
  \begin{subfigure}[h]{0.48\textwidth}
    \includegraphics[width=\linewidth]{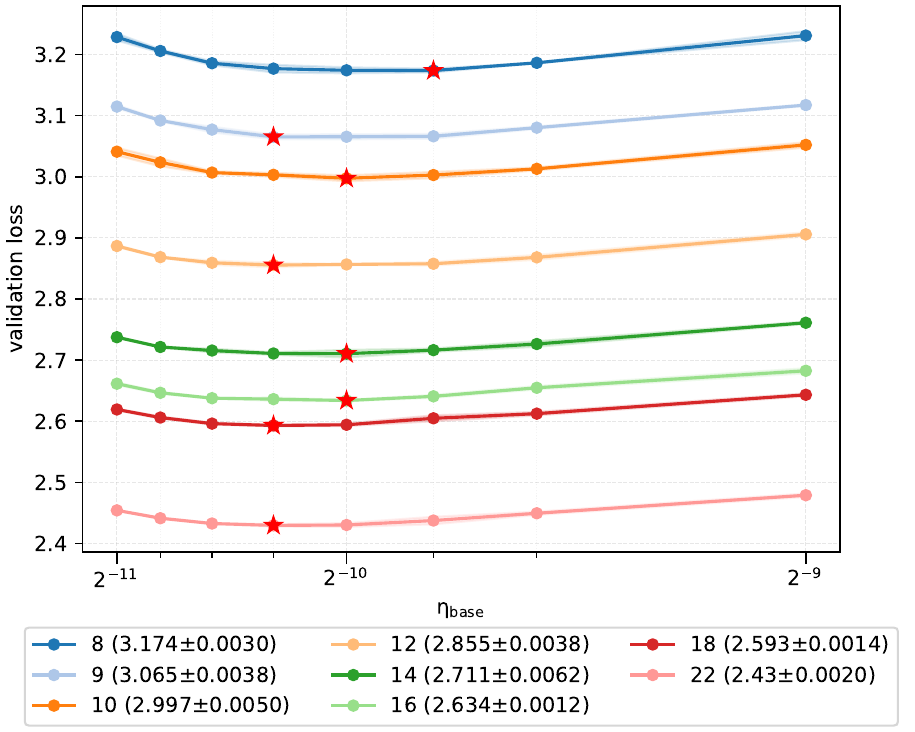}
    \caption{$\nu$GPT (full align)\label{fig:width-sweep-ngpt-mup-vs-muppr-muppr}}
  \end{subfigure}
  \end{tabular}

  \caption{Aspect ratio sweep of ``CompleteP'' and $\nu$GPT (full align) at 20 tokens per parameter
  }
  \label{fig:width-sweep-ngpt-mup-vs-muppr}
\end{figure}

\begin{figure}[h]
  \centering
  \begin{tabular}{c c}
    \begin{subfigure}[h]{0.48\textwidth}
      \includegraphics[width=\linewidth]{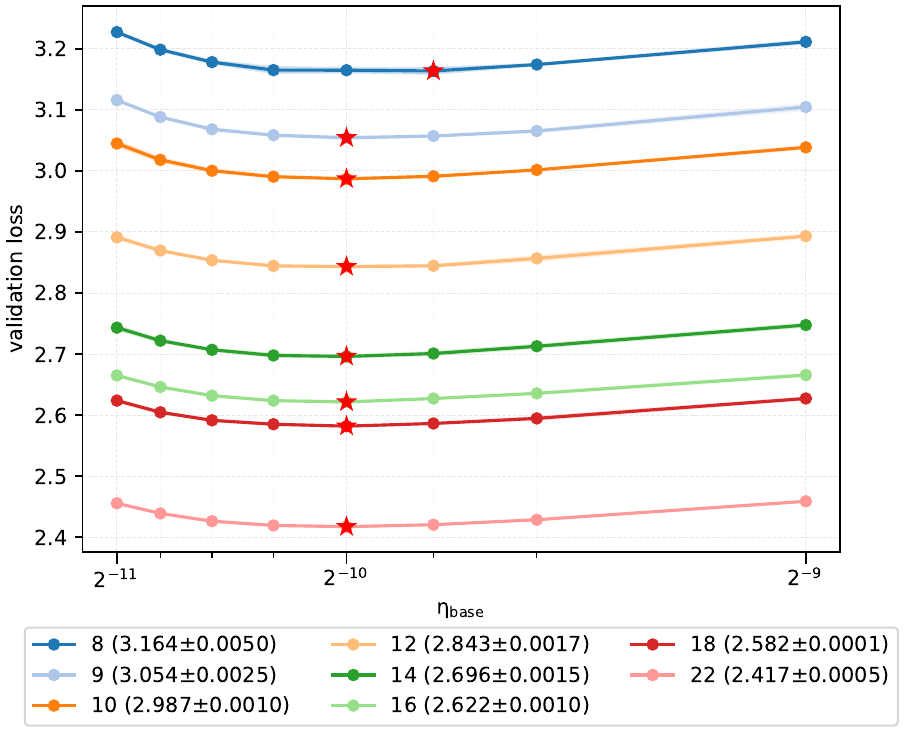}
      \caption{``CompleteP'' with tuned $\eta_{\text{output}}$\label{fig:width-sweep-ngpt-mup-vs-muppr-tuned-mup}}
    \end{subfigure}
    &
      \begin{subfigure}[h]{0.48\textwidth} \includegraphics[width=\linewidth]{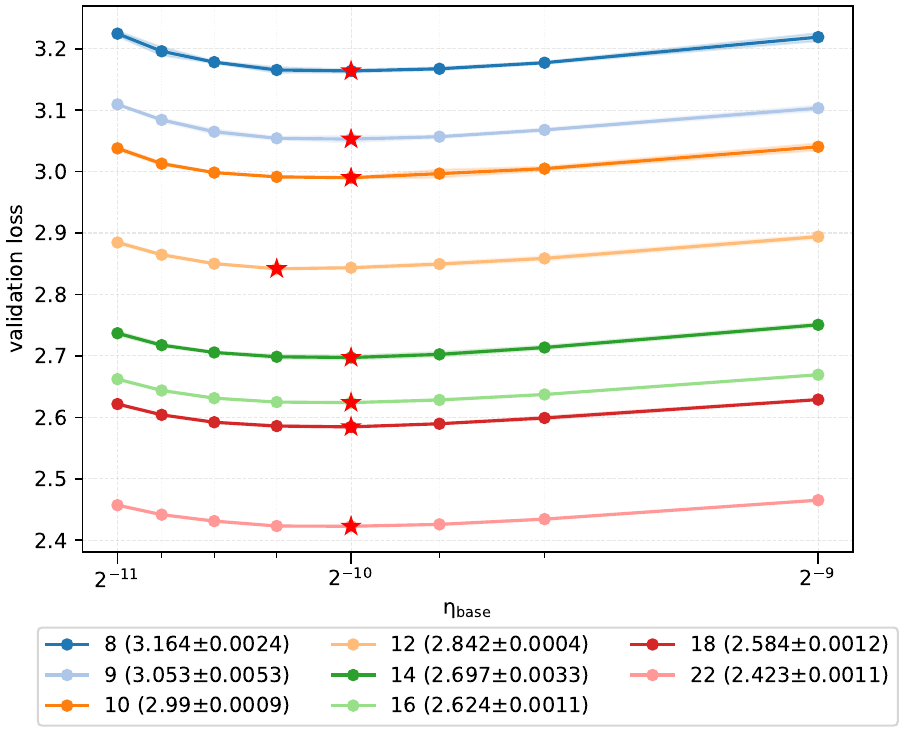}
        \caption{$\nu$GPT (full align) with tuned $\eta_{\text{output}}$\label{fig:width-sweep-ngpt-mup-vs-muppr-tuned-muppr}}
      \end{subfigure}
  \end{tabular}
  \caption{Aspect ratio sweep of ``CompleteP'' and $\nu$GPT (full align) with tuned $\eta_{\text{output}}$ at 20 tokens per parameter
  }
\label{fig:width-sweep-ngpt-mup-vs-muppr-tuned}
\end{figure}

\ifthenelse{\boolean{isInternal}}{
  \clearpage
  \newpage
\input{internal_experiments_repo}
}{}

\bibliographystyle{assets/plainnat}
\bibliography{ngpt_transfer}

\end{document}